%% file: tmlr2026.tex
\documentclass{article}

\usepackage[preprint]{styles/tmlr}

\usepackage{hyperref}
\usepackage[center]{styles/VABcommands}
\usepackage{styles/VABenvironments}
\usepackage{cleveref}
\usepackage{thmtools} 
\usepackage{thm-restate}
\AtBeginEnvironment{theorem}{}
\AtBeginEnvironment{restatable}{}
\hypersetup{
    colorlinks,
    linkcolor={red!50!black},
    citecolor={blue!50!black},
    urlcolor={blue!80!black}
}
\usepackage{booktabs}
\usepackage{makecell}
\usepackage{subcaption}
\usepackage{algorithm}
\usepackage{algpseudocode}
\usepackage{forest}
\usepackage{csquotes}
\MakeOuterQuote{"}

\newtheorem*{openproblem}{Open Problem}

\title{Heat and Matérn Kernels on Matchings}

\author{%
\name Dmitry Eremeev
\email dimaeremeev2002@gmail.com  \\
\addr Higher School of Economics
\AND
\name Salem Said
\email salem.said@univ-grenoble-alpes.fr \\
\addr University Grenoble Alpes
\AND
\name Viacheslav Borovitskiy
\email viacheslav.borovitskiy@gmail.com \\
\addr University of Edinburgh
}

\def\month{TODO}  
\def\year{TODO} 
\def\openreview{\url{https://openreview.net/forum?id=XXXX}} 

\newcommand\blfootnote[1]{%
  \begingroup
  \renewcommand\thefootnote{}\footnote{#1}%
  \addtocounter{footnote}{-1}%
  \endgroup
}

\begin{document}

\maketitle

\blfootnote{\vspace{0.2cm}Our source code is available at \url{https://github.com/eremeev-d/matchings-gp}}

\vspace*{-1cm}

\begin{abstract}

Applying kernel methods to matchings is challenging due to their discrete, non-Euclidean nature.
In this paper, we develop a principled framework for constructing geometric kernels that respect the natural geometry of the space of matchings.
To this end, we first provide a complete characterization of stationary kernels, i.e. kernels that respect the inherent symmetries of this space.
Because the class of stationary kernels is too broad, we specifically focus on the heat and Matérn kernel families, adding an appropriate inductive bias of smoothness to stationarity.
While these families successfully extend widely popular Euclidean kernels to matchings, evaluating them naively incurs a prohibitive super-exponential computational cost.
To overcome this difficulty, we introduce and analyze a novel, sub-exponential algorithm leveraging zonal polynomials for efficient kernel evaluation.
Finally, motivated by the known bijective correspondence between matchings and phylogenetic trees---a crucial data modality in biology---we explore whether our framework can be seamlessly transferred to the space of trees, establishing novel negative results and identifying a significant open problem.
\end{abstract}

\section{Introduction}

Matchings—partitions of a set of $2n$ elements into $n$ pairs—are fundamental combinatorial structures that model paired relationships across diverse domains.
In sports analytics, for instance, a matching represents a single round of a tournament schedule, determining which teams play against each other.
Similarly, in logistics and planning, matchings model the assignment of students to dormitory rooms or the pairing of participants in networking activities.
Despite the practical importance of these structures, applying machine learning techniques to tasks involving them remains challenging.
Standard machine learning models are typically designed for continuous Euclidean spaces ($\mathbb{R}^d$), whereas matchings are discrete, combinatorial objects that cannot be naturally treated as vectors without destroying their inherent geometric structure.

While we do not address these in the paper, our investigation is motivated by potential applications such as Bayesian optimization \citep{garnett2023} and kernel hypothesis testing \citep{gretton2012}.
For instance, one might wish to optimize a tournament schedule to maximize fairness, or to statistically detect latent biases in student housing allocations.
These scenarios illustrate the need to model complex functions and distributions over the non-Euclidean domain of matchings, underscoring the potential utility of learning thereon.

In this purely theoretical paper, we develop the foundations for kernel methods on matchings.
To enable this, one could technically define a kernel by embedding matchings into a high-dimensional vector space (e.g., via one-hot encoding of a matching's edges), but such an ad-hoc approach would fail to respect the natural geometry of the space of matchings.
A more principled approach is to use a true \emph{geometric} kernel \citep{mostowsky2025}: a kernel that respects the natural symmetries of the space---i.e. is \emph{stationary} under the relabeling of the underlying items---and varies smoothly with respect to a meaningful distance.

This geometric perspective has driven recent progress in machine learning, building practical generalizations of standard kernels---such as the heat\footnote{The heat kernel is also commonly known as the squared exponential, diffusion, Gaussian, or RBF kernel.} and Matérn families \citep{scholkopf2002, porcu2024}---to non-Euclidean domains like Riemannian manifolds \citep{borovitskiy2020, azangulov2024a, azangulov2024b} and graphs \citep{borovitskiy2021, borovitskiy2023, bolin2024, kondor2002, doumont2025}.
Adopting this paradigm, we present a principled framework for constructing kernels on matchings.

We begin by formalizing the geometry of the space of matchings in \Cref{sec:stationary}, characterizing it as the homogeneous space $S_{2n}/H_n$ and identifying the pair $(S_{2n}, H_n)$ as a \emph{Gelfand pair}.
This algebraic structure enables us to provide a complete characterization of all \emph{stationary} kernels on matchings.
However, stationarity alone is insufficient in practice, as it is too permissive, encompassing kernels such as the delta kernel ($k(x, x') = \1_{x = x'}$) that preclude any meaningful generalization.
To address this, we additionally view the set of matchings as vertices of a weighted quotient of the \emph{Cayley graph} of the symmetric group $S_{2n}$.
This construction induces a natural distance and a discrete notion of smoothness governed by the spectral theory of the graph Laplacian.
Using this spectral perspective, we identify the heat and Matérn kernels as the specific stationary kernels that generalize the smoothness properties of their Euclidean counterparts (\Cref{sec:kernels:heat_matern}).

However, this spectral characterization encounters an immediate computational barrier: the cardinality of the matching space $|\c{X}_n|$ grows super-exponentially.
Constructing the full graph Laplacian, let alone diagonalizing it---which is required for naive kernel computation---is intractable for all but the smallest cases ($n \leq 7$).
To overcome this, we utilize the expansion of stationary kernels into \emph{zonal spherical functions}.
This effectively reduces the task of kernel evaluation (or approximation) to the computation of these functions.
Crucially, we propose a novel algorithm for this task based on \emph{zonal polynomials} in \Cref{sec:computation}, reducing the computational complexity from the super-exponential scale of $|\c{X}_n|$ to the sub-exponential scale of the partition function $p(n)$---admittedly, still super-polynomial.
This algorithmic advance allows us to perform exact or highly accurate approximate kernel computations for much larger problem sizes (e.g., $n=25$).

\looseness=-1
Finally, in \Cref{sec:discussion}, we analyze the computational complexity and approximation quality of our proposed techniques, and we explore the connection between matchings and phylogenetic trees.
Although there is a known bijection between them, we prove it does not preserve geometry.
Consequently, the "push-forward" of our geometric kernels on matchings to the space of trees fails to yield kernels that respect the geometry of phylogenetic trees, highlighting the existence of a geometry-preserving bijection as an important open problem.

\subsection{Background}
\label{sec:background}

Intuitively, \emph{kernels} are similarity measures.
Formally, a function $k: \c{X} \times \c{X} \to \mathbb{R}$ on some set $\c{X}$ is a kernel if it satisfies the technical condition of being \emph{positive semi-definite}.
That is, if the matrix $\m{K}_{n n} = [k(x_i, x_j)]_{i, j}$ is positive semi-definite for all $n \in \mathbb{N}$ and all $x_1, \ldots, x_n \in \c{X}$.
Designing functions that represent a useful notion of similarity on a given $\c{X}$ and, at the same time, satisfy this technical condition is far from trivial.

When $\c{X} = \mathbb{R}^d$, the family of \emph{heat} kernels (also known as \emph{RBF}, \emph{squared expoential}, \emph{Gaussian}, or \emph{diffusion} kernels), and their extension known as the \emph{Matérn} kernels, are perhaps the most popular \citep{rasmussen2005,porcu2024}.
A key reason for this is they are essentially the simplest kernels satisfying:
\1* \emph{Stationarity}.
This means that shifting both arguments of the kernel by the same vector $\mathbf{s} \in \mathbb{R}^d$ leaves the kernel value unchanged: $k(\mathbf{x} + \mathbf{s}, \mathbf{x}' + \mathbf{s}) = k(\mathbf{x}, \mathbf{x}')$.
In many settings, e.g. in kernel ridge regression or Gaussian process regression, this translates into the following: training on a shifted dataset $\mathcal{D}_{\mathbf{s}} = (\mathbf{x}_1 + \mathbf{s}, y_1, \ldots, \mathbf{x}_n + \mathbf{s}, y_n)$ yields a model $f_{\mathbf{s}}(\cdot)$ related to the original model $f$ by the same shift $f_{\mathbf{s}}(\cdot) = f(\cdot + \mathbf{s})$.
This is a property so natural that it hardly requires further explanation.
\2* \emph{Smoothness}.
In its simplest form, this can be formalized as continuity: $k(\mathbf{x} + \v{\varepsilon}, \mathbf{x}') \approx k(\mathbf{x}, \mathbf{x}')$ for small $\v{\varepsilon} \in \mathbb{R}^d$.
More generally, in $\mathbb{R}^d$, this is formalized by differentiability---that is, we assume $k$ has a certain number of derivatives.\footnote{For Matérn kernels, the degree of differentiability can be non-integer and is controlled by a parameter $\nu \in \mathbb{R}_{>0}$.
See \citet{rasmussen2005} and \citet{porcu2024} for more details on Matérn kernels.}
Assuming some smoothness is crucial for \emph{generalization}: if $k$ could change arbitrarily even for nearby points, generalization to unseen data would be impossible.
\0*

These two properties---stationarity and smoothness---admit far-reaching generalizations, enabling the construction of useful kernels for input domains $\c{X}$ far beyond $\mathbb{R}^d$~\citep{azangulov2024a, azangulov2024b, borovitskiy2023}.
Notably, such generalizations include extensions of the heat and Matérn kernels~\citep{borovitskiy2020, borovitskiy2021}.\footnote{
Interestingly, generalizations based on stationarity and smoothness yield more successful kernel constructions than extending explicit formulas for the heat and Matérn kernels from $\mathbb{R}^d$ to other domains~\citep{feragen2015, dacosta2025}.
}
To illustrate the generalized notions of stationarity and smoothness in non-Euclidean settings, and how their interplay gives rise to generalized heat and Matérn kernels, we now discuss an example setting.
Specifically, let us assume $\c{X}$ is the node set of a graph---the setting that will turn out to be of key importance throughout the rest of the paper.
Formally, let $\c{X} = V$ be the set of nodes of an undirected graph without loops $G = (V, E)$ with edges $E \subseteq V \times V$ and with $V = \{1, \ldots, d\}$.
For simplicity, we assume $G$ is unweighted.

For a graph $G$, the analog of translations in $\mathbb{R}^d$ are \emph{automorphisms} of $G$, i.e., the set of all permutations $\sigma : V \to V$ preserving edge structure, i.e. such that
$(\sigma(x), \sigma(x')) \in E$ if and only if $(x, x') \in E$.
Therefore, a kernel $k: V \x V \to \R$ is called stationary, if and only if $k(\sigma(x), \sigma(x')) = k(x, x')$ for all such $\sigma$.
Intuitively, this reflects the fact that labeling nodes by indices $1, \ldots, d$ is arbitrary: for many tasks such as regression, relabeling the data should simply transform the model according to the same relabeling.
This is a concrete example of the general concept of being \emph{stationary under a group action}, which we revisit in~\Cref{sec:stationary:group_action}.

Smoothness on graphs is more subtle.
While it is possible to define proximity via the shortest path distance, this does not yield a quantitative smoothness scale analogous to the number of (generalized) derivatives.
Instead, smoothness can be characterized spectrally: the decay of coefficients in the basis of graph Laplacian eigenvectors generalizes the role of smoothness and Fourier decay in classical settings (see, e.g.,$\!$ \citet{chung1997}).

\citet{kondor2002} and \citet{borovitskiy2021} introduced general-purpose versions of the heat and Matérn kernels which respect the above-mentioned notions of stationarity and smoothness on graphs.
Let the connectivity of $G$ be encoded in the $d \times d$ adjacency matrix $\m{A}$, and define the graph Laplacian as $\m{\Delta} = \m{D} - \m{A}$, where $\m{D}$ is the diagonal degree matrix with entries $\m{D}_{ii} = \sum_j \m{A}_{ij}$.
It is well-known that $\m{\Delta}$ is a symmetric positive semi-definite matrix.
Let $\{\v{u}_i\}_{i=1}^d$ be an orthonormal basis of \emph{eigenvectors} of $\m{\Delta}$, with corresponding \emph{eigenvalues} $\lambda_i \geq 0$.
We can reinterpret them as \emph{eigenfunctions} $u_i: V \to \R$ with $u_i(x) = \del{\v{u}_i}_x$.
We denote the heat kernel on $V$ by $k_{\infty,\kappa}(\cdot,\cdot)$ and the Matérn kernel by $k_{\nu, \kappa}(\cdot,\cdot)$, where $\kappa,\nu > 0$.

Then, according to \citet{borovitskiy2021}, both can be computed via
\[
\label{eqn:graph_kernels}
k_{\nu, \kappa}(x, x')
=
\frac{1}{C_{\nu, \kappa}}
\sum_{\ell=1}^{d}
\Phi_{\nu, \kappa}(\lambda_\ell) u_{\ell}(x) u_{\ell}(x') ,
\qquad
\Phi_{\nu, \kappa}(\lambda) =
\begin{cases}
\exp\left(-\frac{\kappa^2}{2} \lambda \right) & \text{if } \nu = \infty , \\[0.15cm]
\left( \frac{2\nu}{\kappa^2} + \lambda \right)^{-\nu} & \text{if } \nu < \infty.
\end{cases}
\]
Here, $x, x'$ are nodes of the graph, and $C_{\nu,\kappa} > 0$ is a normalization constant ensuring $\frac{1}{\abs{\c{X}}}\sum_{x \in \c{X}} k_{\nu,\kappa}(x,x) = 1$.

\section{Stationary Kernels}
\label{sec:stationary}

Intuitively, since the specific labeling $1, \ldots, 2n$ of elements in a matching is arbitrary, a kernel on matchings should remain unchanged when both of its arguments are relabeled in the same way.
We formalize this intuition in~\Cref{sec:stationary:group_action} using the framework of stationarity under a group action.
Examining the relevant group action further in~\Cref{sec:stationary:homogeneous}, we naturally identify the structure of matchings as a homogeneous space.
This structure allows us to explicitly characterize the class of stationary kernels on matchings in~\Cref{sec:stationary:description}.

\subsection{Group Action and Stationarity under a Group Action}
\label{sec:stationary:group_action}

As mentioned above, the labeling $1, \ldots, 2n$ of elements in a matching is arbitrary.
This arbitrariness is an example of a \emph{symmetry}.
More generally, a symmetry of a set $\c{X}$ is any bijective map $g: \c{X} \to \c{X}$ that preserves its relevant structure.
The choice of structure determines what transformations qualify as symmetries.

Formally, we represent the symmetries of $\c{X}$ by considering the action of a group $G$ on $\c{X}$, expressed as a \emph{group action} $\lacts: G \times \c{X} \to \c{X}$.
In typical notation, $g \lacts x$ denotes the action of $g \in G$ on $x \in \c{X}$.
Given a group $G$ acting on $\c{X}$ via $\lacts$, a kernel $k: \c{X} \times \c{X} \to \mathbb{R}$ is called \emph{stationary under the group action} if and only if
\[
\label{eqn:stationary_kernel}
k(g \lacts x, g \lacts x') = k(x, x')
\qquad
\text{for all } x, x' \in \c{X},\ g \in G.
\]

For the set $\c{X}_n$ of matchings of size $n$, the relevant symmetries are permutations $\sigma: \{1, \ldots, 2n\} \to \{1, \ldots, 2n\}$ \citep{ceccherinisilberstein2008}.
The collection of all such permutations is called the \emph{symmetric group}~$S_{2n}$ of degree $2n$.
For $\sigma \in S_{2n}$ and a matching $x \in \c{X}_n$, we define the action by
\begin{equation}
\label{eqn:s2n_action}
\sigma \lacts x := \{ \{ \sigma(i_1), \sigma(i_2) \}, \ldots, \{ \sigma(i_{2n-1}), \sigma(i_{2n}) \} \}
,
\end{equation}
where $x = \{\{i_1, i_2\}, \ldots, \{i_{2n-1}, i_{2n}\}\} \in \c{X}_n$.

Combining this group action with the invariance condition~\eqref{eqn:stationary_kernel} defines stationarity for kernels on matchings.
We now describe how this group action offers an alternative, algebraic perspective on $\c{X}_n$: the space of matchings is a \emph{homogeneous space}.
This perspective is key to the explicit characterization of stationary kernels in~\Cref{sec:stationary:description}.

\subsection{The Set of Matchings \texorpdfstring{$\c{X}_n$}{} as a Homogeneous Space}
\label{sec:stationary:homogeneous}
Recall that the group action is said to be \emph{transitive} if, for all $x,x' \in \c{X}$, there exists $g \in G$ such that $g\lacts x = x'$.
Intuitively, this means the action is rich enough to transform any point into any other point.
In this case, $\c{X}$ is called a \emph{homogeneous space}, and it may be identified with the \emph{coset space} $G/H$, where $H$ is the \emph{isotropy subgroup} fixing some basepoint $x_0 \in \c{X}$:
\[
    H = \left\{ h \in G : h \lacts x_0 = x_0 \right\}.
\]
The elements of $G/H$ are \emph{cosets} $gH = \{g h : h \in H\}$.
The natural bijective correspondence between $G/H$ and $\c{X}$ maps the coset $gH$ to 
the element $g\lacts x_0$.
Here, $G$ acts on $G/H$ by left multiplication $g \lacts (g' H) = (g g') H$.

In the case of matchings $\c{X}_n$ with the action~\eqref{eqn:s2n_action}, let us fix the basepoint as $x_0 = \{\{1, 2\}, \ldots, \{2n-1, 2n\}\}$.
Any permutation $h$ that leaves $x_0$ unchanged (that is, $h \lacts x_0 = x_0$) can either swap the elements within a given pair (that is, map $(i, i+1)$ to $(i+1, i)$ for some odd $i$) or permute entire pairs among themselves (that is, send the $j$th pair $(2j-1, 2j)$ to the $k$th pair $(2k-1, 2k)$).
Together, these generate the isotropy subgroup $H$, which is isomorphic to the \emph{hyperoctahedral group} $H_n$, the symmetry group of the $n$-dimensional hypercube~\citep{ceccherinisilberstein2008}.
Thus, the set of matchings can be identified with the coset space $S_{2n}/H_n$, and this perspective enables the application of general theory of stationary kernels on homogeneous spaces.

\paragraph{Functions on $G/H$ and functions on $G$.}
An important consequence of viewing a homogeneous space as a quotient $G/H$ is that any function $\phi: G/H \to \mathbb{R}$ can be \emph{lifted} to a function $\widetilde{\phi}: G \to \mathbb{R}$ on the group $G$, constant on cosets, by defining
\[
\label{eqn:lifting}
\widetilde{\phi}(g) = \phi(gH)
,
&&
g \in G
.
\]
Conversely, any function $\widetilde{\psi}: G \to \mathbb{R}$ can be \emph{projected} down to a function on $\psi: G/H \to \R$ via averaging:
\[
\label{eqn:projection}
\psi(gH) = \frac{1}{|H|} \sum_{h \in H} \widetilde{\psi}(g h)
,
&&
g \in G
.
\]
Notably, the lifting from $\phi$ to $\widetilde{\phi}$ is information-preserving, while the projection from $\widetilde{\psi}$ to $\psi$ typically loses information due to averaging.
The constructions in~\Cref{eqn:lifting,eqn:projection} extend naturally to kernels: any kernel $k: G/H \times G/H \to \mathbb{R}$ can be lifted to a kernel $\widetilde{k}: G \times G \to \mathbb{R}$, and, conversely, any kernel on the group can be projected to the quotient.
Importantly, positive semidefiniteness is preserved in both directions.

\subsection{The Class of Stationary Kernels}
\label{sec:stationary:description}

As discussed in~\Cref{sec:stationary:homogeneous}, the set of matchings $\c{X}_n$ of size $n$ can be naturally regarded as the homogeneous space $S_{2n}/H_n$, where $S_{2n}$ is the symmetric group of degree $2n$ and $H_n$ is the hyperoctahedral group.
Thanks to this identification, we can apply the general theory of stationary kernels on homogeneous spaces as developed in~\citet{yaglom1961}.
Moreover, in our case, the pair $(S_{2n}, H_n)$ satisfies the technical assumptions of a \emph{Gelfand pair}~\citep{ceccherinisilberstein2008}---we omit the discussion of this notion here---further simplifying the general structure of stationary kernels on $\c{X}_n$.
This leads to the following description.

\begin{restatable}{theorem}{GeneralStationaryKernel}
\label{thm:general_stationary_kernel}
A kernel $k$ on the homogeneous space $\c{X}_n = S_{2n} / H_n$ is stationary \textup{(}in the sense of~\Cref{eqn:stationary_kernel}\textup{)} if and only if it can be written as
\[
k(\sigma H_n , \pi H_n) = \sum_{\v{\mu} \vdash n} a_{\v{\mu}}\, \phi_{\v{\mu}} (\pi^{-1} \sigma),
\]
where the sum runs over $\v{\mu} \vdash n$, all \emph{integer partitions} $\v{\mu} = (\mu_1, \ldots, \mu_s)$ of $n$; $a_{\v{\mu}} \in \mathbb{R}_{\geq 0}$; and $\phi_{\v{\mu}}: S_{2n} \to \mathbb{R}$ are the \emph{zonal spherical functions} on the homogeneous space $S_{2n}/H_n$ described in~\Cref{thm:character_projection} below.
\end{restatable}
\begin{proof}
This is a version of the Bochner--Godement theorem (see, for example, Section~1.1, Theorem~1.2 in~\citet{faraut2007}), with the general representation-theoretic indexing of zonal spherical functions specialized to the combinatorial partition-based indexing for matchings.
Please see \Cref{appendix:stationary_kernels} for details.
\end{proof}

\Cref{thm:general_stationary_kernel} generalizes the classical Bochner’s theorem, which states that every continuous positive semidefinite function is proportional to the inverse Fourier transform of a unique probability measure.
Here, the expansion into a linear combination of zonal spherical functions plays the role of inverse Fourier transform, while specifying a probability measure amounts to assigning the non-negative coefficients $a_{\v{\mu}}$ in the expansion.

The \emph{zonal spherical functions} from~\Cref{thm:general_stationary_kernel} admit the following description in terms of \emph{characters} of the symmetric group $S_{2 n}$ which, in their turn, posses well-known analytical expressions.

\begin{restatable}{proposition}{CharacterProjection}
\label{thm:character_projection}
Let $\v{\rho} = (\rho_1, \ldots, \rho_s)$ be a partition of $n$.
For the zonal spherical function $\phi_{\v{\rho}}: S_{2 n} \to \mathbb{R}$, we have
\[
\phi_{\v{\rho}} (\sigma) = \frac{1}{|H_n|^2} \sum_{\pi_1 , \pi_2 \in H_n} \chi^{(2\v{\rho})} (\pi_1 \sigma \pi_2),
\]
where $2\v{\rho} = (2\rho_1, \ldots, 2\rho_s)$ is a partition of $2n$ and $\chi^{(2\v{\rho})}$ is the corresponding \emph{character} of the group $S_{2n}$.
\end{restatable}
\begin{proof}
See \Cref{appendix:stationary_kernels} for the proof and further discussion on zonal spherical functions.
\end{proof}

Thus, zonal spherical functions can be viewed as two-sided $H_n$-averages of the corresponding characters of $S_{2n}$.
Analytic expressions for these characters are available, for example, via the Murnaghan--Nakayama rule (see Theorem~7.17.1 in~\citet{stanley2001}).
However, zonal spherical function computation using averaging and these analytic expressions is computationally expensive---the issue we return to in~\Cref{sec:computation}.

\begin{restatable}{remark}{StationaryAreBiInvariantProjection}
It can be shown that the stationary kernels on $S_{2n}/H_n$ are precisely those obtained as the projection (i.e., double coset averaging) of bi-invariant kernels on the group $S_{2n}$.
Here, a kernel $k: S_{2n} \x S_{2n} \to \R$ is bi-invariant if $k(\pi \sigma, \pi \sigma') = k(\sigma \pi, \sigma' \pi) = k(\sigma, \sigma')$ for all $\sigma, \sigma', \pi \in S_{2n}$.
\end{restatable}
\begin{proof}
See \Cref{appendix:stationary_kernels}.
\end{proof}

\section{Heat and Matérn Kernels}
\label{sec:kernels:heat_matern}

The class of stationary kernels described in the previous section is far too broad for practical applications.
For instance, setting $a_{\v{\mu}} = 1$ for all $\v{\mu}$ yields the delta kernel $k(x, x') = \mathbf{1}_{x = x'}$ which precludes any generalization.
To define useful kernels, it is necessary to also impose a degree of smoothness.
This is precisely where the heat and Matérn kernels become relevant.
As discussed in~\Cref{sec:background}, constructing such kernels on discrete spaces is subtle, as it is difficult to define a quantitative scale of differentiability on such spaces.

Here, we address this challenge indirectly, circumventing it.
We first describe heat and Matérn kernels on the group $S_{2n}$ where an appropriate notion of smoothness can be obtained by regarding it as its \emph{Caley graph}.
We then obtain kernels on the set of matchings by canonically projecting these group kernels onto the matchings space, as described at the end of~\Cref{sec:stationary:homogeneous}.
Finally, we show that these projected kernels can also be interpreted as heat and Matérn kernels defined directly on a certain associated \emph{quotient graph}---see details in~\Cref{appendix:quotient}.

\subsection{Symmetric Group}

We begin by introducing a notion of distance between two elements $g, g'$, of a finite abstract group $G$.
Intuitively, this distance can be understood as the minimum number of elementary operations needed to transform $g$ into $g'$.
Formally, let $\c{S} = \{g_1, \ldots, g_m\}$ be a set of \emph{generators} for the group $G$---i.e, such a set that any $g \in G$ can be written as $g = g_{i_1} \cdots g_{i_\ell}$ for some $\ell \in \N$ and some sequence $g_{i_1}, \ldots, g_{i_\ell} \in \c{S}$ (repetitions allowed).
Given a set $\c{S}$, we can  define the distance $d(g, g')$ as the minimal length $\ell$ needed to express $g^{-1}g'$ as a product of generators from~$\c{S}$.
It turns out \citep{kondor2008} that this is the shortest-path distance between $g$ and $g'$ in the \emph{Cayley graph} of group $G$ generated by $\c{S}$;
that is, the graph $\Gamma(G, \c{S})$ whose nodes are the group elements $V = G$ and whose edges are pairs of the form $(g, g_i g)$ where $g \in G$ and $g_i \in \c{S}$.

Now, consider the specific case $G = S_{2n}$ of the symmetric group of permutations of $2n$ elements, which acts as the group of symmetries of matchings.
There are several possible choices for the generating set~$\c{S}$.
However, not all of them behave equally well---see Corollary~5.2.2 and the subsequent discussion in~\citet{kondor2008}.
A particularly appealing choice is the set $\c{S}$ that consists of all transpositions, i.e., all permutations that swap exactly two elements $1 \leq i, j \leq 2n$.
With this choice, the resulting heat and Matérn kernels on $\Gamma(G,\c{S})$---as defined by~\Cref{eqn:graph_kernels}---are bi-invariant, meaning
\[
k(\pi \sigma, \pi \sigma') = k(\sigma \pi, \sigma' \pi) = k(\sigma, \sigma')
\qquad
\text{for all }
\sigma, \sigma', \pi \in S_{2n}
.
\]
This is shown for heat kernels in~\citet{kondor2008}.
However, the exact same argument applies to Matérn~kernels.

The argument above uses the definition of  heat and Matérn kernels in terms of eigenfunctions and eigenvalues of the graph Laplacian associated with the Cayley graph, as described in~\Cref{sec:background}.
Importantly, by leveraging the algebraic structure of $S_{2n}$, one can reformulate the standard expansion by grouping together eigenfunctions corresponding to the same eigenvalue, thereby dramatically reducing the number of terms and making it simpler.
This leads to the character-theoretic expansion for kernels on $S_{2n}$, as presented below.

\begin{restatable}{proposition}{KernelsAreCharactersSum}
\label{thm:kernels_character_sum}
Let $\Phi: \mathbb{R}_{\geq 0} \rightarrow \mathbb{R}_{\geq 0}$, e.g., $\Phi(\lambda) = e^{-\frac{\kappa^2}{2}\lambda}$ or $\Phi(\lambda) = \left(\frac{2 \nu}{\kappa^2} + \lambda\right)^{-\nu}$.
Then, for any kernel of the form
\[
\label{eqn:s2n_eigenvector_expansion}
k(\sigma, \sigma')
=
\frac{1}{C(\Phi)}
\sum_{\pi \in S_{2 n}}
\Phi(\lambda_{\pi}) u_{\pi}(\sigma) \overline{u_{\pi}(\sigma')}
,
&&
\sigma, \sigma' \in S_{2 n}
,
\]
we have
\[
\label{eqn:s2n_character_expansion}
k(\sigma, \sigma')
=
\frac{1}{C(\Phi)}
\sum_{\v{\rho} \vdash 2n}
\Phi(\lambda_{\v{\rho}})
\frac{d_{\v{\rho}}}{|S_{2 n}|}
\chi^{(\v{\rho})}((\sigma')^{-1} \sigma)
,
&&
\sigma, \sigma' \in S_{2 n}
,
\]
where the sum is over all partitions $\v{\rho}$ of $2n$, $\lambda_{\v{\rho}}$ is the graph Laplacian eigenvalue corresponding to the respective character  $\chi^{(\v{\rho})}$ of $S_{2 n}$, and $d_{\v{\rho}}/|S_{2 n}|$ is a normalization constant ($d_{\v{\rho}}$ is the \emph{dimension of the representation} of $S_{2 n}$ corresponding to the character indexed by $\v{\rho}$).
\end{restatable}
\begin{proof}
See Appendix~\ref{appendix:heat_and_matern_kernels}.
The expression comes from reinterpreting the sum $\sum u_{\pi}(\sigma) u_{\pi}(\sigma')$ over all $\pi$ with the same eigenvalue as proportional to a character of $S_{2 n}$ or a sum of a few such characters.
\end{proof}

By substituting $\Phi_{\nu, \kappa}$ from~\Cref{eqn:graph_kernels} in place of~$\Phi$ in the above theorem, we immediately obtain a character-based expansion of the heat and Matérn kernels on the symmetric group $S_{2 n}$.

\subsection{Matchings}

To obtain heat and Matérn kernels on the set of matchings~$\c{X}_n$, we project the corresponding kernels on the symmetric group~$S_{2n}$ onto the quotient~$S_{2n}/H_n$, following the projection idea described at the end of~\Cref{sec:stationary:homogeneous}.
It turns out that averaged characters in the expansion~\eqref{eqn:s2n_character_expansion} either vanish or coincide with the zonal spherical functions of $S_{2n}/H_n$ from~\Cref{sec:stationary}, which leads to the main result of this section:

\begin{restatable}{proposition}{HeatMaternKernelMatchings}\label{thm:heat_matern_kernel_matchings}
Projecting---with projection defined in~\Cref{eqn:projection}---of the heat and Matérn kernels on $S_{2n}$ onto the quotient $S_{2n}/H_n$ results in the kernels stationary under the action of $S_{2n}$ which are given by
\[
\label{eqn:heat_matern_matchings}
k_{\nu, \kappa}(\sigma H_n, \sigma' H_n)
=
\frac{1}{C_{\nu, \kappa}}
\sum_{\v{\rho} \vdash n}
\Phi_{\nu, \kappa}(\lambda_{2\v{\rho}})
\,
\frac{d_{2\v{\rho}}}{|S_{2n}|}\, \phi_{\v{\rho}} ((\sigma')^{-1} \sigma),
\qquad
\Phi_{\nu, \kappa}(\lambda) =
\begin{cases}
\exp\left(-\frac{\kappa^2}{2} \lambda \right) & \text{if } \nu = \infty , \\
\left( \frac{2\nu}{\kappa^2} + \lambda \right)^{-\nu} & \text{if } \nu < \infty.
\end{cases}
\]
where the sum is over all partitions $\v{\rho}$ of $n$; $2\v{\rho} = (2\rho_1, \ldots, 2\rho_s)$ is the respective partition of $2n$; $d_{2\v{\rho}}$ is as in~\Cref{thm:kernels_character_sum}; $\lambda_{2\v{\rho}}$ is the Cayley graph Laplacian eigenvalue corresponding to the character $\chi^{(2\v{\rho})}$; and $\phi_{\v{\rho}}$ is the zonal spherical function obtained by averaging the respective character as in~\Cref{thm:character_projection}.
\end{restatable}

\begin{proof}
Please refer to \Cref{appendix:heat_and_matern_kernels}.
\end{proof}

\begin{remark}
Alternatively, instead of averaging from $S_{2n}$, one can consider the quotient Cayley graph corresponding to the space of matchings $\c{X}_n$ and define heat and Matérn kernels directly on that graph.
This construction leads to the same kernels as above.
See \Cref{appendix:quotient} for further discussion.
\end{remark}

\section{Computation}
\label{sec:computation}

Here, we develop algorithms for computing the heat and Matérn kernels introduced above.
We begin in~\Cref{sec:computation:reduction} by showing that kernel evaluation---either exact or approximate---ultimately reduces to computing zonal spherical functions.
In~\Cref{sec:computation:existing}, we review existing algorithms for computing zonal spherical functions, emphasizing their super-exponential cost.
In~\Cref{sec:computation:ours}, we present our own sub-exponential algorithm, which enables kernel computations for much larger values of~$n$.
The key insight is that each zonal spherical function depends only on a generalized distance to a fixed basepoint, and the number of distinct generalized distance values is sub-exponential in~$n$.
A discussion, including computational complexity estimates, practical considerations, and numerical experiments, is deferred to the next section.

\subsection{From Kernels to Zonal Spherical Functions}
\label{sec:computation:reduction}

As observed in~\Cref{sec:kernels:heat_matern}, the heat and Matérn kernels can be interpreted as kernels on the node set of a certain graph.
In principle, this would allow for evaluation via~\Cref{eqn:graph_kernels}, provided that one computes the eigenvalues and eigenvectors of the corresponding graph Laplacian.
However, the graph Laplacian matrix has size $\abs{\c{X}_n} \times \abs{\c{X}_n}$, where $\abs{\c{X}_n} = \frac{(2n)!}{2^n n!}$ is the number of matchings---already super-exponential in~$n$.
As a result, full spectral decomposition quickly becomes computationally intractable, even for moderate values of~$n$.
This motivates the search for alternative approaches.

Alternatively, we may exploit the expansion in~\Cref{eqn:heat_matern_matchings}.
Here, the number of terms is equal to the number of integer partitions of $n$, denoted $p(n)$.
This is a classical combinatorial quantity, and it is well known that $p(n)$ is sub-exponential in $n$, growing roughly as $O(\exp(\pi\sqrt{n}))$.\footnote{It is sub-exponential because of the squared root in the exponent, which makes the resulting expression grow much slower than $\exp(a n)$ for any $a > 0$---as opposed to the expression $\frac{(2n)!}{2^n n!}$ which grows faster than $\exp(a n)$ for any $a > 0$.}
Thus, the number of terms in~\Cref{eqn:heat_matern_matchings} is far more manageable than in~\Cref{eqn:graph_kernels}.
Moreover, not all terms in~\Cref{eqn:heat_matern_matchings} contribute equally to the kernel.
As we discuss in~\Cref{sec:discussion}, in practical applications the sum can be efficiently truncated to fewer than a hundred terms while still providing a good approximation.
We refer to this process as \emph{truncation}, and we will work throughout this section with \emph{truncated kernels} of form
\begin{equation}
\label{eqn:TruncatedHeatMaternMatchingKernels}
k_{\c{R}}(\sigma_1 H, \sigma_2 H) = \sum_{\v{\rho} \in \c{R}} \Phi(\lambda_{2\v{\rho}}) \frac{d_{2\v{\rho}}}{(2n)!} \, \phi_{\v{\rho}} (\sigma_2^{-1} \sigma_1),
\end{equation}
where $\c{R}$ is a chosen subset of partitions.
Importantly, according to the general form of stationary kernels in \Cref{thm:general_stationary_kernel}, the truncated kernel in~\Cref{eqn:TruncatedHeatMaternMatchingKernels} remains a well-defined stationary kernel for any choice~of~$\c{R}$.

In the expansion above, three quantities must be computed for each partition $\v{\rho}$: the Laplacian eigenvalue $\lambda_{2\v{\rho}}$, the representation dimension $d_{2\v{\rho}}$, and the zonal spherical function $\phi_{\v{\rho}}$.

\begin{itemize}
    \item \emph{Representation dimensions $d_{2\v{\rho}}$} can be computed analytically via the classical \emph{hook length formula}.
    This algebraic formula can be easily implemented with negligible cost---$O(n)$ per partition---as detailed in~\Cref{appendix:eigenvalues_and_dim}.
    The hook length formula is standard and, for example, is available in the \texttt{SageMath} mathematical software system \citep{sage}.
    \item \emph{Laplacian eigenvalues $\lambda_{2\v{\rho}}$} correspond precisely to those for kernels on the symmetric group $S_{2n}$, studied in detail by~\citet{kondor2008}.
    Algorithmically, they can be calculated using the \emph{Murnaghan--Nakayama rule}, as described in~\Cref{appendix:eigenvalues_and_dim}, with complexity $O(n^2)$ per partition.
    While we are not aware of existing readily usable implementations beyond our own, the procedure is straightforward.
    \item \emph{Zonal spherical functions $\phi_{\v{\rho}}$} are the key computational challenge.
    In the following sections, we review existing approaches for their computation, discuss their bottlenecks, and introduce our own, substantially more efficient algorithm.
\end{itemize}

\subsection{Zonal Spherical Functions: Existing Algorithms}
\label{sec:computation:existing}

The most direct method for computing zonal spherical functions is to use~\Cref{thm:character_projection}, which expresses them as averages of characters of the symmetric group~$S_{2n}$.
Even if one optimistically assumes that computing the characters themselves requires only $O(1)$ time per evaluation---which is not the case---this approach still requires averaging over $|H|^2$ terms, where $H$ is the hyperoctahedral group, whose size is $|H| = 2^n n! \gg \exp(n)$.
Consequently, the overall computational cost remains super-exponential in~$n$.

Another available approach is to use the explicit formula for zonal spherical functions, as described, for example, in~\citet{ceccherinisilberstein2008}---we do not reproduce it in the main text for simplicity, please refer to \Cref{appendix:zsf_naive} for details.
However, direct application of this formula leads to a time complexity of $O\left( |\c{X}_n|\, \rho_1'! \cdots \rho_r'! \cdot n \right)$ for evaluating a single zonal spherical function $\phi_{\v{\rho}}$, where $\v{\rho}' \vdash n$ is the conjugate partition (i.e., $\rho_j' = \lvert \{ i : \rho_i \geq j \} \rvert$).
This complexity exceeds $|\c{X}_n|$ operations, and thus is also super-exponential in $n$.
Although the explicit formula is thus impractical for actual computation, it yields the following structural property, which will be useful below.

\begin{proposition}
\label{thm:zsf_generalized_distance}
Let $\phi$ be zonal spherical function.
Then, its value $\phi(x)$ depends only on \emph{generalized distance} $d(x, x_0)$ defined below.
Namely, if $x, y$ are two matchings such that $d(x, x_0) = d(y, x_0)$, then $\phi(x) = \phi(y)$.
\end{proposition}
\begin{proof}
Directly follows from Equation~(11.5) from~\citet{ceccherinisilberstein2008}.
\end{proof}

\begin{figure}[t]
\centering
\begin{tikzpicture}[
    dot/.style={circle, fill, inner sep=1.5pt},
    every label/.style={font=\sffamily}
]
    \draw (1.25,1.25) circle (1.25);
    \node[dot, label=below:1] at (1.25,2.5) {};
    \node[dot, label=above:2] at (1.25,0) {};

    \draw (4,0) rectangle (6.5,2.5);
    \node[dot, label=below right:3] at (4,2.5) {};
    \node[dot, label=above right:4] at (4,0) {};
    \node[dot, label=below left:5] at (6.5,2.5) {};
    \node[dot, label=above left:6] at (6.5,0) {};

    \draw (8,0) rectangle (10.5,2.5);
    \node[dot, label=below right:7] at (8,2.5) {};
    \node[dot, label=above right:8] at (8,0) {};
    \node[dot, label=below left:9] at (10.5,2.5) {};
    \node[dot, label=above left:10] at (10.5,0) {};
\end{tikzpicture}
\caption{If $x = \{ \{1, 2\}, \{3, 4\}, \{5, 6\}, \{7, 8\}, \{9, 10\} \}$ and $y = \{ \{1, 2\}, \{3, 5\}, \{4, 6\}, \{7, 9\}, \{8, 10\} \}$, then the union $x \cup y$ forms the edge set of the disconnected graph shown here.
It consists of two cycles of length four and one cycle of length two; thus, the generalized distance in this case is $d(x, y) = (2, 2, 1)$.}
\label{fig:generalized-distance}
\end{figure}
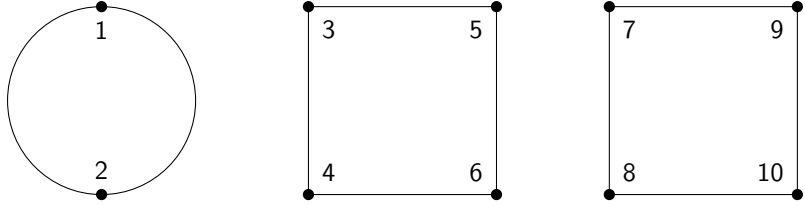

We now define the \emph{generalized distance} $d$ referenced above (see~\Cref{fig:generalized-distance} for an illustrative example).
\begin{definition}
Given two matchings $x, y \in \c{X}_n$, construct an undirected graph on the vertex set $\{1, 2, \ldots, 2n\}$ with edge set $x \cup y$; that is, $\{i, j\}$ is an edge if it appears in either $x$ or $y$.
It is straightforward to verify that the connected components of this graph are disjoint cycles, each of even length~\citep{ceccherinisilberstein2008}.
Let $2\mu_1 \geq 2\mu_2 \geq \cdots \geq 2\mu_s$ denote the lengths of these cycles.
Then the generalized distance between $x$ and $y$ is then defined to be the partition $d(x, y) = \v{\mu} = (\mu_1, \mu_2, \ldots, \mu_s) \vdash n$.
\end{definition}

Crucially, \Cref{thm:zsf_generalized_distance} implies that the number of distinct values a zonal spherical function may take is at most the number of possible generalized distances---that is, the number of partitions $p(n)$ of $n$.
As discussed above, $p(n)$ is dramatically smaller than the number of matchings $|\c{X}_n|$.
This insight suggests that there could exist algorithms for computing all values of a zonal spherical function whose cost depends only on $p(n)$, rather than on $|\c{X}_n|$.
In the next section, we develop an algorithm that realizes this possibility.

\subsection{Zonal Spherical Functions: Our Algorithm}
\label{sec:computation:ours}

We now describe our new algorithm for computing zonal spherical functions.
Our approach is based on the connection between zonal spherical functions and \emph{zonal polynomials}~\citep{james1960, james1961, james1968, macdonald1998}, a class of homogeneous symmetric polynomials widely used in multivariate statistics~\citep{muirhead1982}.

Let $\v{\rho} \vdash n$ and $m \geq n$.
One way to define the zonal polynomial $\mathcal{C}_{\v{\rho}}: \R^{m} \to \R$ is via its expansion in the basis of \emph{monomial symmetric polynomials} $m_{\v{\kappa}}$ \citep{james1968}:
\[
\mathcal{C}_{\v{\rho}}(z_1, \ldots, z_m) = \sum_{\v{\kappa} \leq \v{\rho}} c_{\v{\rho}, \v{\kappa}} \, m_{\v{\kappa}}(z_1, \ldots, z_m),
&&
m_{\v{\kappa}}(z_1, \ldots, z_m)
=
\!\!\!\!\!\!\!\!\!\!\!\!\!
\sum_{\substack{1 \leq i_1, \ldots, i_{\abs{\v{\kappa}}} \leq m \\ \text{all indices distinct}}} z_{i_1}^{\kappa_1} \cdots z_{i_{\abs{\v{\kappa}}}}^{\kappa_{\abs{\v{\kappa}}}}
\]
where the first sum is over partitions $\v{\kappa}$ of $n$ which are lexicographically not larger than $\v{\rho}$.
The coefficients $c_{\v{\rho}, \v{\kappa}}$ are uniquely determined by recurrence relations detailed in~\Cref{appendix:zsf_our_algorithm}.
Moreover, these coefficients can be efficiently computed using a dynamic programming algorithm, as explained in the same~\Cref{appendix:zsf_our_algorithm}.

It was shown by~\citet{bergeron1992, james1961} that the coefficients of zonal polynomials in \emph{another basis} correspond---up to a simple normalization constant---to the values of zonal spherical functions, as stated below.

\begin{restatable}{theorem}{ZpCoefsAreZsf}
Let $\v{\rho}\vdash n$ be a partition, $\mathcal{C}_{\v{\rho}}$ the corresponding zonal polynomial of $m\geq n$ variables, and $\phi_{\v{\rho}}$ the corresponding zonal spherical function.
Then, there exists a constant $c = c(\v{\rho})$ such that
\[
C_{\v{\rho}}(z_1, \ldots, z_m) = \sum_{\v{\mu} \vdash n} \left(c \cdot \phi_{\v{\rho}}(\v{\mu}) |A_{\v{\mu}}| \right) p_{\v{\mu}}(z_1, \ldots, z_m),
\]
where $A_{\v{\mu}} = \{x \in \c{X}_n \mid d(x, x_0) = \v{\mu} \}$ is the \emph{generalized sphere} of radius $\v{\mu}$ whose size $|A_{\v{\mu}}|$ can be computed in $O(n)$ time and the \emph{the products of power sum symmetric polynomials} $p_{\v{\mu}}$ are defined~by
\[
p_{\v{\mu}} (z_1, \ldots, z_m) = \prod_{j=1}^{\abs{\v{\mu}}} p_{\mu_j}(z_1, \ldots, z_m),
&&
p_d (z_1, \ldots, z_m) = z_1^d + \ldots + z_m^d
.
\]
\end{restatable}
\begin{proof}
See \Cref{appendix:zsf_our_algorithm}.
\end{proof}

Given an expansion of the zonal polynomial in the basis of products of power sum symmetric polynomials,
\[
\mathcal{C}_{\v{\rho}} (z_1, \ldots, z_m) = \sum_{\v{\mu} \vdash n} b_{\v{\rho}, \v{\mu}}\, p_{\v{\mu}}(z_1, \ldots, z_m),
\]
the values of the zonal spherical function $\phi_{\v{\rho}}$ can be easily recovered.
Specifically, using the fact that $\phi_{\v{\rho}}((1,\ldots, 1)) = 1$ and $|A_{(1,\ldots, 1)}| = 1$, we have
\[
\phi_{\v{\rho}}(\v{\mu}) = \frac{b_{\v{\rho}, \v{\mu}}}{c \cdot |A_{\v{\mu}}|} = \frac{b_{\v{\rho}, \v{\mu}}}{b_{\v{\rho}, (1, \ldots, 1)}  \cdot |A_{\v{\mu}}|} .
\]
Thus, the key idea of our algorithm is to obtain the required coefficients $b_{\v{\rho}, \v{\mu}}$ by performing a change of basis from monomial symmetric polynomials to products of power sum symmetric polynomials.

It remains to understand how to perform the change of basis.
Let $T_{\v{\kappa}, \v{\mu}}$ denote the transition matrix between monomial symmetric polynomials and products of power sum symmetric polynomials, so that
\[
m_{\v{\kappa}}(z_1, \ldots, z_m) = \sum_{\v{\mu} \vdash n} T_{\v{\kappa}, \v{\mu}}\, p_{\v{\mu}}(z_1, \ldots, z_m).
\]
There exist several algorithms for computing $T_{\v{\kappa}, \v{\mu}}$.
In this work, we adopt a modified version of the approach described by~\citet{merca2015}.
We describe this algorithm and our modifications in~\Cref{appendix:zsf_our_algorithm}.

\begin{algorithm}[t]
\caption{Zonal Spherical Function via Zonal Polynomials}
\label{algorithm:ZSFUsingZP}
\begin{algorithmic}
\State \textbf{Input:} Partitions $\v{\rho}, \v{\mu} \vdash n$ and precomputed transition matrix entries $T_{\v{\kappa}, \v{\mu}}$ for all relevant $\v{\kappa}, \v{\mu} \vdash n$
\State \textbf{Output:} Value of the zonal spherical function $\phi_{\v{\rho}} (\v{\mu})$
\State Compute the coefficients $c_{\v{\rho}, \v{\kappa}}$ for all $\v{\kappa} \vdash n$ using the algorithm from~\Cref{appendix:zsf_our_algorithm}
\State Compute $b_{\v{\rho}, \v{\mu}} \gets \sum_{\v{\kappa} \vdash n} c_{\v{\rho}, \v{\kappa}}\, T_{\v{\kappa}, \v{\mu}}$
\State Compute $\phi_{\v{\rho}}(\v{\mu}) \gets \frac{b_{\v{\rho}, \v{\mu}}}{b_{\v{\rho}, (1, \ldots, 1)} \cdot |A_{\v{\mu}}|}$
\end{algorithmic}
\end{algorithm}

The full procedure is summarized in~\Cref{algorithm:ZSFUsingZP}.
Note that the algorithm for computing $c_{\v{\rho}, \v{\kappa}}$ produces all coefficients for a fixed $\v{\rho}$.
The computation of $T_{\v{\kappa}, \v{\mu}}$ is based on recursive calls; in practice, computing it for a single value of~$\v{\mu}$ is often as costly as computing it for all $\v{\mu}$.
Therefore, it is convenient to precompute the entire transition matrix in advance.
Alternatively, as in our software implementation, one can build $T_{\v{\kappa}, \v{\mu}}$ on demand by caching the previously computed values; the cost of precomputing each entry is then amortized over all algorithm calls.
However, this latter strategy is more complicated to analyze theoretically, so for the purpose of complexity estimates, we focus on the version that precomputes the full matrix.
In either case, the crucial point is that the change of basis occurs in the $p(n)$-dimensional space, and thus the overall cost depends on $p(n)$ rather than $|\c{X}_n|$.
A detailed complexity analysis is provided in the next section.

\section{Discussion} \label{sec:discussion}

We begin in~\Cref{sec:discussion:complexity} by comparing---both theoretically and through empirical results---the computational complexity that arises when using our proposed method with the existing baselines.
In~\Cref{sec:discussion:truncation}, we present heuristics for selecting a subset of partitions to truncate the kernel expansion, and analyze the associated approximation error.
Next,~\Cref{sec:discussion:hyperparameters} addresses practical considerations for choosing hyperparameters for heat and Matérn kernels on matchings.
The final~\Cref{sec:discussion:phylogenetic} is motivated by an intimate connection between matchings and \emph{phylogenetic trees}, a data modality important in many biological applications.
It explores how the techniques developed throughout the paper could be leveraged to define and compute kernels on sets of phylogenetic trees, and the open problem associated with that.

\subsection{Computational Complexity}
\label{sec:discussion:complexity}

We now estimate the computational complexity of kernel evaluation; see~\Cref{table:asymptotics} for a summary comparing different approaches for computing zonal spherical functions.
During the precomputation stage, it is necessary to select partitions for truncation $\mathcal{R}$, compute their eigenvalues $\lambda_{2\v{\rho}}$ and dimensions $d_{2\v{\rho}}$, and perform any additional preprocessing required for the chosen zonal spherical function algorithm.
Selection of $\mathcal{R}$ via the maximal partition value strategy we will describe later in~\Cref{sec:discussion:truncation} can be carried out in $O\big(|\mathcal{R}| n\big)$ time using a simple recursion.
Eigenvalues and dimensions can each be computed in $O(n^2)$ and $O(n)$ time per partition, respectively, resulting in a total complexity of $O\big(|\mathcal{R}| n^2\big)$ for this step.
For the computation of the zonal spherical functions themselves, the use of the explicit formula requires $O\big(\abs{\c{X}_n}\, n! n p(n)\big)$ operations for precomputation (\Cref{appendix:zsf_naive}), whereas our proposed method reduces this to $O\big(p(n)^2 n^3\big)$ (\Cref{appendix:zsf_our_algorithm}).

\begin{table}[b]
\centering
\begin{tabular}{@{} l
                @{\hskip 16pt} c c
                @{\hskip 16pt} c c
                @{}}
\toprule
& Precomputation time & Query time & Precomputation memory & Query memory \\
\midrule
\Cref{thm:character_projection}
    & $O\left(|\c{R}| n^2\right)$
    & $O\left(|\c{R}| \cdot 2^n n! \cdot T(n)\right)$
    & $O(|\c{R}|+n)$ & $O(M(n))$ \\
Explicit formula
    & $O\left(\abs{\c{X}_n}\, n! n p(n)\right)$
    & $O\left(|\mathcal{R}| n\right)$
    & $O(p(n)^2 n)$ & $O(n)$ \\
Our method
    & $O\left(p(n)^2 n^3\right)$
    & $O\left(|\mathcal{R}|\, p(n)\, n^3\right)$
    & $O(p(n)^2 n^2)$ & $O(p(n) n)$ \\
\bottomrule
\end{tabular}
\caption{
Time and memory complexities for kernel evaluation using different algorithms for computing zonal spherical functions.
Here, $|\mathcal{R}|$ denotes the size of the truncation partition set, $p(n)$ is the number of integer partitions of $n$, and $\abs{\c{X}_n}$ is the number of matchings on $2n$ elements.
$T(n)$ and $M(n)$ stand for amortized time and memory complexities of computing value of character of $S_{2n}$.
In practice, we use caching instead of full precomputation in both the ``Explicit formula`` approach and in our method, making them more efficient than the theoretical analysis shows.
}
\label{table:asymptotics}
\end{table}

After precomputation, evaluation of $k(x, x')$ for a single pair $x, x' \in \c{X}_n$ involves computing $|\mathcal{R}|$ values of zonal spherical functions and taking a linear combination with precomputed coefficients.
In the approach based on \Cref{thm:character_projection}, computing one zonal spherical function value essentially reduces to evaluating $|H_n|$ character values of $S_{2n}$. Therefore, even if the cost of character evaluation is ignored, the time complexity is at least $O\big(2^n n!\big)$, which grows super-exponentially in $n$.
In the naive approach---where all possible values of zonal spherical functions are precomputed in advance, requiring a massive amount of memory---this reduces to computing the generalized distance and performing a lookup, resulting in $O\big(|\mathcal{R}| n\big)$ time per query.
By contrast, our method requires $O\big(|\mathcal{R}|\, p(n)\, n^3\big)$ operations to compute $|\mathcal{R}|$ zonal spherical function values on the fly, as detailed in \Cref{appendix:zsf_our_algorithm}.

\begin{table}[t]
\centering
\resizebox{\textwidth}{!}{
\begin{tabular}{lcccccc}
\toprule
 & $n=5$ & $n=6$ & $n=10$ & $n=15$ & $n=20$ & $n=25$ \\
\midrule
\multicolumn{7}{c}{Wall-clock time (seconds)} \\
\midrule
\Cref{thm:character_projection} & $8.53 \pm 0.08$ & $131.98 \pm 2.43$ & -- & -- & -- & -- \\
Explicit formula & $6.49 \pm 0.09$ & $115.46 \pm 1.70$ & -- & -- & -- & -- \\
Our method & $7.54 \pm 0.07$ & $13.97 \pm 0.08$ & $56.06 \pm 0.51$ & $216.78 \pm 3.63$ & $907.90 \pm 29.56$ & $3639.62 \pm 73.28$ \\
\midrule
\multicolumn{7}{c}{Memory consumption (MB)} \\
\midrule
\Cref{thm:character_projection} & $33.7 \pm 0.4$ & $34.1 \pm 0.1$ & -- & -- & -- & -- \\
Explicit formula & $33.1 \pm 0.2$ & $35.1 \pm 0.1$ & -- & -- & -- & -- \\
Our method & $33.9 \pm 0.2$ & $33.5 \pm 0.4$ & $33.6 \pm 0.2$ & $38.2 \pm 0.2$ & $91.0 \pm 0.1$ & $576.2 \pm 0.1$ \\
\bottomrule
\end{tabular}
}
\caption{
Wall-clock time (in seconds) and peak memory consumption (in MB) required to compute $100 \times 100$ covariance matrix using different methods for different values of $n$.
Results are reported as mean $\pm$ standard deviation over 5 trials.
}
\label{table:computation-time}
\end{table}

\paragraph{Numerical experiments.}

To further analyze the performance of the considered methods, we measured their runtime and memory consumption using our own software implementation.
It is important to note that, unlike the theoretical analysis, our implementation of the explicit-formula approach does not rely on full precomputation. 
Instead, values are computed on demand and then cached. 
Our method follows the same strategy: rather than precomputing all values in advance, we compute them on demand and store them in a cache.

Our experimental setup is as follows.
We generate $100$ random matchings $x_1, \ldots, x_{100}$ and measure the time and memory required to compute the respective kernel matrix $\del{k(x_i, x_j)}_{1 \leq i, j \leq 100}$.
This includes both precomputation during kernel initialization and evaluation of $100 \times 100 = 10,\!000$ kernel values.
We repeat all experiments $5$ times to obtain more reliable results and report the mean and standard deviation of the observed run times and memory requirements.

The results are reported in~\Cref{table:computation-time}. 
We observe that the naive methods are already about one order of magnitude slower at $n=6$.
To further evaluate the scalability of our approach, we run it for larger values of $n$.
The results show that our method scales up to $n=25$ while staying within an approximate one-hour budget to compute a $100 \times 100$ covariance matrix.

\subsection{Truncation}
\label{sec:discussion:truncation}

\begin{figure}[b]
\begin{subfigure}[b]{0.20\textwidth} 
\includegraphics[width=1.0\textwidth, trim=55pt 0pt 0pt 0pt]{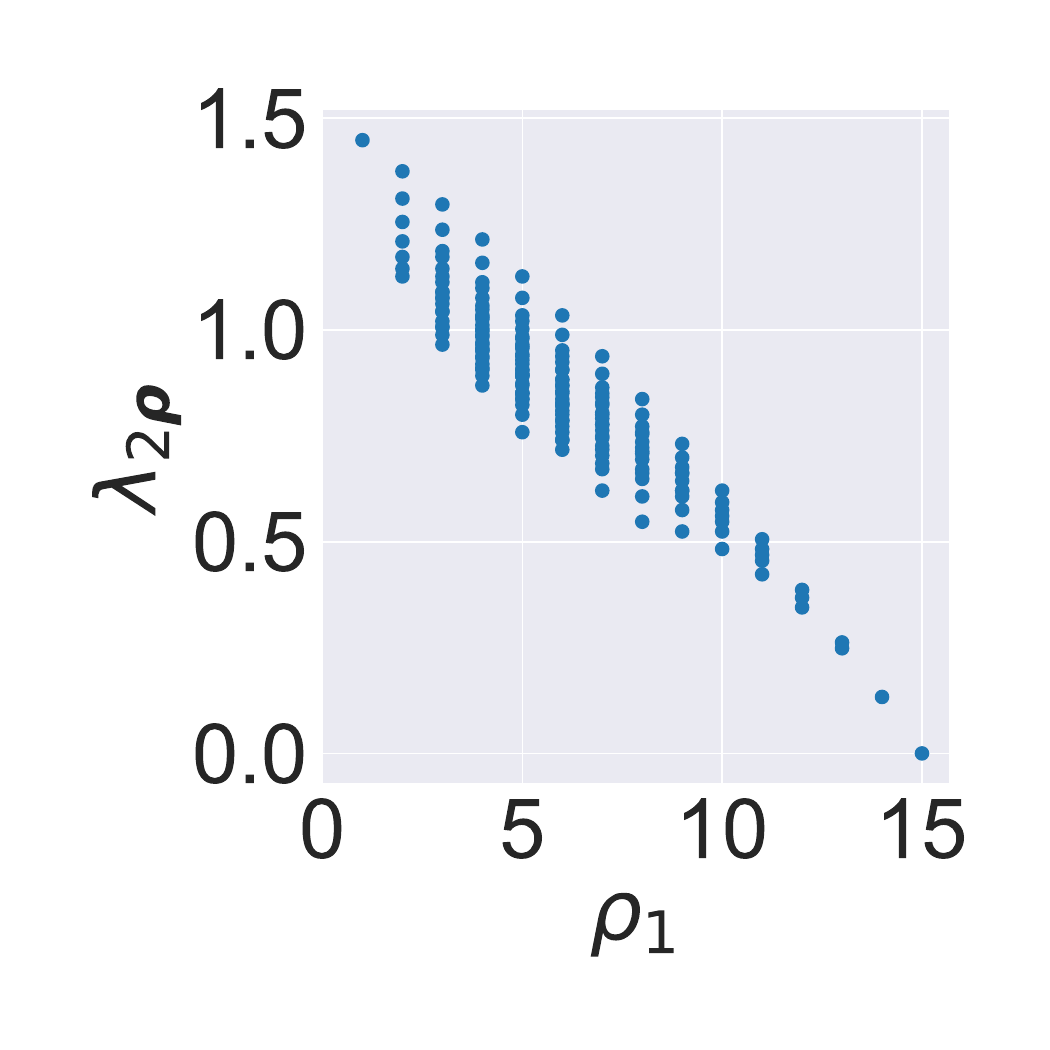}
\caption{$n=15$}
\end{subfigure}
\hfill
\begin{subfigure}[b]{0.20\textwidth}  
\includegraphics[width=1.0\textwidth, trim=55pt 0pt 0pt 0pt]{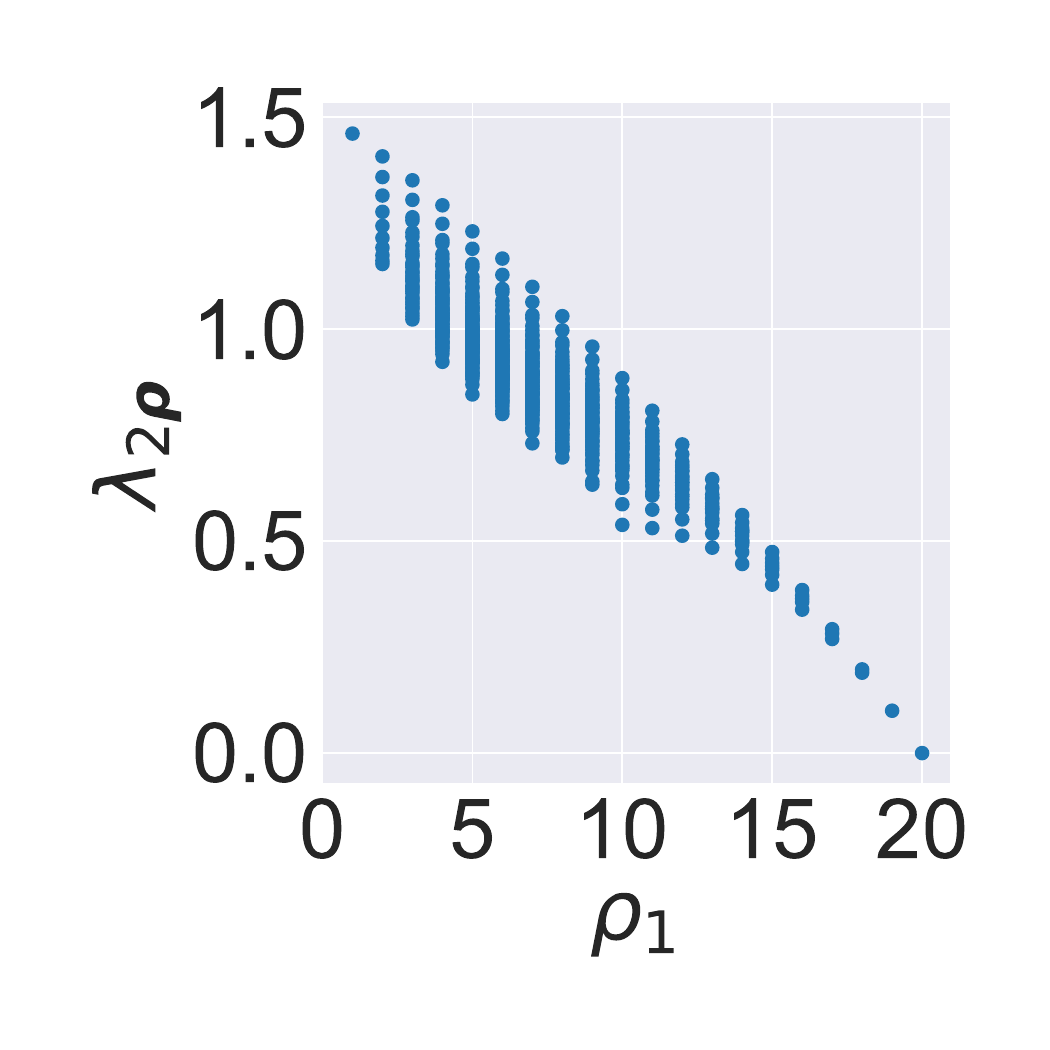}
\caption{$n=20$}
\end{subfigure}
\hfill
\begin{subfigure}[b]{0.20\textwidth}  
\includegraphics[width=1.0\textwidth, trim=55pt 0pt 0pt 0pt]{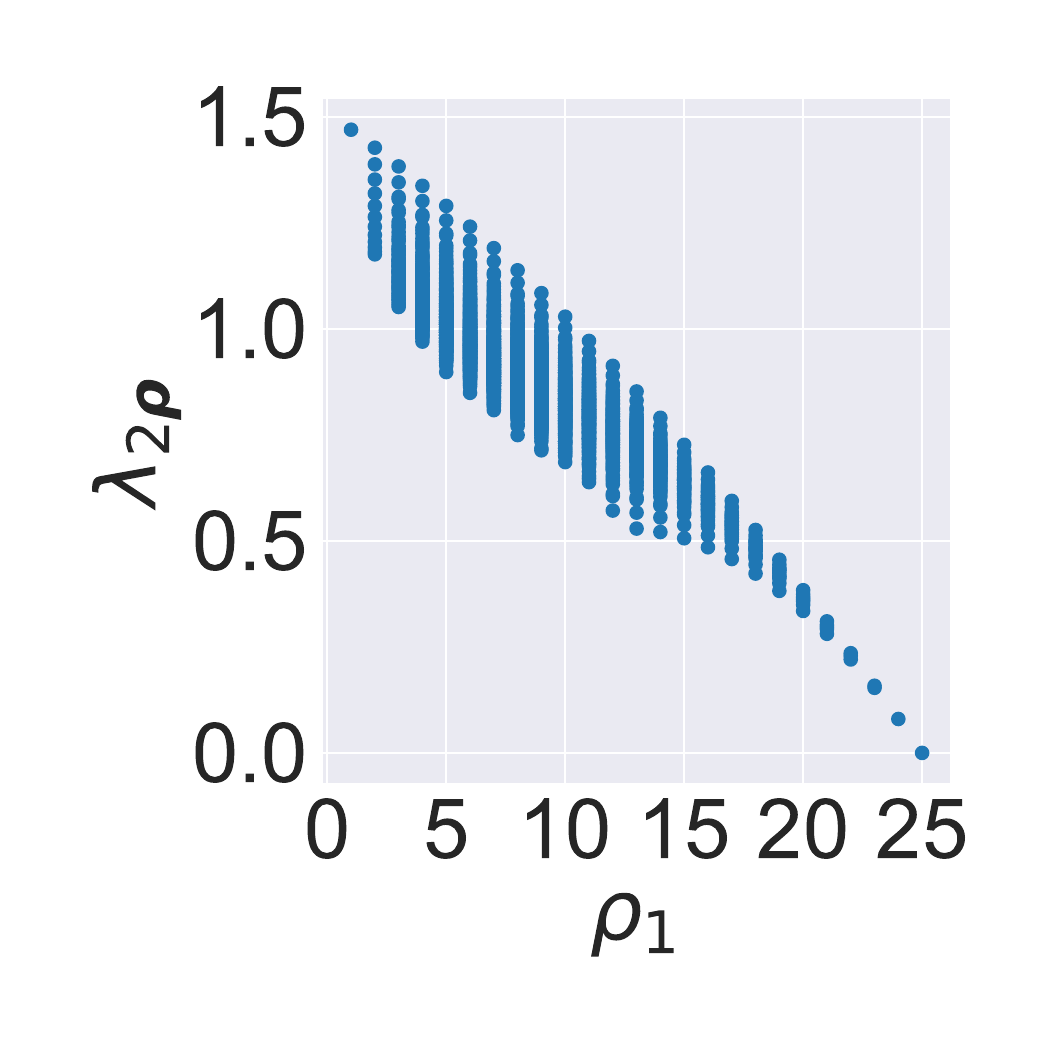}
\caption{$n=25$}
\end{subfigure}
\hfill
\begin{subfigure}[b]{0.20\textwidth}  
\includegraphics[width=1.0\textwidth, trim=55pt 0pt 0pt 0pt]{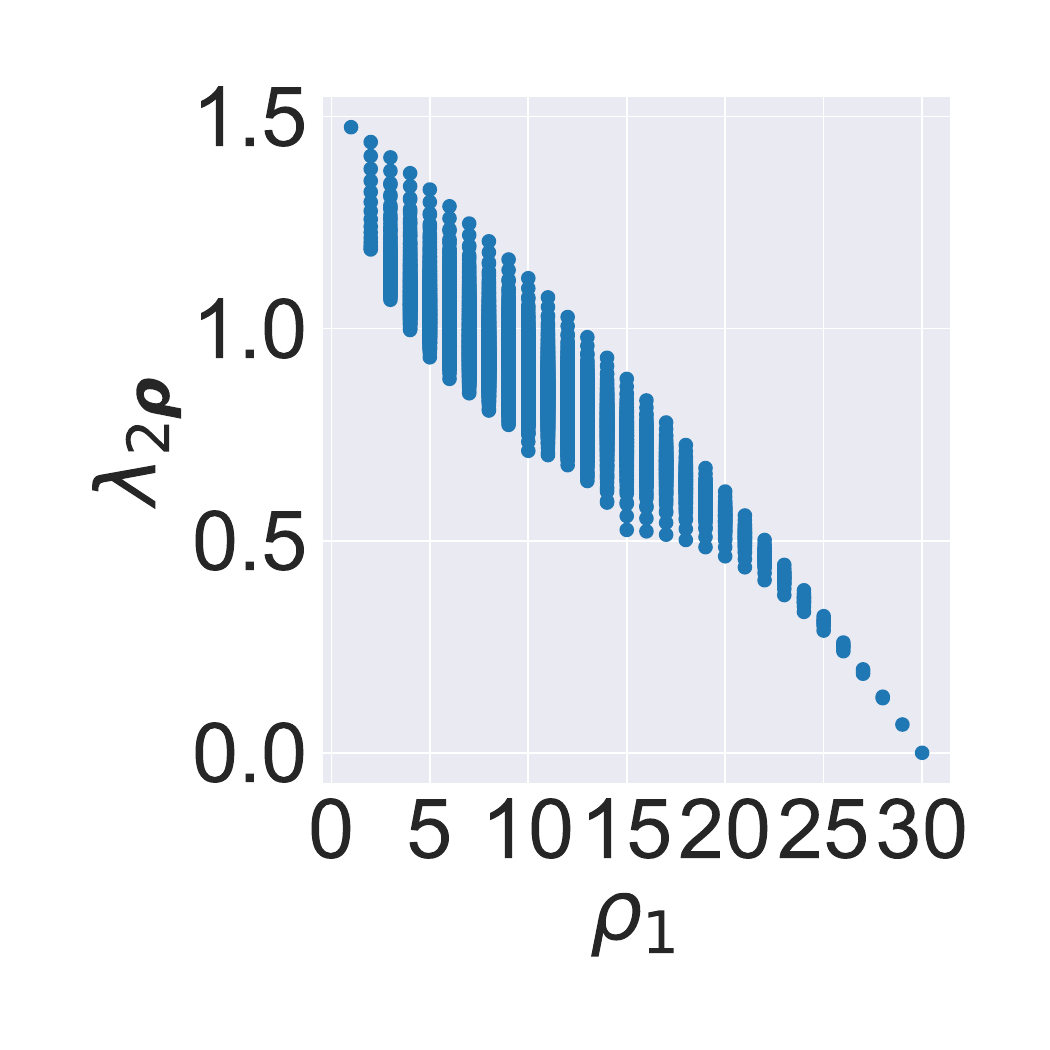}
\caption{$n=30$}
\end{subfigure}
\caption{Laplacian eigenvalues $\lambda_{2\v{\rho}}$ versus the maximal part $\rho_1$, for all partitions $\v{\rho} = (\rho_1, \ldots, \rho_s) \vdash n$.}
\label{fig:scatter-max-eigenvalue}
\end{figure}

The individual terms in the expansion~\Cref{eqn:heat_matern_matchings} contribute unequally to the resulting kernel.
Replacing the sum over all partitions in the kernel definition with a sum over a subset $\mathcal{R} \subseteq \{ \v{\rho} \mid \v{\rho} \vdash n\}$ often yields a good approximation---importantly, one that always remains a well-defined positive semi-definite kernel---at a fraction of the computational cost.
In practice, a feasible size for $\mathcal{R}$ is typically between 20 and 100 partitions, and a natural question is how to select which partitions to include.
We propose selecting partitions corresponding to \emph{low frequencies}.
Indeed, by~\Cref{thm:kernels_character_sum}, each term in the sum in~\Cref{eqn:heat_matern_matchings} is associated with a group of terms in the Mercer expansion of the kernel, i.e. the expansion in~\Cref{eqn:graph_kernels}.
For the latter, it is common to retain only the low-frequency terms---i.e., those corresponding to the smallest Laplacian eigenvalues---as these correspond to smoother and better-behaved functions, and we argue that this intuition carries over naturally to our setting.

To implement this, one may simply compute all Laplacian eigenvalues during precomputation, and select the partitions with the lowest corresponding eigenvalues.
However, we have also observed that certain characteristics of the partitions themselves can act as proxies for these eigenvalues.
For example, as illustrated in~\Cref{fig:scatter-max-eigenvalue}, the maximal part $\rho_1$ of the partition $\v{\rho} = (\rho_1, \rho_2, \ldots, \rho_s) \vdash n$---we assume ${\rho_1 \geq \rho_2 \geq \ldots \geq \rho_s}$---is well-correlated with the eigenvalue $\lambda_{\v{\rho}}$, particularly for the small $\lambda_{\v{\rho}}$.
Thus, one can efficiently select $\mathcal{R}$ as the set of partitions with largest $\rho_1$, which will roughly coincide with the partitions contributing the lowest-frequency terms.
It is easy in practice to enumerate such partitions via a simple recursive algorithm, yielding a selection complexity of $O(\abs{\mathcal{R}} n)$.
Our software implementation employs exactly this heuristic.
Alternative heuristics for partition selection are discussed in \Cref{appendix:additional_results}.

\begin{figure}[t]
\centering
\begin{subfigure}[b]{0.20\textwidth}
\includegraphics[width=\textwidth, trim=55pt 20pt 0pt -10pt, clip]{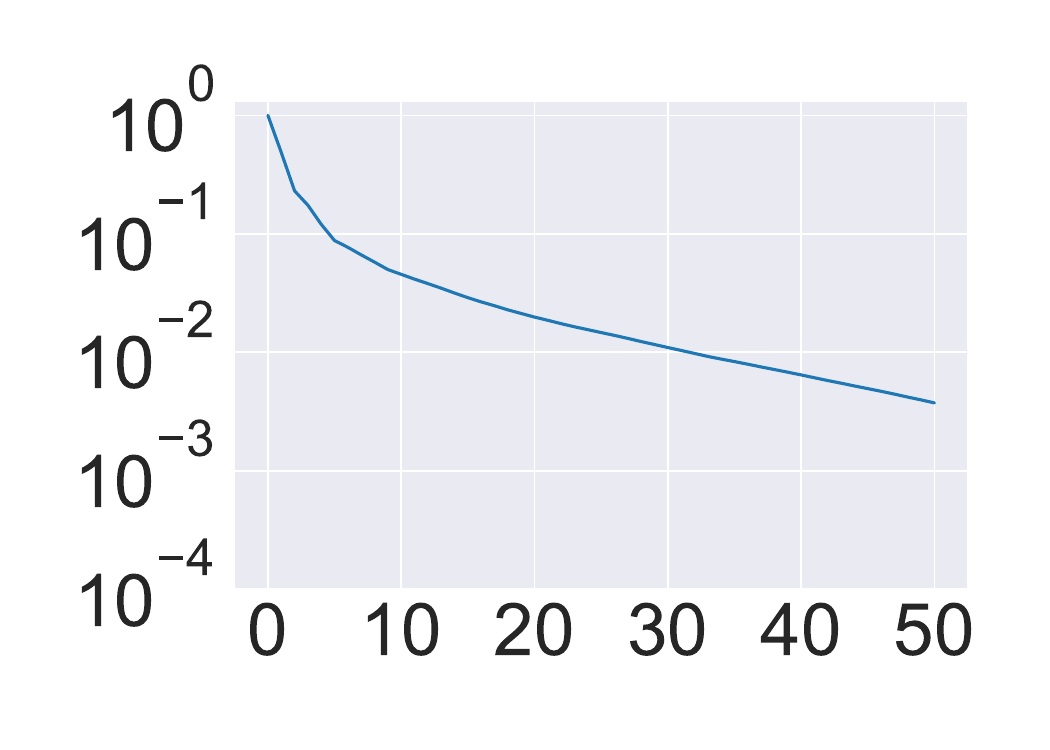}
\caption{$n=15,\ \nu=2.5$}
\end{subfigure}
\hfill
\begin{subfigure}[b]{0.20\textwidth}
\includegraphics[width=\textwidth, trim=55pt 20pt 0pt -10pt, clip]{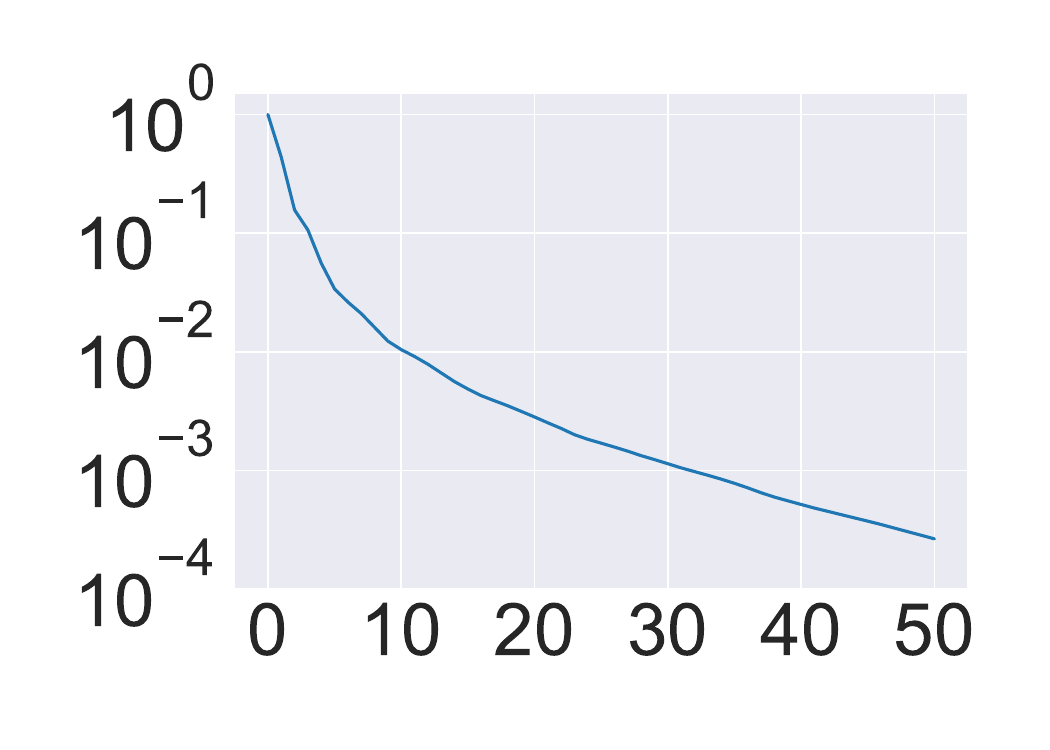}
\caption{$n=15,\ \nu=\infty$}
\end{subfigure}
\hfill
\begin{subfigure}[b]{0.20\textwidth}
\includegraphics[width=\textwidth, trim=55pt 20pt 0pt -10pt, clip]{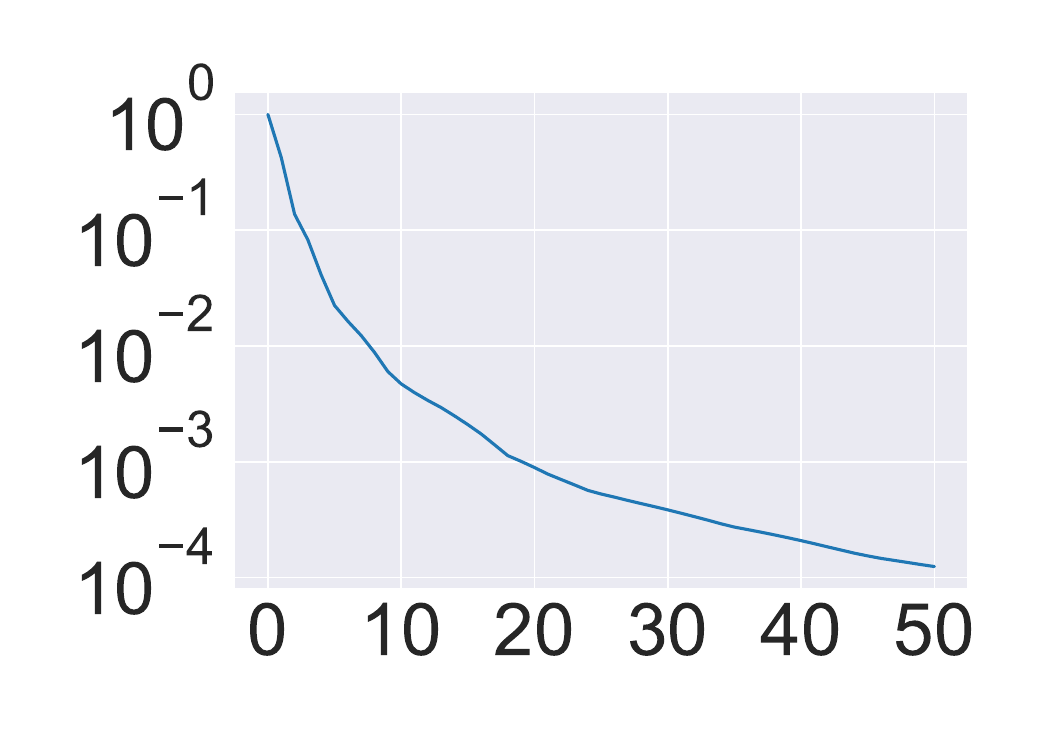}
\caption{$n=30,\ \nu=2.5$}
\end{subfigure}
\hfill
\begin{subfigure}[b]{0.20\textwidth}
\includegraphics[width=\textwidth, trim=55pt 20pt 0pt -10pt, clip]{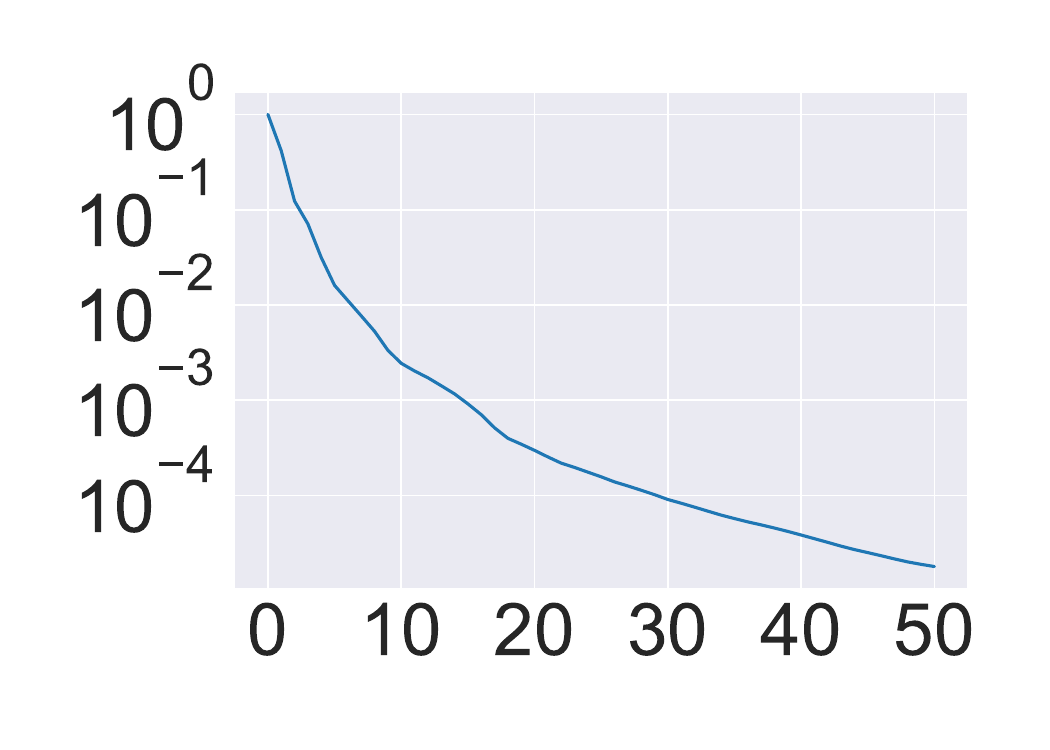}
\caption{$n=30,\ \nu=\infty$}
\end{subfigure}
\caption{
Approximation quality of truncated kernels as a function of $|\mathcal{R}|$.
The $x$-axis represents the number of partitions $|\mathcal{R}|$ used in the truncation; the $y$-axis displays the relative $L^2$ approximation error $\norm{k - k_\mathcal{R}}_{L^2(\c{X}_n \times \c{X}_n)} / \norm{k}_{L^2(\c{X}_n \times \c{X}_n)}$ (log scale).
For each $|\mathcal{R}|$, the set $\mathcal{R}$ is selected by the maximizing~$\rho_1$.
}
\label{fig:approximation-error}
\end{figure}

\paragraph{Approximation quality.}

A truncated kernel is an approximation to the exact kernel.
It is therefore natural to ask how good this approximation is in practice.
We analyze this question as follows.
Recall that the truncated kernel is defined as
\[
k_\mathcal{R}(x, x') = \sum_{\v{\rho} \in \mathcal{R}} \Phi(\lambda_{\v{\rho}}) \frac{d_{\v{\rho}}}{|S_{2n}|} \phi_{\v{\rho}} (\sigma_2^{-1} \sigma_1),
&&
x = \sigma_1 H_n, \quad
x' = \sigma_2 H_n,
\]
and the exact kernel $k(x, x')$ is given by $k_{\mathcal{R}}(x, x')$ with $\mathcal{R}$ taken to be the set of all partitions of $n$.
We quantify the approximation error by the relative squared $L_2$ norm,
\[
\frac{\| k - k_\mathcal{R}\|_{L_2(\c{X}_n \times \c{X}_n)}^2}{\| k\|_{L_2(\c{X}_n \times \c{X}_n)}^2},
&&
\norm{f}_{L^2(\c{X}_n \x \c{X}_n)}^2
=
\sum_{x, x' \in \c{X}_n} \abs{f(x, x')}^2.
\]
This error is easily analyzed, since the zonal spherical functions form an orthogonal system; hence, the squared norm of the sum equals the sum of squared norms of the individual terms.
We provide the formal derivation in~\Cref{appendix:approximation_formula_derivation}, and empirical results are reported in~\Cref{fig:approximation-error}.
We find that the approximation error decays rapidly as more terms are included, with $30$ terms often sufficient to reach an error on the order of $10^{-3}$, justifying the use of truncated kernels.
Finally, note that the approximation quality also depends on the choice of hyperparameters; our strategy for their selection is described in the following section.
Additional plots illustrating the approximation quality for a wider range of settings can be found in \Cref{appendix:additional_results}.

\subsection{Hyperparameters}
\label{sec:discussion:hyperparameters}

Besides choosing the set of representations for truncation, in practice we also need to specify two hyperparameters for the kernel: the lengthscale $\kappa$ and the smoothness parameter $\nu$.

For a fixed smoothness parameter $\nu$, we choose the lengthscale $\kappa$ using a simple heuristic.
We consider the expansion in~\Cref{eqn:heat_matern_matchings}, and in particular the coefficients corresponding to the partitions $\v{\rho} = (n)$ and $\v{\rho} = (n-1,1)$.
We then choose $\kappa$ so that the ratio of these two coefficients is equal to $2$.
This heuristic is intended to avoid two degenerate regimes.
If $\kappa$ is too large, the kernel approaches a constant kernel and therefore treats all inputs as almost equally similar.
If $\kappa$ is too small, the kernel approaches a delta kernel and therefore treats all inputs as almost completely dissimilar.
The choice of partitions is motivated by spectral considerations.
Ideally, we want this ratio condition to involve the partitions associated with the largest and second-largest eigenvalues.
In the set $\mathcal{R}$, however, partitions are selected according to the size of $\rho_1$ rather than directly by their eigenvalues.
It is therefore natural to use the partitions with the largest and second-largest values of $\rho_1$, namely $(n)$ and $(n-1,1)$.

For the \emph{smoothness parameter}, it can be empirically observed that values of $\nu$ popular in the Euclidean case, such as $\nu = 2.5$, do not yield reasonable kernels in our setting.
For example, such choices tend to result in overly ``rough'' kernels, as the coefficients in~\Cref{eqn:heat_matern_matchings} end up amplifying high frequencies more than low frequencies.
We hypothesize that this is due to the need to adjust kernels for the effective ``dimensionality'' of the domain.
Indeed, while the original definition of Matérn kernels in the Euclidean case incorporated a dimension correction, this adjustment is omitted in the standard graph-based definition~\citep{borovitskiy2021}, since the notion of dimension is not well-defined for a general graph:
\[
\text{No dimension correction: } \quad \Phi(\lambda) = \left(\frac{2\nu}{\kappa^2} + \m{\Delta}\right)^{-\nu}
\qquad
\text{Dimension correction: } \quad \Phi(\lambda) = \left(\frac{2\nu}{\kappa^2} + \m{\Delta}\right)^{-(\nu + \frac{d}{2})}\!\!\!\!\!\!\!\!\!\!\!\!\!\!\!\!\!\!.
\]
However, since the graphs we consider are regular with constant degree, interpreting the degree as the effective dimension is natural.
This idea has appeared in previous work~\citep{brualdi1991, borovitskiy2023}.
In our case, the degree is $n(n-1)$---see derivation in~\Cref{appendix:quotient}---which is a relatively large value.
Therefore, accounting for this correction may be crucial in practice.
Empirically, with this correction, kernels behave as expected, and values such as $\nu \in \{0.5, 1.5, 2.5, \infty\}$ serve as good defaults---please refer to \Cref{appendix:additional_results} for examples with and without correction.
We also note that, in this case, it is straightforward to differentiate the (approximate) kernel with respect to the hyperparameter $\nu$, as for other topologically compact domains~\citep{borovitskiy2020, borovitskiy2021}.
This allows, in settings such as Gaussian process regression, for $\nu$ to be optimized over jointly with $\kappa$, rather than selected manually.
Importantly, in this case, the dimension correction may still be important for choosing a reasonable initialization of $\nu$.

\subsection{Connection to Phylogenetic Trees}
\label{sec:discussion:phylogenetic}

\emph{Phylogenetic trees}~\citep{warnow2018} are the principal way of representing evolutionary relationships, whether between biological species and populations~\citep{kapli2020}, or in modeling the historical development of natural languages~\citep{dunn2015}, and more.
The problem of defining meaningful kernels on phylogenetic trees---and on related domains, such as trees annotated with node features---has attracted research attention, leading to several proposed constructions~\citep{weyenberg2014, aiolli2007, vert2002}.

However, none of the existing kernels are as principled as the heat and Matérn families.
An avenue for addressing this relies on a combinatorial connection between trees and matchings.
Specifically, there exists a bijective correspondence between the set of rooted phylogenetic trees with $n+1$ leaves and the set of matchings $\c{X}_n$ of size~$n$ (that is, matchings of $2n$ elements,~\citet{diaconis1998}).
This correspondence naturally leads to a question: could we define new, principled kernels on spaces of phylogenetic trees by ``pushing forward'' the (stationary, smooth) kernels constructed on matchings, via a bijection between them?

Formally, a (rooted, binary) phylogenetic tree is a rooted tree in which every internal node has degree three---that is, two children and one parent, except at the root---and in which each of the $n+1$ leaves is assigned a unique numerical label from $1$ to $n+1$.
The bijection of~\citet{diaconis1998} proceeds by assigning labels $n+2, \ldots, 2n$ to the $n-1$ internal nodes (excluding the root), thereby extending the label set up to $2n$, and then "matching" nodes that share the same parent.
Full details of this labeling procedure can be found in~\cite{ceccherinisilberstein2008}; see also~\Cref{fig:ptree-numering-example} for an illustrative example.

\begin{figure}[t]%
\hfill
\begin{subfigure}[b]{0.245\textwidth}%
\centering%
\begin{forest}
[R
    [$\bullet$
        [$\bullet$
            [1]
            [5]
        ]
        [4]
    ]
    [$\bullet$
        [3]
        [2]
    ]
]
\end{forest}%
\caption{}%
\end{subfigure}%
\hfill%
\begin{subfigure}[b]{0.245\textwidth}%
\centering%
\begin{forest}%
[R
    [$\bullet$
        [6
            [1]
            [5]
        ]
        [4]
    ]
    [$\bullet$
        [3]
        [2]
    ]
]
\end{forest}%
\caption{}%
\end{subfigure}%
\hfill%
\begin{subfigure}[b]{0.245\textwidth}%
\centering%
\begin{forest}%
[R
    [$\bullet$
        [6
            [1]
            [5]
        ]
        [4]
    ]
    [7
        [3]
        [2]
    ]
]
\end{forest}%
\caption{}%
\end{subfigure}%
\hfill%
\begin{subfigure}[b]{0.245\textwidth}%
\centering%
\begin{forest}%
[R
    [8
        [6
            [1]
            [5]
        ]
        [4]
    ]
    [7
        [3]
        [2]
    ]
]
\end{forest}%
\caption{}%
\end{subfigure}%
\hfill%
\caption{Example of the inner node labeling procedure for a rooted phylogenetic tree, following the algorithm of~\citet{diaconis1998}.
At each step, select a pair of nodes with the same parent such that both nodes are already labeled while their parent is not yet labeled.
Among all such pairs, choose the pair containing the minimal label.
Then, assign the next unused label to their parent node.
The transition from (a) to (b) corresponds to the selection of the pair $\{1, 5\}$ (which contains $1$, the smallest label), and the assignment of label $6$ to their parent.
The transition from (b) to (c) selects $\{2, 3\}$ and assigns label $7$ to their parent.
The final transition from (c) to (d) selects the only remaining feasible pair, $\{4, 6\}$, and labels their parent with~$8$.
The resulting matching is $\{\{1,5\},\, \{2,3\},\, \{4,6\},\, \{7,8\}\}$.}
\label{fig:ptree-numering-example}
\end{figure}
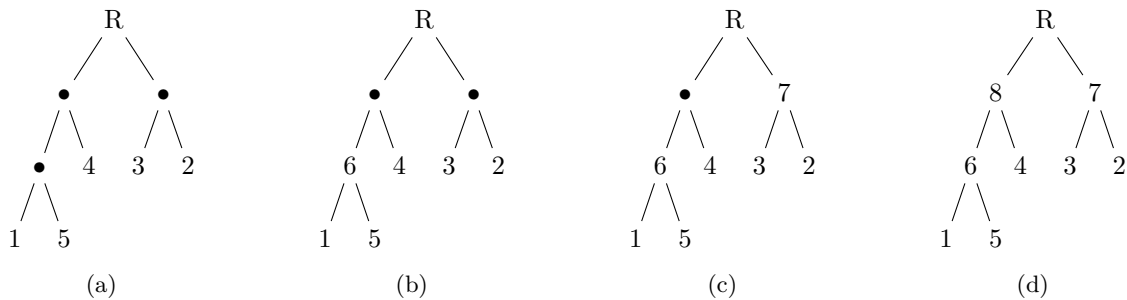
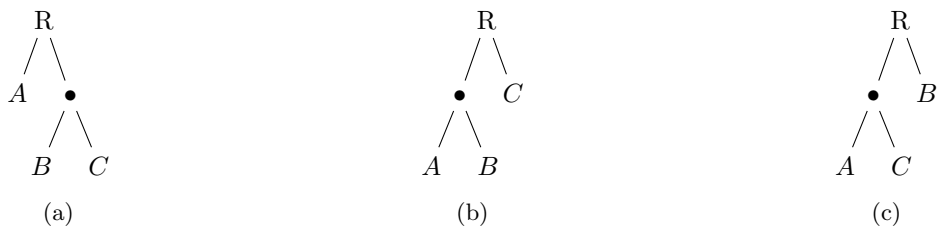
\begin{figure}[b]%
\hfill%
\begin{subfigure}[b]{0.31\textwidth}%
\centering%
\begin{forest}%
[R [$A$] [{$\bullet$} [$B$] [$C$]]]%
\end{forest}%
\caption{}%
\end{subfigure}%
\hfill%
\begin{subfigure}[b]{0.31\textwidth}%
\centering%
\begin{forest}%
[R [{$\bullet$} [$A$] [$B$]] [$C$]]%
\end{forest}%
\caption{}%
\end{subfigure}%
\hfill%
\begin{subfigure}[b]{0.31\textwidth}%
\centering%
\begin{forest}%
[R [{$\bullet$} [$A$] [$C$]] [$B$]]%
\end{forest}%
\caption{}%
\end{subfigure}%
\hfill%
\caption{Illustration of the \emph{nearest neighbor interchange} (NNI) operation. Starting from the tree in (a), where $A$, $B$, and $C$ denote arbitrary subtrees, a single NNI move interchanges two adjacent subtrees around an internal edge. Figures (b) and (c) show the two distinct topologies that may be obtained from (a) by performing an NNI move at the indicated edge. Note that $R$ does not need be the root of the entire tree---it may also be the root of a subtree.
}%
\label{fig:nni-example}%
\end{figure}%

Unfortunately, we can show that this bijection does not preserve geometry in any meaningful sense.
Specifically, there are small perturbations of phylogenetic trees that can be mapped to very large perturbations of matchings, and vice versa.
To formalize this, we need to define distances both on the set of phylogenetic trees and on the set of matchings.
The latter can be naturally defined as the shortest path distance in the quotient graph discussed in~\Cref{sec:kernels:heat_matern}.
For the distance on the set of phylogenetic trees, we choose the commonly used \emph{NNI-distance}~\citep{robinson1971, moore1973}, where \emph{NNI} stands for \emph{nearest neighbor interchange}, an elementary operation that defines distance~$1$ between trees and is illustrated in~\Cref{fig:nni-example}.

The fact that small perturbations of a phylogenetic tree can be mapped to very large perturbations of matchings is formalized by the following result.

\begin{restatable}{proposition}{BijectionSingleNniBigMatchingDist}
\label{BijectionSingleNniBigMatchingDist}
For every $n \geq 7$, under the bijection of~\citet{diaconis1998}, there exist a pair of rooted binary phylogenetic trees $T$ and $T'$ with $n+1$ leaves that differ by a single NNI move, such that their corresponding matchings $x, x'$ in $\mathcal{X}_n$ have quotient Cayley graph distance between them at least $\frac{n-1}{2}$.
\end{restatable}

\begin{proof}
See \Cref{appendix:ptrees}.
\end{proof}

The next result shows the converse: a small change in a matching can result in large changes to the three.

\begin{restatable}{proposition}{MatchingSmallMoveBigNni}
\label{prop:matching-small-move-big-nni}
For every $n \geq 9$, there exists a matching $x \in \mathcal{X}_n$ and a transposition $\sigma$ such that the phylogenetic trees corresponding to $x$ and $\sigma \lacts x$ differ by NNI-distance at least $\frac{n}{2}-2$.
\end{restatable}

\begin{proof}
See Appendix~\ref{appendix:ptrees}.
\end{proof}

Taken together, these results show that the bijection of~\citet{diaconis1998} does not preserve any meaningful notion of smoothness.
Consequently, the heat and Matérn kernels on matchings pushed forward via this bijection will be fundamentally misaligned with the underlying geometry of phylogenetic trees.

Given this negative result, a natural question arises:

\begin{openproblem}
\label{openproblem:ptree}
Does there exist an efficiently computable bijection between matchings and phylogenetic trees such that ``small'' changes in one space correspond to ``small'' changes in the other and vice versa?
\end{openproblem}

Some hope may be drawn from the following argument.
Recent work by~\citet{richman2025} proposes an alternative encoding for the nodes of a phylogenetic tree: they assume zero-based labels $\{0, 1, \ldots, n\}$ for the leaves and label the internal nodes, \emph{including the root}, using negative numbers $\{-1, -2, \ldots, -n\}$.
Without loss of generality, we may instead assume that leaves are labeled by $\{1, \ldots, n+1\}$ as usual, while internal nodes are labeled by $\{n+2, \ldots, 2n+1\}$, with one of them corresponding to the root node.

The label assigned to the root is different for different trees, making it impossible to consistently ``match'' nodes that share the same parent: such a procedure would result in a matching over a different set of labels for every tree.
Nonetheless, using these labels, we can define an \emph{embedding} from the set of phylogenetic trees with $n+1$ leaves into the set of matchings $\c{X}_{n+1}$ of size $n+1$ (i.e., on $2(n+1)$ points).
To do this, we introduce an auxiliary label $n+2$ and ``match'' the label corresponding to the root for the given tree with this additional label.
A key feature of this embedding is that a small change in a phylogenetic tree necessarily leads to a small change in the associated matching, as made precise in the next proposition:

\begin{restatable}{proposition}{EmbeddingSingleNniSmallMatchingDistance}
\label{EmbeddingSingleNniSmallMatchingDistance}
In the embedding based on~\citet{richman2025}, any two phylogenetic trees differing by a single NNI move are mapped to matchings whose quotient Cayley graph distance is at most two.
\end{restatable}

\begin{proof}
See Appendix~\ref{appendix:ptrees}.
\end{proof}

\begin{table}[t]
\vspace*{-0.75cm}
\centering
\caption{
Comparison of known constructions relating matchings and phylogenetic trees.
}
\label{tab:bijection-properties}
\begin{tabular}{lcc}
\toprule
Method & \makecell{Small matching distance \\ $\Rightarrow$ small tree distance} & \makecell{Small tree distance \\ $\Rightarrow$ small matching distance} \\
\midrule
Bijection of \citet{diaconis1998} & No (Prop.~\ref{prop:matching-small-move-big-nni}) & No (Prop.~\ref{BijectionSingleNniBigMatchingDist}) \\
Embedding of \citet{richman2025}     & No (not invertible) & Yes (Prop.~\ref{EmbeddingSingleNniSmallMatchingDistance}) \\
\bottomrule
\end{tabular}
\end{table}

However, because not every matching of size $n+1$ corresponds to a phylogenetic tree, this mapping lacks invertibility: "small" perturbations of matchings can lead outside the the set of phylogenetic trees.
In fact, the fraction of matchings corresponding to phylogenetic trees in~$\c{X}_{2n}$ is as little as $\frac{4n}{(2n+2)(2n+1)} \leq 1/3$.

As summarized in Table~\ref{tab:bijection-properties}, to the best of our knowledge there is currently no efficient, geometry-preserving bijection between matchings and phylogenetic trees that enables direct transfer of kernels while respecting geometry.
Resolving this open problem would provide a principled pathway for constructing meaningful and interpretable kernels on phylogenetic trees by leveraging the powerful tools developed for matchings.

\bibliography{references}
\bibliographystyle{tmlr}

\newpage

\appendix

\section{Kernel Definition}

\subsection{Stationary Kernels}\label{appendix:stationary_kernels}

In this section, we will prove \Cref{thm:general_stationary_kernel} and \Cref{thm:character_projection}. \Cref{thm:general_stationary_kernel} describes the general form of stationary kernels on matchings in terms of zonal spherical functions, while \Cref{thm:character_projection} describes zonal spherical functions as two-sided averages of group characters.

We will extensively use the machinery of representation theory and Gelfand pairs. While we will give a brief introduction to Gelfand pairs, we also refer the reader to the systematic presentation to be found in~\citep{ceccherinisilberstein2008}. 

Remember that harmonic (Fourier) analysis can be generalized far beyond the classic Euclidean case.
For example, there exists a well-studied theory of harmonic analysis on locally-compact \textit{abelian} groups.
However, this theory becomes more complicated when the group is not abelian.
Intuitively, Gelfand pairs provide a setting in which harmonic analysis on non-abelian groups retains some of the simplicity of the abelian case. 
\begin{definition}
Let $G$ be a finite group, and $H$ a subgroup of $G$. We will say that the pair $(G,H)$ is a Gelfand pair if the algebra of $H$-bi-invariant functions ${L(H\backslash G/H) = \{f\in L(G)\mid f(h_1 g h_2) = f(g) \,\, \forall h_1, h_2 \in H, \forall g \in G\}}$ is commutative with respect to convolution.
\end{definition}

\begin{definition} 
Let $(G, H)$ be a Gelfand pair.
The function $\phi\in L(H\backslash G/H)$ is called a zonal spherical function if it is not identically zero and if
$$ \frac{1}{|H|}\sum_{h \in H} \phi(g_1 h g_2) = \phi(g_1) \phi(g_2) $$
for all $g_1, g_2 \in G$.
\end{definition}

Let $(G, H)$ be a Gelfand pair.
One good thing about zonal spherical functions is that they constitute an orthogonal basis in the space of $H$-bi-invariant functions $L(H \backslash G / H)$.
Additionally, every zonal spherical function $\phi$ can be represented as a diagonal matrix coefficient.
Namely, $\phi(g) = \langle u, \rho(g)u \rangle$ for some unitary irreducible representation $(V, \rho)$ and some $u \in V, \|u\|=1$.
The corresponding representation $(V, \rho)$ is called a spherical representation associated with $\phi$.
Moreover, the zonal spherical function $\phi$ uniquely determines the representation $(V,\rho)$ up to equivalence.

In the context of the present paper, the group $G$ is the symmetric group $S_{2n}$, and the subgroup $H$ is the stabilizer of some matching $x_0 \in X$.
One can show that in this case $(G, H)$ is indeed a Gelfand pair.
Moreover, all spherical representations are known.
Remember that non-equivalent unitary irreducible representations of $S_{2n}$ are enumerated by partitions of $2n$.
Then, spherical representations are exactly those that correspond to a partition of the form $2\rho = (2\rho_1, 2\rho_2, \ldots, 2\rho_s) \vdash 2n$ for some partition $\rho = (\rho_1, \ldots, \rho_s)$ of $n$.

Now, we are going to provide proofs for \Cref{thm:general_stationary_kernel} and \Cref{thm:character_projection}.
We will start with the latter one, which describes zonal spherical functions in terms of two-sided averages of characters.
But before, let us prove the following lemma.

\begin{restatable}{lemma}{AveragingIsOrthoprojeciton}
\label{lemma:AveragingIsOrthoprojeciton}
Let $G$ be a finite group and $H$ a subgroup of $G$.
Define the operator 
$$ \Pr : L(G) \rightarrow L(G) \quad\quad\quad (\Pr f)(g) := \frac{1}{|H|^2} \sum_{h_1, h_2 \in H} f(h_1 g h_2) , $$
then $\Pr$ is the orthogonal projection from $L(G)$ onto $L(H\backslash G/H)$, the subspace of $H$-bi-invariant functions.
\end{restatable}
\begin{proof}
Note that this proof almost completely repeats the proof of Lemma 2 and the comments after it in the paper \citet{borovitskiy2023}.
It is easy to see that $\Pr\Pr = \Pr$, and therefore $\Pr$ is a projection.
Also, if $f\in L(H\backslash G/H)$, then $\Pr f = f$.
It remains only to prove that if $f\bot L(H\backslash G /H)$, then $\Pr f = 0$.

For $g, k\in G$ we denote
$$\alpha(g, k) = \left| \left\{ (h_1, h_2) \in H^2 \mid h_1 g h_2 = k \right\} \right|.$$
Let us prove that if $k\in H g H$, then $\alpha(g,k) = \alpha(g,g)$.
Indeed, let $k = \widetilde{h}_1 g \widetilde{h}_2$.
Then,
\begin{align*}
\alpha(g, k) &= \left| \left\{ (h_1, h_2) \in H^2 \mid h_1 g h_2 = k \right\} \right| \\
&= \left| \left\{ (h_1, h_2) \in H^2 \mid h_1 g h_2 = \widetilde{h}_1 g \widetilde{h}_2 \right\} \right| \\
&= \left| \left\{ (h_1, h_2) \in H^2 \mid \widetilde{h}_1^{-1} h_1 g h_2 \widetilde{h}_2^{-1} = g \right\} \right| \\
&= \left| \left\{ (h_1, h_2) \in H^2 \mid h_1 g h_2 = g \right\} \right| .
\end{align*}

Therefore, if $f\bot L (H\backslash G/H)$, then for any $g\in G$,
\begin{align*}
(\Pr f) (g) 
&= \frac{1}{|H|^2} \sum_{h_1, h_2 \in H} f(h_1 g h_2)
= \frac{1}{|H|^2} \sum_{k \in HgH} \alpha(g, k) f(k) \\
&= \frac{1}{|H|^2} \sum_{h \in HgH} \alpha(g, g) f(k) 
= \frac{\alpha(g, g)}{|H|^2} \sum_{k \in HgH} f(k) = 0 ,
\end{align*}
where the latter equality is true due to the fact that $f$ is orthogonal to $\1_{HgH}\in L^2(H\backslash G / H)$.
\end{proof}

\CharacterProjection*
\begin{proof}

According to Lemma \ref{lemma:AveragingIsOrthoprojeciton}, averaging corresponds to orthogonal projection onto $L(H \backslash G / H)$.
As was mentioned above, zonal spherical functions form an orthogonal basis in that space.
So, it is sufficient to compute inner products of characters with zonal spherical functions.

Consider the representation corresponding to $2\rho$.
We know that its zonal spherical function $\phi_\rho$ can be represented as a diagonal matrix coefficient $\phi_\rho(g) = \langle u_1, \rho(g) u_1 \rangle$ for some $u_1, \|u_1\|=1$.
If the vector $u_1$ is extended to an orthonormal basis $\{u_1, \ldots, u_{d_{2\rho}}\}$, then
\[
\chi^{(2\rho)}(g) = \sum_{i=1}^{d_\rho} \langle u_i, \rho(g) u_i \rangle = \phi_\rho(g) + \sum_{i=2}^{d_\rho} \pi_{i, i}^{(2\rho)}(g) ,
\]
and we know that matrix coefficients are orthogonal.
Therefore, orthogonal projection of $\chi^{(2\rho)}$ is clearly $\phi_\rho$.

\end{proof}

\begin{remark}
Similarly, if $\chi^{(\mu)}$ is a character of a unitary irreducible non-spherical representation, then $\chi^{(\mu)}$ is the sum of some matrix coefficients that are orthogonal to zonal spherical functions, because all zonal spherical functions are (diagonal) matrix coefficients.
Therefore, the orthogonal projection of $\chi^{(\mu)}$ is zero.
\end{remark}

Now, let us move to discussing \Cref{thm:general_stationary_kernel}.

\GeneralStationaryKernel*
\begin{proof}
First, let us note that a stationary kernel $k(\bullet H_n, \bullet H_n)$ is uniquely determined by the $H$-bi-invariant function $k(\bullet) = k(\bullet H_n, H_n)$ since $k(\sigma H_n, \pi H_n) = k(\pi^{-1} \sigma H_n, H_n)$.
Requiring the kernel $k(\bullet H_n, \bullet H_n)$ to be positive semi-definite is equivalent to requiring the function $k(\bullet)$ to be positive semi-definite.
Then, we can apply the Bochner--Godement theorem (see, for example, Section~1.1, Theorem~1.2 in~\citet{faraut2007}).
In particular, this states that if $G$ is a finite group and $(G, H)$ is a Gelfand pair, then any positive semi-definite $H$-bi-invariant function can be uniquely represented as a sum of zonal spherical functions with non-negative coefficients.
Since in our case all zonal spherical functions are indexed by partitions, the sum over all zonal spherical functions can be written as a sum over all partitions.
\end{proof}

\StationaryAreBiInvariantProjection*
\begin{proof}
Theorem 2 in \citet{yaglom1961} provides a general form of bi-invariant kernels on $S_{2n}$.
Namely, each bi-invariant kernel on $S_{2n}$ can be uniquely represented as the sum of characters with non-negative coefficients:
$$k(\sigma, \pi) = \sum_{\v{\mu} \vdash 2n} a_{\v{\mu}} \chi^{(\v{\mu})} (\pi^{-1} \sigma).$$
By \Cref{thm:character_projection}, the average of a character will either be a zonal spherical function (if the corresponding representation is spherical) or zero (if the corresponding representation is not spherical).
Therefore, projections of bi-invariant kernels have the form
$$k(\sigma H_n, \pi H_n) = \sum_{\v{\mu} \vdash n} a_{2\v{\mu}} \phi_{\v{\mu}} (\pi^{-1} \sigma),$$
which clearly corresponds to stationary kernels according to \Cref{thm:general_stationary_kernel}.
Moreover, since $a_{\v{\mu}}$ can be arbitrary non-negative coefficients, one can get all possible stationary kernels in this way.
\end{proof}

\subsection{Heat and Matérn Kernels}\label{appendix:heat_and_matern_kernels}

In this section, we provide proofs for \Cref{thm:kernels_character_sum} and \Cref{thm:heat_matern_kernel_matchings}.
In order to do so, we need to show that matrix coefficients are eigenvectors of the Laplacian $\m{\Delta}$.
This is stated in the following lemma.
\begin{lemma}
\label{lemma:MatrixCoefsAreLaplacianEigenvectors} 
Let $G$ be a finite group, and $W\subseteq G$ a subset such that $W$ generates $G$, $W^{-1}=W, e\notin W$, and $gWg^{-1} = W$ for any $g \in G$.
Let $\Gamma(G,W)$ be the corresponding Cayley graph and $\m{\Delta}$ its Laplacian.
Then, there is an orthonormal basis of eigenvectors of $\m{\Delta}$, consisting of properly normalized matrix coefficients of non-equivalent irreducible unitary representations $\sqrt{\frac{d_\rho}{|G|}}\pi_{i, j}^{(\rho)}$.
Moreover, within one irreducible representation $\rho$, all matrix coefficients have the same eigenvalue $\lambda_\rho$.
\end{lemma}
\begin{proof}

Note that this proof almost completely repeats the proof of result 13 in the paper \cite{azangulov2024a}, with the only difference that we consider finite groups instead of compact Lie groups.

The fact that $\sqrt{\frac{d_\rho}{|G|}} \pi_{i,j}^{(\rho)}$ form an orthonormal basis is known from the general theory of representations.
Therefore, it is enough for us to prove that for all $\rho$ there exists $\lambda_\rho \in \mathbb{R}$ such that for all $1 \leq i, j \leq d_\rho$, $\m{\Delta} \pi_{i, j}^{(\rho)} = \lambda_\rho \pi_{i, j}^{(\rho)}$.
Recall that each element of the group $g\in G$ defines left and right shift operators on the space of functions $L(G)$:
$$L_g, R_g: L(G) \rightarrow L(G) \quad\quad\quad (L_g f)(h) = f(g^{-1} h) \quad\quad\quad (R_g f)(h) = f(h g).$$
Note that these shifts commute with the Laplacian $\m{\Delta}$:
\begingroup
\allowdisplaybreaks
\begin{align*}
(L_g \m{\Delta} f)(h) 
&= (\m{\Delta} f)(g^{-1} h) \\
&= \sum_{w \in W} \left( f(g^{-1}h) - f(wg^{-1}h) \right) \\
&= \sum_{w \in W} \left( f(g^{-1}h) - f(g^{-1} (gwg^{-1}) h) \right) \\
&= \sum_{w \in W} \left( f(g^{-1}h) - f(g^{-1} w h) \right) \\
&= \sum_{w \in W} \left( (L_g f)(h) - (L_g f)(w h) \right) \\
&= (\m{\Delta} L_g f)(h) , \\
(R_g \m{\Delta} f)(h) &= (\m{\Delta} f)(hg) \\
&= \sum_{w \in W} \left( f(hg) - f(whg) \right) \\
&= \sum_{w \in W} \left( (R_g f)(h) - (R_g f)(wh) \right) \\
&= (\m{\Delta} R_g f)(h) .
\end{align*}
\endgroup

Let us define $ \psi^{(\rho)}_{j, k} := \m{\Delta} \pi^{(\rho)}_{j, k}$.
Then, $ (\m{\Delta} L_g \pi^{(\rho)}_{j, k})(h) = (L_g \psi^{(\rho)}_{j, k})(h) = \psi^{(\rho)}_{j, k} (g^{-1} h) $.
And at the same time, $ (\m{\Delta} L_g \pi^{(\rho)}_{j, k})(h) = \left( \m{\Delta} \sum_l \pi^{(\rho)}_{j, l} (g^{-1}) \pi^{(\rho)}_{l, k} \right)(h) = \sum_l \pi^{(\rho)}_{j, l} (g^{-1}) \psi^{(\rho)}_{l, k}(h) $,
which means that $ \psi^{(\rho)}_{j, k} (g^{-1} h) = \sum_l \pi^{(\rho)}_{j, l} (g^{-1}) \psi^{(\rho)}_{l, k}(h) $.
Denoting $ \psi^{(\rho)} (h) = \left[ \psi^{(\rho)}_{i, j} (h) \right]_{i, j} , $ the obtained result can be written in matrix form: $ \psi^{(\rho)} (gh) = \pi^{(\rho)}(g) \psi^{(\rho)}(h) $.
By doing the same calculations for $R_g$, we get $ \psi^{(\rho)} (hg) = \psi^{(\rho)}(h) \pi^{(\rho)}(g) $.
Then, substituting $h=e$, we have $ \pi^{(\rho)} (g) \psi^{(\rho)} (e) = \psi^{(\rho)} (g) = \psi^{(\rho)} (e) \pi^{(\rho)} (g)$.
So, according to Schur's lemma, $ \psi^{(\rho)}(e) = \lambda_\rho \m{I}$, a scalar multiple of the $d_\rho \times d_\rho$ identity matrix.
Therefore, $ \psi^{(\rho)}(g) = \lambda_\rho \pi^{(\rho)} (g) $, which proves the required statement.
\end{proof}

\KernelsAreCharactersSum*
\begin{proof}
By \Cref{lemma:MatrixCoefsAreLaplacianEigenvectors}, properly normalized matrix coefficients $\sqrt{\frac{d_{\v{\rho}}}{|S_{2n}|}} \pi_{i, j}^{(\v{\rho})}$ constitute an orthonormal basis of Laplacian eigenvectors.
In \Cref{eqn:s2n_eigenvector_expansion}, one can replace the orthonormal eigenbasis $\{u_\pi\}_{\pi \in S_{2n}}$ with any other orthonormal eigenbasis, for example, the one consisting of properly normalized matrix coefficients.
Then,
\begingroup
\allowdisplaybreaks
\begin{align*}
k(\sigma, \sigma') 
&= \frac{1}{C(\Phi)} \sum_{\pi \in S_{2n}} \Phi(\lambda_\pi) u_\pi(\sigma) \overline{u_\pi (\sigma')}
\\ &= \frac{1}{C(\Phi)} \sum_{\v{\rho} \vdash 2n} \Phi(\lambda_{\v{\rho}}) \sum_{1 \leq i, j \leq d_{\v{\rho}}} \left( \sqrt{\frac{d_{\v{\rho}}}{|S_{2n}|}} \pi_{i, j}^{(\v{\rho})} (\sigma) \right) \left( \sqrt{\frac{d_{\v{\rho}}}{|S_{2n}|}} \overline{\pi_{i, j}^{(\v{\rho})} (\sigma')} \right)
\\ &= \frac{1}{C(\Phi)} \sum_{\v{\rho} \vdash 2n} \Phi(\lambda_{\v{\rho}}) \frac{d_{\v{\rho}}}{|S_{2n}|} \sum_{1 \leq i, j \leq d_{\v{\rho}}} \pi_{i, j}^{(\v{\rho})} (\sigma) \pi_{j, i}^{(\v{\rho})} ((\sigma')^{-1})
\\ &= \frac{1}{C(\Phi)} \sum_{\v{\rho} \vdash 2n} \Phi(\lambda_{\v{\rho}}) \frac{d_{\v{\rho}}}{|S_{2n}|} \sum_{1 \leq i \leq d_{\v{\rho}}} \pi_{i, i}^{(\v{\rho})} (\sigma (\sigma')^{-1}) 
\\ &= \frac{1}{C(\Phi)} \sum_{\v{\rho} \vdash 2n} \Phi(\lambda_{\v{\rho}}) \frac{d_{\v{\rho}}}{|S_{2n}|} \chi^{(\v{\rho})} (\sigma (\sigma')^{-1}) 
\\ &= \frac{1}{C(\Phi)} \sum_{\v{\rho} \vdash 2n} \Phi(\lambda_{\v{\rho}}) \frac{d_{\v{\rho}}}{|S_{2n}|} \chi^{(\v{\rho})} ((\sigma')^{-1} \sigma) .
\end{align*}
and the last expression is exactly the one in the proposition.
\end{proof}
\endgroup

\HeatMaternKernelMatchings*
\begin{proof}
By definition of the projection,
\begin{align*}
k_{\nu, \kappa}(\sigma H_n, \sigma' H_n) 
&= \frac{1}{|H|^2} \sum_{h_1, h_2 \in H_n} k_{\nu, \kappa} (\sigma h_1, \sigma' h_2) 
\intertext{By \Cref{thm:kernels_character_sum}, one may expand $k_{\nu, \kappa}$ into a character sum}
\\ &= \frac{1}{|H|^2} \sum_{h_1, h_2 \in H_n} \frac{1}{C_{\nu, \kappa}} \sum_{\v{\rho} \vdash 2n} \Phi(\lambda_{\v{\rho}}) \frac{d_{\v{\rho}}}{|S_{2n}|} \chi^{(\v{\rho})} ( h_2^{-1} (\sigma')^{-1} \sigma h_1)
\\ &= \frac{1}{C_{\nu, \kappa}} \sum_{\v{\rho} \vdash 2n} \Phi(\lambda_{\v{\rho}}) \frac{d_{\v{\rho}}}{|S_{2n}|} \frac{1}{|H|^2} \sum_{h_1, h_2 \in H_n} \chi^{(\v{\rho})} ( h_2 ((\sigma')^{-1} \sigma) h_1)
\intertext{Finally, by \Cref{thm:character_projection}, the two-sided average of a character will be either a zonal spherical function or zero}
\\ &= \frac{1}{C_{\nu, \kappa}} \sum_{\v{\rho} \vdash n} \Phi(\lambda_{2\v{\rho}}) \frac{d_{2\v{\rho}}}{|S_{2n}|}  \phi_{\v{\rho}} ((\sigma')^{-1} \sigma) .
\end{align*}
as required in the proposition.
\end{proof}

\subsection{Quotient Cayley Graph}\label{appendix:quotient}

In this section, we will define quotient Cayley graph and discuss two related questions.
First, we will prove that defining the kernel on quotient Cayley graph is equivalent to averaging the kernel on Cayley graph, therefore justifying our definition through averaging.
Second, we will compute degree of the corresponding quotient Cayley graph, which we will further use to employ degree correction for hyperparameter selection.

\begin{definition}
Let $\Gamma(S_{2n}, W)$ be Cayley graph, where the generation set is given by the set of all transpositions $W$.
Then, (weighted) quotient Cayley graph $\Gamma_{/H}(S_{2n}, W)$ is defined as follows.
Its nodes are given by cosets $\sigma H$, and edge connects two cosets $\sigma H, \pi H$ if and only if there was at least one edge between nodes $\sigma h_1$ and $\pi h_2$ in the original Cayley graph for some $h_1, h_2 \in H$.
The weight of an edge between the cosets $\sigma H$ and $\pi H$ is defined as the number of edges connecting these two cosets in the original Cayley graph.
\end{definition}

Since the set of matchings $\mathcal{X}_n$ can be represented as a homogeneous space $S_{2n} / H$, nodes of the quotient Cayley graph can be associated with matchings.
Therefore, the quotient Cayley graph provides a natural geometry for the set of matchings induced from the geometry on $S_{2n}$.

\paragraph{Equivalent ways to define kernels}

Remember that \Cref{eqn:graph_kernels} allows to define kernels on arbitrary graphs.
Then, there are two ways to define kernel on the quotient Cayley graph $\Gamma_{/H}(S_{2n}, W)$.
First is to apply \Cref{eqn:graph_kernels} directly to the quotient Cayley graph.
Second is to use it for the original Cayley graph, and then apply averaging, as defined in \Cref{eqn:projection}.
Here, in \Cref{thm:two_quotient_kernels} we will prove that these ways are equivalent.
For the proof, we will follow proof from \cite{borovitskiy2023} with minor changes.
Specifically, we will adopt the proof of the theorem 5 and all needed prerequisites.

\begin{definition}
Consider an unweighted undirected graph $G = (V, E)$ with adjacency matrix $\m{A}_G$.
For a set $C \subseteq V$ and a vertex $v \in V$ define $deg(v,C)= | \{v' \in C:(v,v')\in E \} |$.
A partition $V = \bigcup_{i=1}^m V_i$ is called equitable if and only if 

$$ deg(v,V_i) = deg(v',V_i) \text{ for all } v,v' \in V_j \text{ and all } j \in \{1,\ldots,m\} $$.
\end{definition}

One can proof that in our case the partition, consisting of cosets is equitable.

\begin{proposition}
Let $\Gamma(S_{2n}, W)$ be the Cayley graph.
Then, the set of cosets $\{ gH \}$ is equitable partition.
\end{proposition}
\begin{proof}
Follows proof of Proposition 1 from \cite{borovitskiy2023}.

If $g, g' \in gH$, then there exists $h: g' = gh$.
If there is edge $(g, k) \in \Gamma (S_{2n}, W)$, then $g \in Wk = kW$, therefore exists $w \in W: g = wk$.
Therefore, $g' = gh = wkh$, so there is the edge $(g', kh) \in \Gamma (S_{2n}, W)$.
And vice versa, if there is edge $(g', kh) \in \Gamma (S_{2n}, W)$, then there exists $w \in W: g' = wkh$, and therefore $g = wk$, so there is edge $(g, k) \in \Gamma (S_{2n}, W)$.
\end{proof}

In this case, there is good basis of $L^2(V)$, consisting of eigenfunctions of adjacency matrix that are either constant on cosets or average to zero.
Namely, the following theorem holds.

\begin{theorem}{(Theorem 7 from \cite{borovitskiy2023})}\label{theorem:OrthonormalBasisOfEquitableFunctions}
Consider an equitable partition $V= \bigcup_{i=1}^{m} V_i$.
There is orthonormal basis $\{ f_j \}_{j=1}^{|V|}$ of $L^2(V)$ consisting of eigenfunctions of $\m{A}_G$ split in two groups $\{1, \ldots, |V|\} = \Lambda_1 \cap \Lambda_2$ such that
$$ f_j |_{V_i} = c_{ij} \text{ if } j \in \Lambda_1, \qquad \sum_{v \in V_i} f_j(v) = 0 \text{ if } j \in \Lambda_2, \qquad 1 \leq i \leq m . $$
Moreover, $|\Lambda_1| = m$ and $|\Lambda_2| = |V| - m$.
In words, there are $m$ functions in $\Lambda_1$, all of which are piecewise-constant on the partition, and any function in $\Lambda_2$ has zero average in each part of the partition.
\end{theorem}
\begin{proof}
See \cite{borovitskiy2023}
\end{proof}

\begin{proposition}
Let $\Pr$ be defined as following
$$(\Pr f) (g) = \frac{1}{|H|}\sum_{h \in H} f(gh) \qquad \Pr: L^2(S_{2n}) \rightarrow L^2(S_{2n}) , $$
then $\Pr$ is orthogonal projection onto space of $H$-invariant functions $L^2(\mathcal{X}_n)$.
\end{proposition}
\begin{proof}
First, it is easy to see that $\Pr \Pr f = \Pr f$, and if $f \in L^2(\mathcal{X}_n)$, then $\Pr f = f$.

So, the only thing we need to proof is that if $f \in L^2(\mathcal{X}_n)^\perp$, then $\Pr f = 0$.
Indeed, let $f \in L^2(X)^\perp$, then
$$
(\Pr f)(g) = \frac{1}{|H|} \sum_{h \in H} f(gh) = \frac{1}{|H|} \sum_{g' \in gH} f(g') = \frac{1}{|H|} \langle f , 1_{gH} \rangle = 0 .
$$
\end{proof}

\begin{theorem}{(Theorem 5 from \cite{borovitskiy2023}})\label{thm:two_quotient_kernels}
Consider the $\Phi$-kernel $k$, i.e., kernel of the form like in \Cref{eqn:graph_kernels}, induced by the symmetric normalized Laplacian.
Consider it's $H$-invariant version $k_{/H}$ obtained through averaging.
Then, $k_{/H}(g_1, g_2) = \psi(g_1 H) \psi (g_2 H) k_{\Phi} (g_1 H, g_2 H)$, where $k_{\Phi}$ is the $\Phi$-kernel on the quotient graph $\Gamma_{/H}$ with the same function $\Phi$.
Moreover, $\psi(g H) = |gH|^{-1/2} = |H|^{-1/2}$.
\end{theorem}
\begin{proof}

Using the Theorem \ref{theorem:OrthonormalBasisOfEquitableFunctions}, one can obtain an orthonormal basis $\{f_j\}_{j=1}^{|S_{2n}|}$ of eigenfunctions of adjacency matrix $\m{A}$ such that $f_1, \ldots, f_{|\mathcal{X}_{n}|}$ are constant on all cosets and the rest average to zero over all cosets.
Let $\lambda_j$ be the eigenvalue of $\m{A}$, corresponding to $f_j$.
Then, the $f_j$ is eigenvector of normalized symmetric Laplacian $\Lambda_{sym}$ with eigenvalue $\lambda_j^{sym} = 1 - \lambda_j/|W|$, where $W$ is the generating set of Cayley graph.
Then, 
$$ k(g_1, g_2) = \sum_{j=1}^{|\mathcal{X}_{n}|} \Phi (\lambda_j^{sym}) f_j (g_1) f_j (g_2) + \sum_{j=|\mathcal{X}_{n}|+1}^{|S_{2n}|} \Phi (\lambda_j^{sym}) f_j (g_1) f_j (g_2) . $$
Remember that $\Pr$ is orthogonal projector, and therefore
\begin{align*}
k_{/H}(g_1, g_2) &= \sum_{j=1}^{|X|} \Phi (\lambda_j^{sym}) (\Pr f_j) (g_1) (\Pr f_j) (g_2) + \sum_{j=|\mathcal{X}_{n}|+1}^{|S_{2n}|} \Phi (\lambda_j^{sym}) (\Pr f_j) (g_1) (\Pr f_j) (g_2) \\
&= \sum_{j=1}^{|\mathcal{X}_{n}|} \Phi (\lambda_j^{sym}) f_j (g_1) f_j (g_2) .
\end{align*}

By $\m{A}_{/H}$ we will denote adjacency matrix of quotient Cayley graph $\Gamma_{/H} (S_{2n}, W)$.
By $\m{S}$ we will denote the $|S_{2n}| \times |\mathcal{X}_n|$ matrix with $\m{S}_{g_1, g_2H} = 1$ if $g_1 \in g_2 H$ and $\m{S}_{g_1, g_2H} = 0$ otherwise.
It is easy to see that $\m{A}_{/H} = \m{S}^T \m{A} \m{S}$.

Computing the corresponding degree matrix yields
$$(\m{D}_{/H})_{g_1 H, g_1 H} = \sum_{g_2 H \in X} (\m{A}_{/H})_{g_1 H, g_2 H} = \sum_{g_2 H \in X} \sum_{g_1' \in g_1 H} \sum_{g_2' \in g_2 H} \m{A}_{g_1', g_2'} = \sum_{g_1' \in g_1 H} \sum_{g_2' \in S_{2n}} \m{A}_{g_1', g_2'} = |H| \cdot |W| , $$
so $\m{D}_{/H} = |H| \cdot |W| \cdot \m{I}$.
Moreover, $\frac{1}{|H|}\m{S}\m{S}^T f_j = f_j$ for $j=1, \ldots, |\mathcal{X}_n|$.
Define $u_j = |H|^{-1/2} \m{S}^T f_j \in L^2(\mathcal{X}_n)$ and $\m{\Delta}_{sym, /H} = \m{I} - \m{D}_{/H}^{-1/2} \m{A}_{/H} \m{D}_{/H}^{-1/2}$.
Then,

\begin{align*}
\m{\Delta}_{sym, /H} u_j &= u_j - \frac{1}{|H| \cdot |W|} \m{S}^T \m{A} \m{S} u_j\\
&= u_j - \frac{1}{|H|^{1.5} \cdot |W|} \m{S}^T \m{A} \m{S} \m{S}^T f_j \\
&= u_j - \frac{1}{|H|^{0.5} \cdot |W|} \m{S}^T \m{A} f_j \\
&= u_j - \frac{\lambda_j}{|H|^{0.5} \cdot |W|} \m{S}^T f_j \\
&= u_j - \frac{\lambda_j}{|W|} u_j \\
&= (1 - \lambda_j / |W|) u_j = \lambda_j^{sym} u_j .
\end{align*}

Hence, the $\Phi$-kernel $k_{\Phi}$ on $\Gamma_{/H} (S_{2n}, W)$ corresponding to the same $\Phi$ and the symmetric normalized Laplacian is given by 
$$ k_{\Phi}(g_1 H, g_2 H) = \sum_{j=1}^{|\mathcal{X}_{n}|} \Phi(\lambda_j^{sym}) u_j (g_1 H) u_j (g_2 H) . $$

It is easy to see that, $f_j (g) = \psi(g H) u_j (g H)$ for $j \in \{ 1, \ldots, |\mathcal{X}_n|\}$, where $\psi (g H) = |g H|^{-1/2} = |H|^{-1/2}$.
Thus, 
$$k_{/H} (g_1 H, g_2 H) = \psi(g_1 H) \psi(g_2 H) k_{\Phi} (g_1 H, g_2 H)$$

\end{proof}

\paragraph{Degree} 

As we will further see, the quotient Cayley graph is regular, so all nodes have the same degree.
We will further need the exact value of degree in order to make a degree correction to one of the hyperparameters of the defined kernel.
So, here we will compute the (unweighted) degree.

Consider a node in the Cayley quotient graph, i.e., a coset $\sigma H$ or, equivalently, a matching $x \in \mathcal{X}_n$.
There will be an edge between $x$ and $y \in \mathcal{X}_n$ if and only if there is transposition $\pi \in W$ such that $y = \pi \lacts x$.
Let us analyze how $\pi \lacts x$ looks like for all transpositions $\pi$.

Note that if $\pi = (i \, j)$ such that the pair $\{i, j\} \in x$, then $\pi \lacts x = x$.
Otherwise, there are two pairs $\{i, k\}, \{j, l\} \in x$, and the generalized distance (see \Cref{sec:computation:existing} for definition) $d(x, \pi \lacts x) = (2, 1 \ldots, 1)$ since $\pi \lacts x$ has exactly the same pairs as $x$ except $\{i, k\}, \{j, l\}$, which are replaced with $\{j, k\}, \{i, l\}$.
And vice versa, if $d(x, y) = (2, 1, \ldots, 1)$, then there is transposition $\pi$ such that $y = \pi \lacts x$.
So, degree of $x$ is given by $| \{ y \in X \mid d(x, y) = (2, 1, \ldots, 1) \} |$.

It is easy to see that $| \{ y \in X \mid d(x, y) = (2, 1, \ldots, 1) \} | = n(n-1)$ since one should select two pairs in $x$, and for each selection there are two ways to swap elements in them.
So, there are $\binom{n}{2} \cdot 2 = n(n-1)$ matchings that can be formed from $x$ by applying a single transposition.

\section{Computation}

\subsection{Representation Dimension and Laplacian Eigenvalues}

\label{appendix:eigenvalues_and_dim}

In this section, we will employ the machinery of the representation theory of symmetric groups to compute the Laplacian eigenvalues $\lambda_\rho$ and representation dimensions $d_\rho$, whenever $\rho$ is partition of $2n$.
For readers that are not familiar with the representation theory of symmetric groups,  we refer to \citet{kondor2008, ceccherinisilberstein2008}.
Two key results that we will use are the hook length formula and Murnaghan--Nakayama rule.

\begin{theorem}[Hook length formula]
Let $\mu$ be a partition of $m$ and $\chi^{(\mu)}$ the character of the representation corresponding to $\mu$.
Then, 
\[
\chi^{(\mu)} ((1, \ldots, 1)) = \frac{m!}{\prod_{x \in Y(\mu)} \text{hook}(y)} ,
\]
where $Y(\mu)$ is the set of cells of the Young diagram of the form $\mu$, $\text{hook}(x)$ is the hook length of the cell $x$, that is, the number of cells to the right of $x$ plus the number of cells below the cell $x$ plus one (the cell $x$ itself).
\end{theorem}
\begin{proof}
See Section 4.3 in \citet{fulton1996}
\end{proof}

In software implementations, we represent a partition as a tuple $\mu = (\mu_1, \ldots, \mu_s)$, and represent each cell as a pair of indices $(i, j)$, where $1 \leq i \leq s$ is the row index and $1 \leq j \leq \mu_i$ is the column index.
Then, the hook length is given by $1 + (\mu_i - j) + (\mu'_j - i)$, where $\mu'$ is the conjugate partition.
Thus, the implementation of the hook length formula is straightforward and does not require any specialized packages.

\begin{theorem}[Murnaghan--Nakayama Rule]\label{theorem:MurnaghanNakayama Rule}
Let $\lambda = (\lambda_1, \ldots, \lambda_s)$ and $\mu = (\mu_1, \ldots, \mu_r)$ be partitions of $n$.
Then,
\[
\chi^{(\lambda)} (\mu) = \sum_{\xi \in BS(\lambda, \mu_1)} (-1)^{ht(\xi)} \chi^{(\lambda \backslash \xi)} (\mu \backslash \mu_1) ,
\]
where 
\begin{itemize}
\item $BS(\lambda, \mu_1)$ is a set of border strips of size $\mu_1$ (that is, connected skew-shapes of the form $(\lambda, \bullet)$ that do not contain the square $2 \times 2$), the removal of which gives a correct Young diagram.
\item $ht(\xi)$ --- the height of the border strip, that is, the number $1$ less than the number of rows in $\xi$.
\item$\lambda\backslash\xi$ means the partition that remains after removal of border strip $\xi$ from the Young diagram of shape $\lambda$.
\item $\mu\backslash \mu_1$ means the partition that remains after removal of the first element from the $\mu$.
\end{itemize}
\end{theorem}
\begin{proof}
See Theorem 7.17.1 in \cite{stanley2001}
\end{proof}

One can note that a border strip for a given partition is uniquely identified by its starting row and its size.
In our software implementation, in order to iterate over border strips, we simply iterate over all rows and check whether there is a valid border strip starting from that row.
The software implementation of this formula uses simple constructs such as loops and arrays and does not require any specialized packages.

Using the hook length formula and the fact that $d_\rho = \chi^{(\rho)}(e) = \chi^{(\rho)}((1, \ldots, 1))$, one can compute $d_\rho$ with $O(n)$ time complexity.
As for the eigenvalues $\lambda_\rho$, the following proposition holds.

\begin{restatable}{proposition}{LaplaceEigenvaluesFormula}
Let $\rho$ be an irreducible unitary representation of the group $S_{2n}$ and $d_\rho$ be its dimension.
Then, the eigenvalue of the Laplacian of the Cayley graph corresponding to the representation of $\rho$ has the form
\[
\lambda_\rho = \frac{|W|}{d_\rho} \left[ d_\rho - \chi^{(\rho)}((2, 1, \ldots, 1)) \right] ,
\]
where the notation $\chi^{(\rho)}((2, 1, \ldots, 1))$ means the value of $\chi^{(\rho)}$ on some permutation having the cycle type $(2, 1, \ldots, 1)\vdash 2n$.\footnote{Remember that a value of character depends only on cycle type of permutation.}
\end{restatable}
\begin{proof}

Let $u(g) = \Delta_{g, e}$.
Substituting $\Phi(\lambda) = \lambda$ to Proposition \ref{thm:kernels_character_sum}, we have
$$ u(g) = \Delta_{g, e} = \sum_{\rho} \lambda_\rho \frac{d_\rho}{|G|} \chi^{(\rho)} (g) = \sum_{\rho} \lambda_\rho \frac{d_\rho}{|G|} \chi^{(\rho)} (g) . $$
At the same time, according to the definition of Laplacian,
\[
u(g) = \Delta_{g, e} = \begin{cases}
|W|, g = e \\
-1, g \in W
\end{cases} .
\]
Therefore,
$$ |W| \1_{\{e\}} - \1_W = u = \sum_{\rho} \lambda_\rho \frac{d_\rho}{|G|} \chi^{(\rho)} . $$
Recall that $\left\{\frac{1}{\sqrt{|G|}} \chi^{(\rho)} \right\}_{\rho}$ is orthonormal system.
Therefore, $\lambda_\rho$ can be calculated as inner products:
\begin{align*}
\lambda_\rho 
&= \frac{\sqrt{G}}{d_\rho} \left\langle |W|\1_{\{e\}} - \1_W , \frac{1}{\sqrt{|G|}} \chi^{(\rho)} \right\rangle 
= \frac{1}{d_\rho} \left \langle |W|\1_{\{e\}} - \1_W , \chi^{(\rho)} \right\rangle \\
&= \frac{1}{d_\rho} \left[ |W| \chi^{(\rho)} (e) - \sum_{g \in W} \chi^{(\rho)}(g)  \right] 
= \frac{1}{d_\rho} \left[ |W| d_\rho - \sum_{g \in W} \chi^{(\rho)}(g)  \right] \\
&= \frac{1}{d_\rho} \left[ |W| d_\rho - |W| \chi^{(\rho)}((2, 1, \ldots, 1)) \right] = \frac{|W|}{d_\rho} \left[ d_\rho - \chi^{(\rho)} ((2, 1, \ldots, 1)) \right] ,
\end{align*}
where, in the second last equality, we used the fact that all transpositions have a cycle type $(2, 1, \ldots, 1)$.

\end{proof}

As was mentioned above, $d_\rho$ can be efficiently computed using the hook length formula.
And for $\chi^{(\rho)}((2, 1, \ldots, 1))$, we can apply the Murnaghan-Nakayama Rule to show that
\[
\chi^{(\rho)}((2, 1, \ldots, 1)) = \sum_{\xi \in BS(\rho, 2)} (-1)^{ht(\xi)} \chi^{(\rho \backslash \xi)} ((1, \ldots, 1)) .
\]
As stated earlier, it can be shown that the size $|BS(\rho, 2)|$ is linear in terms of $n$.
And therefore, computation of eigenvalues reduces to computation of $O(n)$ representation dimensions $\chi^{(\rho \backslash \xi)}$, so $\lambda_\rho$ can be computed with $O(n^2)$ time complexity.

\subsection{Zonal Spherical Functions\texorpdfstring{: \Cref{thm:character_projection}}{}}\label{appendix:zsf_averaging}

In this section, we will describe time and memory complexity of computing zonal spherical functions via \Cref{thm:character_projection} and discuss its implementation in practice.
First, let us note that since characters are class functions and $H_{n}$ is subgroup of $S_{2n}$, one can simplify formula from \Cref{thm:character_projection}:
\[
\phi_{\v{\rho}} (\sigma) 
= \frac{1}{|H_n|^2} \sum_{\pi_1 , \pi_2 \in H_n} \chi^{(2\v{\rho})} (\pi_1 \sigma \pi_2)
= \frac{1}{|H_n|^2} \sum_{\pi_1 , \pi_2 \in H_n} \chi^{(2\v{\rho})} (\pi_2 \pi_1 \sigma)
= \frac{1}{|H_n|} \sum_{\pi \in H_n} \chi^{(2\v{\rho})} (\pi \sigma) .
\]
Thus, this approach essentially reduces to computing $|H_n| = 2^n n!$ values of characters of $S_{2n}$. 
For computing them, one can use, for example, Murnaghan-Nakayama rule discussed above or other approaches.

We use $T(n)$ and $M(n)$ to denote amortized time and memory complexity of computing a single value of character of the group $S_{2n}$.
Then, query time complexity is $O(|\c{R}| \cdot 2^n n! \cdot T(n))$ and query memory complexity is $O(M(n))$.
Please note that even if $T(n) = O(1)$, query time is at least $O(|\c{R}| \cdot 2^n n!)$, making this approach quickly infeasible in practice.

Unlike other approaches, precomputation time and memory complexity are dominated by precomputing representation dimensions and Laplacian eigenvalues.
Thus, precomputation time complexity is $O(|\c{R}| n^2)$ and memory complexity is $O(|\c{R}| + N)$ since we need to store $|\c{R}|$ values and computing each of them requires $O(n)$ memory.

For fair comparison, similar to other approaches we augment this one with caching mechanism that caches values of zonal spherical functions after computing them.

\subsection{Zonal Spherical Functions: Explicit Formula}\label{appendix:zsf_naive}

\input{algorithms/zsf-naive}

In this section, we will provide an explicit formula for zonal spherical functions, and describe the algorithm that directly implements this formula~\citep{ceccherinisilberstein2008}.
We will then also give the time complexity of that algorithm.

The present section extensively uses the representation theory of symmetric groups.
Readers unfamiliar with this subject are referred to Chapter 10 in \citep{ceccherinisilberstein2008}.
To begin, let us provide some supplementary definitions that will be useful for the explicit formula of zonal spherical functions.

\begin{definition}
Let $\lambda = (\lambda_1, \ldots, \lambda_r)$ be a partition of $n$, $\{t\} \in \mathfrak{T}_{2\lambda}$ be a Young tabloid of the form $2\lambda = (2\lambda_1, \ldots, 2\lambda_r)$, $x = \{\{i_1, i_2\}, \ldots, \{i_{2n-1}, i_{2n}\}\} \in X$ a matching on $2n$ points.
We will say that $\{t\}$ covers $x$ if and only if for all $h \in \lbrace 1,\ldots, n\rbrace$, $i_{2h-1}$ and $i_{2h}$ belong to the same row in ~$\{t\}$.
\end{definition}

\begin{definition}
Let $x_0 = \{\{i_1, i_2\}, \ldots, \{i_{2n-1}, i_{2n}\}\}\in X$ be a matching on $2n$ points.
We will say that the Young tableau $t$ of the form $2\lambda\vdash 2n$ is $x_0$-good if for all $h \in \lbrace 1,\ldots, n\rbrace$, $i_{2h-1}$ and $i_{2h}$ are in the same row, and moreover belong to adjacent columns of $t$.
\end{definition}

Then, the following theorem gives the explicit formula for zonal spherical functions (ZSF).

\begin{theorem}\label{theorem:SphericalFunctionIsSphereIndicatorSum}
Let $X$ be the space of perfect matchings on $2n$ points and $x_0\in X$.
Also, let $H\cong S_2\wr S_n$ be the stabilizer of $x_0$, and $t$ an $x_0$-good Young tableau of the form $2\lambda$.
Then, the zonal spherical function $\phi_\lambda$ corresponding to the spherical representation $D_\lambda S^{2\lambda}$ has the following expression
\[
\phi_\lambda(x) = \sum_{\mu \vdash n} \frac{a_\mu^\lambda}{|A_\mu|} \1_{A_\mu}(x) ,
\]
where 
\[
A_\mu = \{ x \in X \mid d(x, x_0) = \mu \} ,
\]
\[
a_\mu^\lambda = \sum_{\sigma \in C'_t} \sgn (\sigma) |\{ x \in A_\mu \mid \{\sigma \lacts t\} \text{ covers } x\}| . 
\]
Here $C'_t$ is a subgroup of the column stabilizer $C_t$, consisting of permutations that fix each number standing in an even column.
\end{theorem}
\begin{proof}
See section 11.3 in \citep{ceccherinisilberstein2008}
\end{proof}

Now, let us describe the algorithm for computing zonal spherical functions, which is an immediate consequence of the above theorem.
Although the formula for $\phi_\lambda(x)$ contains a sum over all partitions of $n$, there is in fact only one non-zero term for each fixed $x\in X$.
Indeed, the terms of this sum are disjoint indicators multiplied by some constants.
Therefore, to compute $\phi_\lambda(x)$, it is enough to compute $d(x,x_0)$ and, if this turns out to be equal to $\mu$, return $\frac{a_\mu^\lambda}{|A_\mu|}$.
The values of $a_\mu^\lambda$ and $|A_\mu|$ can be computed once and for all at the pre-processing stage.
Algorithm \ref{algorithm:ZSFNaive} contains a pseudo-code description of how this can be done.

Now, let us estimate the time complexity of Algorithm \ref{algorithm:ZSFNaive}.
This algorithm iterates over all possible pairs $(x, \sigma) \in X\times C'_t$, and computes $\sgn(\sigma)$ and $\{\sigma \lacts t\}$ for each pair.
Thus, it has time complexity $O(|X|\cdot|C'_t|\cdot n)$.
Recall that $C'_t$ is a subgroup of the column stabilizer $C_t$, consisting of permutations that also fix all elements of even columns, and therefore\footnote{Recall that $t$ is a Young tableau of shape $2\lambda$, and $(2\lambda)' = (\lambda'_1, \lambda'_1, \lambda'_2, \lambda'_2, \ldots, \lambda'_r, \lambda'_r)$.} $|C'_t| = \lambda'_1! \cdot \ldots \cdot \lambda'_r!$, where $\lambda' = (\lambda'_1, \ldots, \lambda'_r)$ is the conjugate partition of $\lambda$.
Therefore, the time required to process one zonal spherical function is
\[
O(|X| \cdot |C'_t| \cdot n) = O\left( \frac{(2n)!}{2^n n!} \lambda'_1! \cdot \ldots \cdot \lambda'_r! \cdot n \right) ,
\]
and the processing of all zonal spherical functions requires 
\[
O \left( p(n) \cdot \frac{(2n)!}{2^n n!} \lambda'_1! \cdot \ldots \cdot \lambda'_r! \cdot n  \right) = O\left( \frac{(2n)! \cdot n \cdot p(n)}{2^n} \right) ,
\]
where $p(n)$ is the number of partitions of $n$. Note that we used the fact that $\lambda'_1! \cdot \ldots \cdot \lambda'_r! \leq (\lambda'_1 + \ldots + \lambda'_r)! = n!$.
Moreover, if $\lambda = (1, \ldots, 1)$, then $\lambda' = (n)$, and $\lambda'_1 =n!$, which means that, in the worst case scenario, the~estimate of $n!$ is accurate.

Finally, we describe an optimization that substantially improves the performance of the algorithm. 
Instead of relying on full precomputation, we combine on-the-fly computation with a caching mechanism. 
The key observation is that, to compute $a^{\lambda}_{\mu}$, it is sufficient to iterate only over elements $x \in A_{\v{\mu}}$. 
Moreover, the proof of \Cref{lemma:generalized_sphere_size} gives an explicit combinatorial description of the set $A_{\v{\mu}}$. 
This description allows us to enumerate only the relevant elements $x \in A_{\v{\mu}}$, rather than all elements $x$ considered in \Cref{algorithm:ZSFNaive}. 
As a result, the computational cost can be reduced significantly.

\subsection{Zonal Spherical Functions: Our Algorithm}\label{appendix:zsf_our_algorithm}

In this section, we provide further details on our algorithm for computing zonal spherical functions using their relationship with zonal polynomials. 
We outline the main computational steps and analyze their computational and memory complexity.

\paragraph{General algorithm}

As discussed in the main text, our algorithm is grounded in the relationship between zonal spherical functions and zonal polynomials.
In particular, the following theorem states that the values of a zonal spherical function can be directly obtained from the coefficients in the expansion of a zonal polynomial into the basis of products of power-sum symmetric polynomials.

\ZpCoefsAreZsf*
\begin{proof}

This theorem follows immediately from the following key equality:
\[
\mathcal{C}_\rho = c \cdot \sum_{\sigma \in S_{2n}} \phi_{\v{\rho}}(\sigma) p_{d(\sigma x_0, x_0)}
\]
for some nonzero constant $c = c(\rho)$.
Here, $\mathcal{C}_{\v{\rho}}$ denotes the zonal polynomial indexed by the partition $\v{\rho}$, $\phi_{\v{\rho}}$ is the corresponding zonal spherical function, and $d(\sigma x_0, x_0)$ is the partition-valued ``generalized distance'' between matchings.
For a proof of this equality, see Section 4 of \cite{bergeron1992} and \cite{james1961}.

Next, since $\phi_{\v{\rho}}(\sigma)$ depends only on the value $d(\sigma x_0, x_0)$, we can regroup the terms according to all possible partitions $\v{\mu} \vdash n$:
\begin{align*}
\mathcal{C}_{\v{\rho}} &= c \cdot \sum_{\sigma \in S_{2n}} \phi_{\v{\rho}}(\sigma) p_{d(x, x_0)} \\
&= c \cdot \sum_{\v{\mu} \vdash n} \left( \sum_{\sigma \in S_{2n}: \, d(\sigma x_0, x_0) = \v{\mu}} \phi_\rho(\sigma) \right) p_{\v{\mu}} \\
&= \sum_{\v{\mu} \vdash n} \left(c \cdot \phi_\rho(\v{\mu}) |A_{\v{\mu}}| \cdot |H| \right) p_{\v{\mu}} \\
&= \sum_{\v{\mu} \vdash n} \left(\tilde{c} \cdot \phi_\rho(\v{\mu}) |A_{\v{\mu}}| \right) p_{\v{\mu}} ,
\end{align*}
where, similarly to the case of characters, $\phi_{\v{\rho}} (\v{\mu})$ means the value of $\phi_{\v{\rho}}(x)$ for some $x \in \mathcal{X}_n$ with $d(x, x_0) = \v{\mu}$.
This completes the proof.\end{proof}

To complete the result, we now compute the value of $|A_{\v{\mu}}|$ in the following lemma.

\begin{lemma}\label{lemma:generalized_sphere_size}
Let $\mu$ be a partition of $n$ and let $t_k = |\{ j : \mu_j = k \}|$ be the multiplicity of $k$ in $\mu$.
Then,
\[
|A_{\v{\mu}}| = n! \prod_k \frac{2^{t_k (k-1)}}{t_k! k^{t_k}} .
\]
Hence, the size of each generalized sphere $|A_\mu|$ can be calculated in $O(n)$ time.
\end{lemma}
\begin{proof}
For any $y\in \c{X}_n$, $x_0\cup y$ is a disjoint union of even cycles with edges alternating
between $x_0$-edges and $y$-edges.
If a connected component has $2k$ vertices, then it uses exactly $k$
distinct pairs of $x_0$.
Thus $d(x_0,y)=\mu$ if and only if $x_0\cup y$ has exactly $t_k$ components that
each use $k$ distinct $x_0$-pairs, for every $k$.
We count such $y$ in two steps.

First, choose the cyclic order of the $x_0$-pairs inside each cycle.
Since the $n$ pairs of $x_0$ are labeled, the number of ways to arrange them into cycles with cycle-type specified by $\mu$ equals $\frac{n!}{\prod_{k} k^{t_k}\,t_k!}$.
Here, division by $k^{t_k}$ accounts for cyclic rotations inside each cycle of length $k$, and division by $t_k!$ accounts for the indistinguishability of the $t_k$ cycles of the same length $k$.

Second, for each fixed cycle of $k$ $x_0$-pairs, there are exactly $2^{k-1}$ ways to realize the $y$-edges along that cycle.
Indeed, once the cyclic order of the $k$ pairs is fixed, choosing for one pair which of its two elements connects to the ``next'' pair fixes the parity of connections.
The remaining $k-1$ pairs each admit an
independent binary choice, yielding $2^{k-1}$ possibilities.
Over the $t_k$ cycles of length $k$ this contributes
$2^{t_k (k-1)}$.

Multiplying these two steps gives,
\[
|A_\mu|
\;=\;
\Biggl(\frac{n!}{\prod_{k} k^{t_k}\,t_k!}\Biggr)
\cdot \prod_{k} 2^{t_k (k-1)}
\;=\;
n!\;\prod_{k} \frac{2^{\,t_k (k-1)}}{t_k!\,k^{t_k}}\,.
\]
which is the desired result.
\end{proof}

In summary, the key idea of our algorithm is to use the Theorem above to compute coefficients of the $C_{\v{\rho}}$ expansion in the $p_{\v{\mu}}$ basis.
However, direct computation of these coefficients is challenging.
Therefore, we first compute the coefficients in the $m_{\v{\kappa}}$ (monomial symmetric polynomial) basis, and then convert these to the $p_{\v{\mu}}$ basis using a change of basis.
Both steps are described in detail below.

\paragraph{Decomposition in terms of $m_{\v{\kappa}}$}

The following theorem establishes recurrence relations for the coefficients \( c_{\v{\rho}, \v{\kappa}} \) in the expansion of the zonal polynomial \( \mathcal{C}_{\v{\rho}} \) along the basis of monomial symmetric polynomials \( m_{\v{\kappa}} \).
This recurrence also allows for an efficient dynamic programming algorithm for their computation, which we will discuss below.

\begin{theorem}
Let $\v{\rho}\vdash n$, $m\geq n$.
Then, the zonal polynomial $\mathcal{C}_{\v{\rho}}$ can be represented as
\[ 
\mathcal{C}_{{\v{\rho}}}(z_1, \ldots, z_m) = \sum_{\v{\kappa} \leq {\v{\rho}}} c_{{\v{\rho}}, {\v{\kappa}}} \, m_{\v{\kappa}}(z_1, \ldots, z_m) , 
\]
where the order on the partitions is lexicographic, and the coefficients $c_{{\v{\rho}}, {\v{\kappa}}}$ are uniquely determined from the recurrence  relations
\begin{equation}\label{eqn:ZPCoefInitialValue}
\sum_{{\v{\rho}}:\, {\v{\kappa}} \leq {\v{\rho}} \leq (n)} c_{{\v{\rho}}, {\v{\kappa}}} = \frac{n!}{\kappa_1! \ldots \kappa_s!} ,
\end{equation}
and
\begin{equation}\label{eqn:ZPCoefRecurrence}
c_{{\v{\rho}}, {\v{\kappa}}} = \sum_{{\v{\mu}}:\, {\v{\kappa}} < {\v{\mu}} \leq {\v{\rho}}} \frac{(\kappa_r + t) - (\kappa_l - t)}{f_{\v{\rho}} - f_{\v{\kappa}}} c_{{\v{\rho}}, {\v{\mu}}} \quad\quad\quad f_{{\v{\kappa}}} = \sum_{i=1}^{s} \kappa_i (\kappa_i - i) ,
\end{equation}
where the summation is over all $\v{\mu} = (\kappa_1,\ldots, \kappa_{r-1}, \kappa_r + t, \kappa_{r+1}, \ldots, \kappa_{l-1}, \kappa_l - t, \kappa_{l+1}, \ldots, \kappa_s), t=1, \ldots, \kappa_l$ such such that after ordering ${\v{\mu}}$ in non-decreasing order, ${\v{\kappa}} < {\v{\mu}} \leq {\v{\rho}}$.
\end{theorem}
\begin{proof}
See, for example, section 5 in \cite{james1968} or section 7.2 in \cite{muirhead1982}.
\end{proof}

We now describe how these recurrence relations facilitate an efficient dynamic programming scheme for computing all coefficients $c_{\v{\rho}, \v{\kappa}}$ for a fixed $\v{\rho}$.
The key observation is that, in the recurrence \eqref{eqn:ZPCoefRecurrence}, the right-hand side depends only on coefficients indexed by partitions $\v{\mu}$ that are lexicographically larger than $\v{\kappa}$.
Thus, by initializing $c_{\v{\rho}, \v{\rho}} = 1$ (remember that we will anyway re-normalize coefficients $c_{\v{\rho}, \v{\kappa}}$ later, so it is sufficient to compute them only up to some constant), and iterating through the partitions $\v{\kappa}$ in decreasing lexicographic order, we ensure that all required coefficient values for $\v{\mu} > \v{\kappa}$ have been computed previously.

\paragraph{Change of basis}

We now describe the computation of the transition matrix $T_{\v{\kappa}, \v{\mu}}$, which expresses monomial symmetric polynomials $m_{\v{\kappa}}$ in terms of products of power sum symmetric polynomials $p_{\v{\mu}}$.
In~\citet{merca2015}, the authors present an efficient algorithm for decomposing $m_{\v{\kappa}}$ with respect to the $p_{\v{\mu}}$ basis.
For completeness, we describe their algorithm here; see Algorithm~\ref{algorithm:AugmentedMonomials} for details.

In this algorithm, it is convenient to first define \textit{augmented} monomial symmetric polynomials as $\widetilde{m}_{\v{\kappa}} = t_1! \cdot t_2! \cdots t_n! \cdot m_{\v{\kappa}}$, where $t_j$ denotes the multiplicity of $j$ in the partition $\v{\kappa}$.
After computing the decomposition of $\widetilde{m}_{\v{\kappa}}$ in terms of $p_{\v{\mu}}$, one can readily recover the coefficients for the standard $m_{\v{\kappa}}$.

Compared to the original algorithm, we suggest a minor but effective modification: we cache all intermediate values of the recursive function during computation.
This prevents recomputation of already-calculated values, leading to a significant efficiency improvement in practice.

\input{algorithms/augmented-monomials}

\paragraph{Computational complexity}

We now analyze the computational complexity of evaluating a single value of zonal spherical function using \Cref{algorithm:ZSFUsingZP}.

First, consider the computation of $c_{\v{\rho}, \v{\kappa}}$ using the dynamic programming procedure described above.
According to the recurrence in \Cref{eqn:ZPCoefRecurrence}, the sum on the right-hand side involves at most $n^2$ terms: for each $l$, there are at most $\kappa_l$ possible values of $t$, and since $\sum_{i=1}^{s} \kappa_i = n$, the total number of choices over all $l$ is bounded by $n$.
Each $(l,t)$ pair can be combined with at most $n$ possible values of $r$.
Therefore, provided all required $c_{\v{\rho}, \v{\mu}}$ values on the right-hand side are already computed, each individual $c_{\v{\rho}, \v{\kappa}}$ can be evaluated in $O(n^3)$ time since we have $O(n^2)$ terms and for each term we need $O(n)$ operations.
Finally, computing $c_{\v{\rho}, \v{\kappa}}$ for a fixed partition $\v{\rho}$ and all possible $\v{\kappa}$ has $O(p(n) \, n^3)$ complexity.

Second, consider the computation of the transition matrix $T_{\v{\kappa}, \v{\mu}}$.
For complexity analysis, we assume that all matrix entries $T_{\v{\kappa}, \v{\mu}}$ are precomputed at the start of the algorithm.
This yields an upper bound on the total computational cost, and in practice, the cost of precomputation is amortized across all uses thanks to caching.
There are $O(p(1)^2 + p(2)^2 + \ldots + p(n)^2) = O(n \cdot p(n)^2)$ possible pairs $(\v{\kappa}, \v{\mu})$, since both are partitions of $m$ (where $m \leq n$).
Each matrix element can be computed in $O(n^2)$ time, given cached decompositions.
Therefore, the overall cost to precompute the transition matrix is $O(p(n)^2\, n^3)$.

Finally, to compute a single value of a zonal spherical function using \Cref{algorithm:ZSFUsingZP}, we use the previously computed $c_{\v{\rho},\v{\kappa}}$ and $T_{\v{\kappa},\v{\mu}}$ to evaluate coefficients $b_{\v{\rho},\v{\mu}}$.
For each fixed $\v{\rho}$, this requires iterating over all partitions $\v{\kappa}$ and performing $O(n)$ operations for each, giving a complexity of $O(p(n) n)$ for this step.
Hence, an overall complexity of computing $b_{\v{\rho}, \v{\mu}}$ for a fixed $\v{\rho}$ and $\v{\mu}$ given precomputed $T_{\v{\kappa}, \v{\mu}}$ is $O(p(n) \, n^3 + p(n) \,n) = O(p(n) \, n^3)$.

\paragraph{Memory complexity}

During precomputation, we store transition matrices $T_{\v{\kappa}, \v{\mu}}$ for all partitions of $m \leq n$.
Each entry in these matrices is indexed by a pair of partitions, and since there are $p(m)$ partitions of $m$, the total number of entries is
\[
\sum_{m=1}^n p(m)^2 = O\big(n \cdot p(n)^2\big).
\]
In our implementation, these matrices are represented as hash-maps whose keys are partitions, and each partition is stored as a tuple of size $O(n)$.
As a result, the total memory required for precomputation is $O(p(n)^2 n)$.

During querying, computing coefficients such as $c_{\v{\rho}, \v{\kappa}}$ and $b_{\v{\rho}, \v{\mu}}$ involves vectors of length $O(p(n))$, where each index is a partition stored as a tuple of size $O(n)$.
Therefore, the memory required for queries is $O(p(n) n)$.

\paragraph{Alternative implementations} 

Shortly after developing our algorithm, we discovered that a similar method had already appeared in the literature~\citep{hillier2022}.
However, it was introduced in a completely different context, not related to kernels or matchings.
While both approaches use zonal polynomials, there are important differences in their methods.
For example, to the best of our knowledge, their algorithm does not support on-the-fly computation; instead, it requires precomputing all values of the zonal spherical functions.
In contrast, our algorithm can compute these values as needed.
Additionally, our software is implemented in Python, whereas \citet{hillier2022} implementation uses Matlab.

\section{Approximation Quality}

\subsection{Derivation of the approximation error formula}\label{appendix:approximation_formula_derivation}

In this section, we will derive a formula that will allow us to estimate the approximation error of truncated kernels.
 Remember that the relative approximation error is given by 
 $$
 \frac{\|k - k_\mathcal{R}\|_{L^2(\mathcal{X}_n \times \mathcal{X}_n)}}{\|k\|_{L^2(\mathcal{X}_n \times \mathcal{X}_n)}}
 $$
Our first step here is to replace the space $L^2(\mathcal{X}_n \times \mathcal{X}_n)$ by $L^2(\mathcal{X}_n)$, which is easier to analyze.

\begin{lemma}
Let $f \in L^2(\mathcal{X}_n \times \mathcal{X}_n)$ be a bi-invariant function.
Then,
\[
\| f  \|^2_{L^2(\mathcal{X}_n \times \mathcal{X}_n)} = |\c{X}_n| \cdot \| f(\bullet, x_0) \|^2_{L^2(\mathcal{X}_n)}
\]
\end{lemma}
\begin{proof}
The following holds,
\begin{align*}
\| f \|^2_{L^2(\mathcal{X}_n \times \mathcal{X}_n)}
&= \sum_{x, y \in \mathcal{X}_n} |f(x, y)|^2 
= \frac{|\c{X}_n|^2}{|S_{2n}|^2} \sum_{\sigma, \pi \in S_{2n}} |f(\sigma x_0, \pi x_0)|^2 \\
&= \frac{|\c{X}_n|^2}{|S_{2n}|^2} \sum_{\sigma, \pi \in S_{2n}} |f(\pi^{-1} \sigma x_0, x_0)|^2 
= \frac{|\c{X}_n|^2}{|S_{2n}|} \sum_{\sigma \in S_{2n}} |f(\sigma x_0, x_0)|^2 \\
&= |\c{X}_n| \cdot \| f(\bullet, x_0) \|^2_{L^2(\mathcal{X}_n)} .
\end{align*}
as required in the lemma.
\end{proof}
According to this lemma, we are interested in computing 
$$
\frac{\|k(\bullet, x_0) - k_{\c{R}}(\bullet, x_0)\|_{L^2(\c{X}_n)}}{\|k(\bullet, x_0)\|_{L^2(\c{X}_n)}}
$$
However, since both $k$ and $k - k_{\c{R}}$ are of the form $k_{\c{R}'}$ for some $\c{R}'$ (this is $\c{R}'\ = \{\v{\rho} \vdash n\}$ for $k$ and $\c{R}' = \c{R}^c = \{\v{\rho} \vdash n \mid \v{\rho} \notin \c{R}\}$ for $k - k_{\c{R}}$), it is sufficient to derive a formula for $\|k_{\c{R}}(\bullet, x_0)\|_{L^2(\c{X}_n)}$ with an arbitrary $\c{R} \subseteq \{ \v{\rho} \vdash n\}$.
This is given by the following lemma.

\begin{lemma}
For arbitrary $\c{R} \subseteq \{ \v{\rho} \vdash n \}$, and truncated kernel
\[
k_{\c{R}} (\sigma H_n, \sigma' H_n)
= \sum_{\v{\rho} \in \c{R}}
\Phi(\lambda_{2\v{\rho}})
\,
\frac{d_{2\v{\rho}}}{|S_{2n}|}\, \phi_{\v{\rho}} ((\sigma')^{-1} \sigma),
\]
one has
\[
\| k_{\c{R}} (\bullet, x_0)\|_{L^2(\c{X}_n)} = \frac{|\c{X}_n|}{|S_{2n}|^2} \sum_{\v{\rho} \in \mathcal{R}} \Phi(\lambda_{2\v{\rho}})^2 d_{2\v{\rho}} .
\]
\end{lemma}
\begin{proof}
Remember that $\{ \phi_\rho \}_{\rho \in \mathcal{R}}$ is an orthogonal system in $L(X)$, and $\norm{\phi_\rho}^2_{L^2(\c{X}_n)} = \frac{|\c{X}_n|}{d_{2\v{\rho}}}$ (see Corollary 4.6.4 in \citet{ceccherinisilberstein2008}).
Therefore,
\begin{align*}
\|k_{\mathcal{R}}(\bullet, x_0)\|_{L^2(\c{X}_n)}^2
&= \sum_{\v{\rho} \in \c{R}} \left( \Phi(\lambda_{2\v{\rho}}) \frac{d_{2\v{\rho}}}{|S_{2n}|} \right)^2 \|\phi_{\v{\rho}}\|^2_{L^2(\c{X}_n)}
= \sum_{\v{\rho} \in \c{R}} \left( \Phi(\lambda_{2\v{\rho}}) \frac{d_{2\v{\rho}}}{|S_{2n}|} \right)^2 \frac{|\c{X}_n|}{d_{2\v{\rho}}} \\
&= |\c{X}_n| \sum_{\rho \in \mathcal{R}} \left( \Phi(\lambda_{2\v{\rho}}) \frac{d_{2\v{\rho}}}{|S_{2n}|} \right)^2 \frac{1}{d_{2\v{\rho}}}
= |\c{X}_n| \sum_{\rho \in \mathcal{R}} \frac{\Phi(\lambda_{2\v{\rho}})^2 d_{2\v{\rho}}}{|S_{2n}|^2} \\
&= \frac{|\c{X}_n|}{|S_{2n}|^2} \sum_{\rho \in \mathcal{R}} \Phi(\lambda_{2\v{\rho}})^2 d_{2\v{\rho}} \\
\end{align*} 
which is the required formula.
\end{proof}

The computation of eigenvalues $\lambda_{2\v{\rho}}$ and dimensions $d_{2\v{\rho}}$ is relatively easy, compared to the computation of zonal spherical function values.
Therefore, it is possible to explicitly compute the term corresponding to each partition $\rho$ in the sum above.

\subsection{Additional Results}\label{appendix:additional_results}

In this section, we provide additional results on approximation quality.

Specifically, in \Cref{fig:extended-approximation-error} we provide additional plots of approximation error with a wider range of smoothness hyperparameter $\nu$ and of values of $n$.
As one can see, in all considered cases, $50$ partitions is enough to achieve a relative approximation error of $10^{-2}$ or less.

\Cref{fig:alternative-truncation-heuristics} provides scatter plots of the eigenvalue $\lambda_{2\v{\rho}}$ against various attributes of the partition $\v{\rho} = (\rho_1, \ldots, \rho_s)$.
Namely, its maximal entry $\rho_1$, its length $s$ and its minimal entry $\rho_s$, which can serve as heuristics for selection of partitions in the truncation.
Remember that we are interested in selecting partitions corresponding to the lowest eigenvalues.
One can see that among considered heuristics, the maximal entry $\rho_1$ best separates lower eigenvalues from the rest.

Finally, \Cref{fig:spectral-density} provides plots of ``spectral density'' with and without density correction.
Here, by ``spectral density'' we mean the (log-scaled) impact $\Phi(\lambda_{2\v{\rho}})^2 d_{2\v{\rho}}$ of a specific summand on the squared $L^2$ norm of the kernel $\| k \|^{2}_{L^2(\c{X} \times \c{X})}$.
Informally, high values of spectral density in low-eigenvalue regions corresponds to smoother kernels and vice versa.

As discussed in \Cref{appendix:quotient}, the degrees of all nodes in the quotient Cayley graph are equal to $n(n-1)$.
Therefore, including or excluding degree correction can greatly affect the kernel, and \Cref{fig:spectral-density} supports this statement.
Specifically, without degree correction, the spectral density is concentrated in relatively high-eigenvalue regions, while adding the degree correction makes the density concentrated in low-eigenvalue regions.

\section{Phylogenetic Trees}\label{appendix:ptrees}

\input{figures/small-nni-big-matching-dist}

\BijectionSingleNniBigMatchingDist*
\begin{proof}

We prove this result by constructing a specific counterexample.
\Cref{fig:small-nni-big-matching-dist} illustrates the case where $n=8$.
The geometric discrepancy arises from the label extension procedure (see, \Cref{fig:ptree-numering-example} for example) in the Diaconis--Holmes encoding, where the label assigned to a specific internal node depends on the topology of the entire tree.

Consider the caterpillar tree in \Cref{fig:small-nni-big-matching-dist} with leaves labeled
$1,2,\dots,n+1$.
During label extension, the first unresolved pair is $(3,4)$, so its parent receives the
label $n+2$.
This produces the matching pair $(5,n+2)$.
Continuing inductively along the caterpillar, we
obtain the sequence of pairs
\[
(5,n+2),\ (6,n+3),\ (7,n+4),\ \ldots,\ (n+1,2n-2).
\]

Now apply a single NNI move at the root so that the pair $(1,2)$ is formed (as in
\Cref{fig:small-nni-big-matching-dist}).
Then the parent of $(1,2)$ receives the label $n+2$, while the
parent of $(3,4)$ is created one step later and therefore receives label $n+3$.
Consequently, the pair
involving leaf $5$ becomes $(5,n+3)$, and continuing the procedure yields
\[
(5,n+3),\ (6,n+4),\ (7,n+5),\ \ldots,\ (n+1, 2n-1) .
\]

Comparing the two resulting matchings, we see that all leaves $5,6,\dots,n+1$ are paired with different
partners in the two constructions.
A direct inspection shows that the only common pair is $(3,4)$.
Hence, the two matchings differ in at least $n-1$ pairs.
In the quotient Cayley graph, an edge corresponds to applying a single transposition.
Such a move can
modify at most two pairs of a matching.
Therefore, any path transforming the first matching into the
second must have length at least $(n-1)/2$, and the quotient Cayley graph distance between the two
matchings is at least $(n-1)/2$.

\end{proof}

\begin{figure}[t!]
\begin{subfigure}[b]{0.45\textwidth}
\begin{forest}
[R [1] [$\bullet$ [2] [$\bullet$ [10] [$\bullet$ [9] [$\bullet$ [8] [$\bullet$ [7] [$\bullet$ [6] [$\bullet$ [5] [$\bullet$ [3] [4]]]]]]]]]]
\end{forest}
\end{subfigure}
\begin{subfigure}[b]{0.45\textwidth}
\begin{forest}
[R [$\bullet$ [9] [$\bullet$ [7] [$\bullet$ [5] [$\bullet$ [1] [2]]]]] [$\bullet$ [10] [$\bullet$ [8] [$\bullet$ [6] [$\bullet$ [3] [4]]]]]]
\end{forest}
\end{subfigure}
\caption{Illustration for the proof of Proposition~\ref{prop:matching-small-move-big-nni}.}
\label{fig:matching-small-move-big-nni}
\end{figure}

\MatchingSmallMoveBigNni*
\begin{proof}

Recall the Diaconis--Holmes correspondence between matchings in $\c{X}_n$ and rooted binary phylogenetic trees
with $n+1$ leaves.

We use the following simple observation: a single NNI move changes the height of a rooted binary phylogenetic tree by at most $1$.
Consequently, for any two trees, NNI-distance is lower bounded by a difference of heights of corresponding trees.

We now construct, for every $n\ge 9$, two matchings that differ by a single transposition but whose trees
have heights differing by at least $\frac{n}{2}-2$.

Define the matching $x_1\in \c{X}_n$ by
\[
x_1
=
\bigl((3,4),(5,n+2),(6,n+3),\ldots,(n+1,2n-2),(2,2n-1),(1,2n)\bigr).
\]
Let $\sigma$ be the transposition swapping $1$ and $2n-1$.
Then $\sigma\triangleright x_1$ differs from $x_1$ only in the last two pairs, and we obtain
\[
x_2 := \sigma\triangleright x_1
=
\bigl((1,2),(3,4),(5,n+2),(6,n+3),\ldots,(n+1,2n-2),(2n-1,2n)\bigr).
\]

Let $T_1,T_2$ be the rooted phylogenetic trees corresponding to $x_1,x_2$ under the Diaconis--Holmes map.

\medskip\noindent
\textbf{Height of $T_1$.}
In the tree reconstruction procedure, the first eligible pair among labels $\{1,\dots,n+1\}$ is $(3,4)$, so the
first-created internal node $n+2$ has children $3$ and $4$.
Next, among labels $\{1,\dots,n+2\}$, the smallest eligible pair is $(5,n+2)$, so the next internal node
$n+3$ has children $5$ and $n+2$.
Continuing inductively, we create a caterpillar chain: each new internal node has one new leaf as one child
and the previously created internal node as the other child, until the pair $(2,2n-1)$ creates node $2n$.
Finally, the pair $(1,2n)$ attaches leaf $1$ as a sibling of node $2n$ under the root.
Therefore the longest root-to-leaf path has length $n$.

\medskip\noindent
\textbf{Height of $T_2$.}
Now the tree reconstruction starts with two pairs: $(1,2)$ creates $n+2$, and $(3,4)$ creates $n+3$.
After that, the procedure alternates between attaching a new leaf to the current top of the ``left'' chain
(rooted at $n+2$) and attaching a new leaf to the current top of the ``right'' chain (rooted at $n+3$):
\[
(5,n+2)\ \leadsto\ n+4,\qquad (6,n+3)\ \leadsto\ n+5,\qquad (7,n+4)\ \leadsto\ n+6,\qquad (8,n+5)\ \leadsto\ n+7,\ \ldots
\]
Thus we obtain two caterpillar subtrees whose heights differ by at most $1$.
The final pair $(2n-1,2n)$ makes the two top internal nodes siblings under the root.
A direct count shows that the taller of the two chains has at most $n+1$ nodes, and thus the height of the longer chain is at most $n/2+1$, and the overall height of $T_2$ is at most $n/2+2$.

\medskip
Combining the two height estimates gives
\[
\bigl|\mathrm{ht}(T_1)-\mathrm{ht}(T_2)\bigr|
\;\ge\;
n-\left(\frac{n}{2}+2\right)
=
\frac{n}{2}-2 .
\]

\end{proof}

\begin{figure}[t]
\hfill
\begin{subfigure}[b]{0.45\textwidth}
\centering
\begin{forest}
[$u$ [{$v$} [$A$] [$B$]] [$C$]]
\end{forest}
\caption{Subtree $T$: before NNI-move.}
\end{subfigure}
\hfill
\begin{subfigure}[b]{0.45\textwidth}
\centering
\begin{forest}
[$v$ [$A$] [$u$ [$B$] [$C$]]]
\end{forest}
\caption{Subtree $\widetilde{T}$: after NNI-move.}
\end{subfigure}
\hfill
\caption{Illustration for the proof of \Cref{EmbeddingSingleNniSmallMatchingDistance}.
}
\label{fig:embedding-single-nni-small-matching-illustration}
\end{figure}
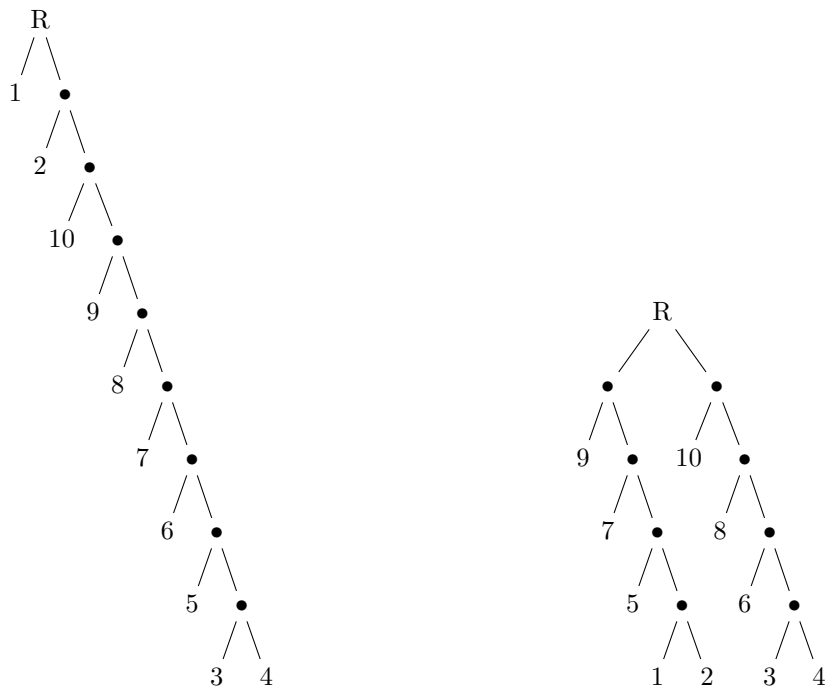

\EmbeddingSingleNniSmallMatchingDistance*
\begin{proof}

By analyzing the embedding proposed by \cite{richman2025}, we establish the following two properties:
\begin{enumerate}
\item \textbf{Locality of internal labels.}
The label of an internal node depends only on the labels of leaves in its subtree.
More precisely, the label equals $-k$, where $k$ is the largest leaf label $j$ in the subtree such that no
internal node in the same subtree has label $-j$.

\item \textbf{Labels present in a subtree.}
Let a subtree have leaf labels $0 \le a_1 < \dots < a_k$. Then the set of labels appearing on \emph{all}
nodes of this subtree is $\{-a_2,\dots,-a_k\}\ \cup\ \{a_1,\dots,a_k\}$.
\end{enumerate}
In particular, unlike the Diaconis--Holmes labeling, the label of a fixed internal node does not depend on
the global topology of the tree.

Now consider an NNI move performed on an internal edge $(u,v)$ inside a rooted subtree $T$ (see
\Cref{fig:embedding-single-nni-small-matching-illustration}).
Assume that $u$ has children $v$ and a subtree $C$, while $v$ has children subtrees $A$ and $B$.
After the NNI move we obtain a subtree $\widetilde T$ in which $v$ has children $A$ and $u$, and $u$ has
children $B$ and $C$.

Let $\mathrm{label}_T(\cdot)$ and $\mathrm{label}_{\widetilde T}(\cdot)$ denote the node labels before and after the move.
By the locality property above, the labelings \emph{within} each of the subtrees $A$, $B$, and $C$ are unchanged,
and therefore all matching pairs internal to $A$, internal to $B$, and internal to $C$ are unchanged as well.
Hence, the only pairs within subtree $T$ that may change are those created at the two internal parents $v$ and $u$.

Denote by $\mathrm{root}(\cdot)$ the root node of a subtree. In $T$ the two relevant pairs are
\[
\bigl(\mathrm{label}_T(\mathrm{root}(A)),\,\mathrm{label}_T(\mathrm{root}(B))\bigr)
\qquad\text{and}\qquad
\bigl(\mathrm{label}_T(v),\,\mathrm{label}_T(\mathrm{root}(C))\bigr),
\]
while in $\widetilde T$ they become
\[
\bigl(\mathrm{label}_{\widetilde T}(\mathrm{root}(A)),\,\mathrm{label}_{\widetilde T}(u)\bigr)
\qquad\text{and}\qquad
\bigl(\mathrm{label}_{\widetilde T}(\mathrm{root}(B)),\,\mathrm{label}_{\widetilde T}(\mathrm{root}(C))\bigr).
\]
Since $\mathrm{root}(A),\mathrm{root}(B),\mathrm{root}(C)$ belong to unchanged subtrees, we may recognize
$\mathrm{label}_T(\mathrm{root}(A))=\mathrm{label}_{\widetilde T}(\mathrm{root}(A))$ and similarly for $B,C$. Thus the change is
\begin{align*}
\bigl(\mathrm{label}_T(\mathrm{root}(A)),\,\mathrm{label}_T(\mathrm{root}(B))\bigr)
&\Rightarrow
\bigl(\mathrm{label}_T(\mathrm{root}(A)),\,\mathrm{label}_{\widetilde T}(u)\bigr),\\
\bigl(\mathrm{label}_T(v),\,\mathrm{label}_T(\mathrm{root}(C))\bigr)
&\Rightarrow
\bigl(\mathrm{label}_T(\mathrm{root}(B)),\,\mathrm{label}_T(\mathrm{root}(C))\bigr).
\end{align*}

We claim that these transformations can be realized by at most two transpositions of labels.
Define
\[
\sigma_1 \;=\; \bigl(\mathrm{label}_T(\mathrm{root}(B)),\,\mathrm{label}_{\widetilde T}(u)\bigr).
\]
Applying $\sigma_1$ swaps the two labels $\mathrm{label}_T(\mathrm{root}(B))$ and $\mathrm{label}_{\widetilde T}(u)$ and therefore
turns the first pair into the desired one. If additionally
$\mathrm{label}_T(v)\neq \mathrm{label}_{\widetilde T}(u)$, we also apply
\[
\sigma_2 \;=\; \bigl(\mathrm{label}_T(v),\,\mathrm{label}_T(\mathrm{root}(B))\bigr),
\]
which corrects the second pair as well. Thus, inside the affected part of the matching, at most two
transpositions suffice.

It remains to check that no additional transpositions are needed to fix the pair involving the root of $T$.
By the subtree label-set property, the set of labels appearing in the subtree rooted at $u$ is unchanged by
NNI. In particular,
\[
\{\mathrm{label}_T(u),\mathrm{label}_T(v)\} \;=\; \{\mathrm{label}_{\widetilde T}(u),\mathrm{label}_{\widetilde T}(v)\}.
\]
Hence there are only two cases:
\begin{enumerate}
\item If $\mathrm{label}_T(v)=\mathrm{label}_{\widetilde T}(u)$, then necessarily $\mathrm{label}_T(u)=\mathrm{label}_{\widetilde T}(v)$.
In this case the root label of the considered subtree is unchanged, and applying $\sigma_1$ does not affect
the pair containing $\mathrm{label}_T(u)$.

\item If $\mathrm{label}_T(v)\neq \mathrm{label}_{\widetilde T}(u)$, then $\mathrm{label}_T(u)=\mathrm{label}_{\widetilde T}(u)$ and
$\mathrm{label}_T(v)=\mathrm{label}_{\widetilde T}(v)$.
In this case, after applying $\sigma_1$ and $\sigma_2$, the pair incident to $\mathrm{label}_T(u)$ is transformed
into the pair incident to $\mathrm{label}_{\widetilde T}(u)$, as required.
\end{enumerate}
Therefore the pair involving the root of the modified subtree remains consistent after the same
(at most two) transpositions.

Finally, since the set of labels appearing in the modified subtree is unchanged, all labels of internal nodes
\emph{outside} this subtree (which depend only on the label sets of their descendant subtrees) remain unchanged.
Thus no other matching pairs are affected.

We conclude that the matchings associated with $T$ and $\widetilde T$ differ by at most two transpositions,
so their distance in the quotient Cayley graph is at most $2$.

\end{proof}

\newpage

\begin{figure}[t]
\centering
\begin{subfigure}[b]{0.20\textwidth}
\includegraphics[width=\textwidth, trim=55pt 20pt 0pt -10pt, clip]{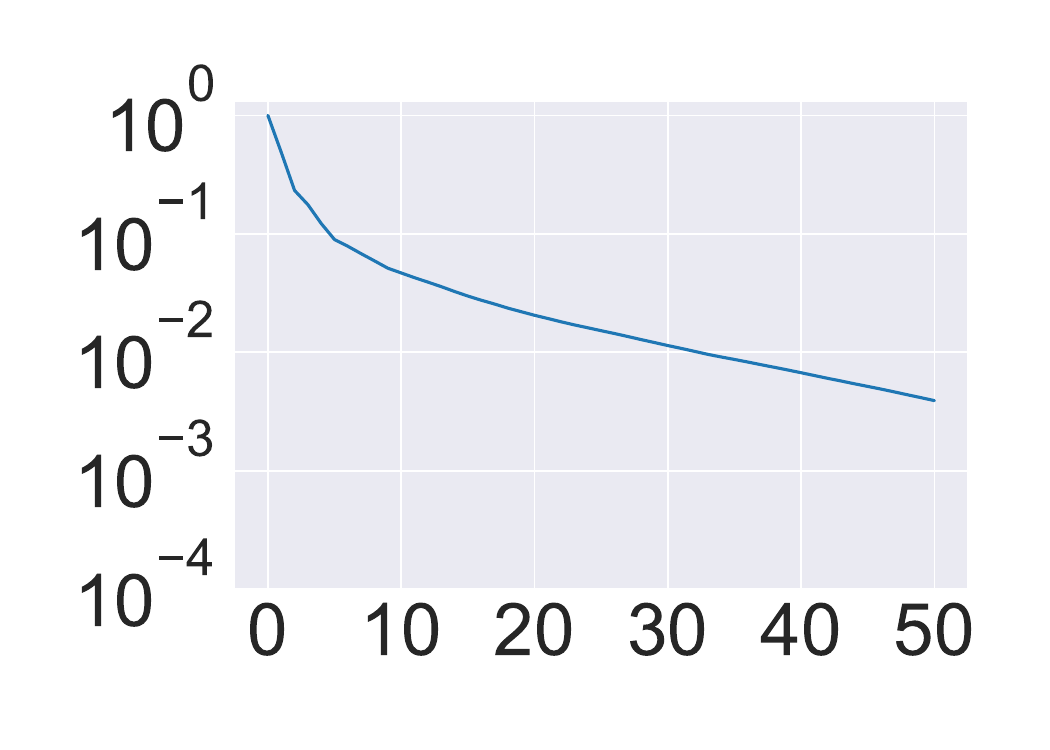}
\caption{$n=15,\ \nu=0.5$}
\end{subfigure}
\hfill
\begin{subfigure}[b]{0.20\textwidth}
\includegraphics[width=\textwidth, trim=55pt 20pt 0pt -10pt, clip]{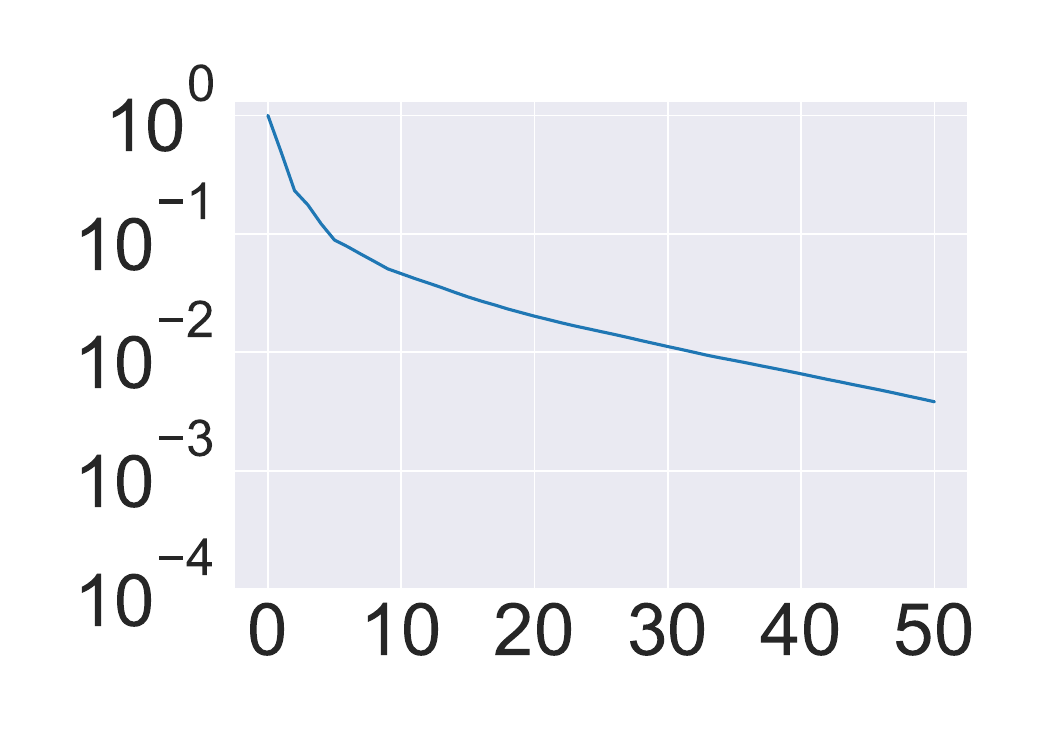}
\caption{$n=15,\ \nu=1.5$}
\end{subfigure}
\hfill
\begin{subfigure}[b]{0.20\textwidth}
\includegraphics[width=\textwidth, trim=55pt 20pt 0pt -10pt, clip]{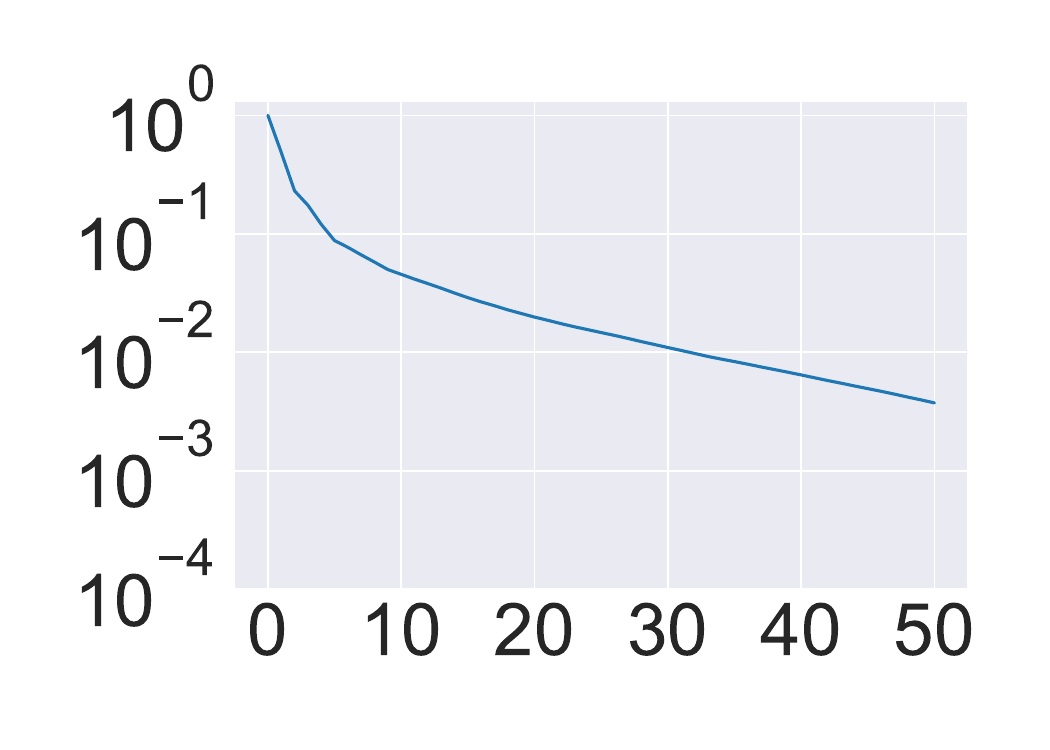}
\caption{$n=15,\ \nu=2.5$}
\end{subfigure}
\hfill
\begin{subfigure}[b]{0.20\textwidth}
\includegraphics[width=\textwidth, trim=55pt 20pt 0pt -10pt, clip]{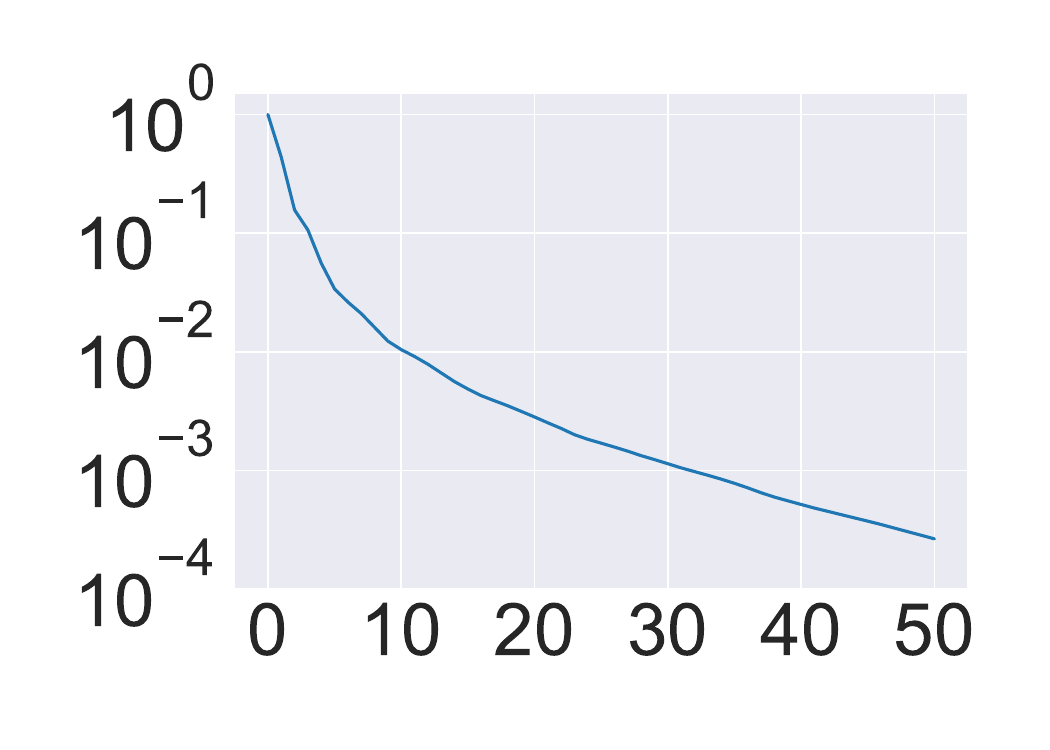}
\caption{$n=15,\ \nu=\infty$}
\end{subfigure}

\begin{subfigure}[b]{0.20\textwidth}
\includegraphics[width=\textwidth, trim=55pt 20pt 0pt -10pt, clip]{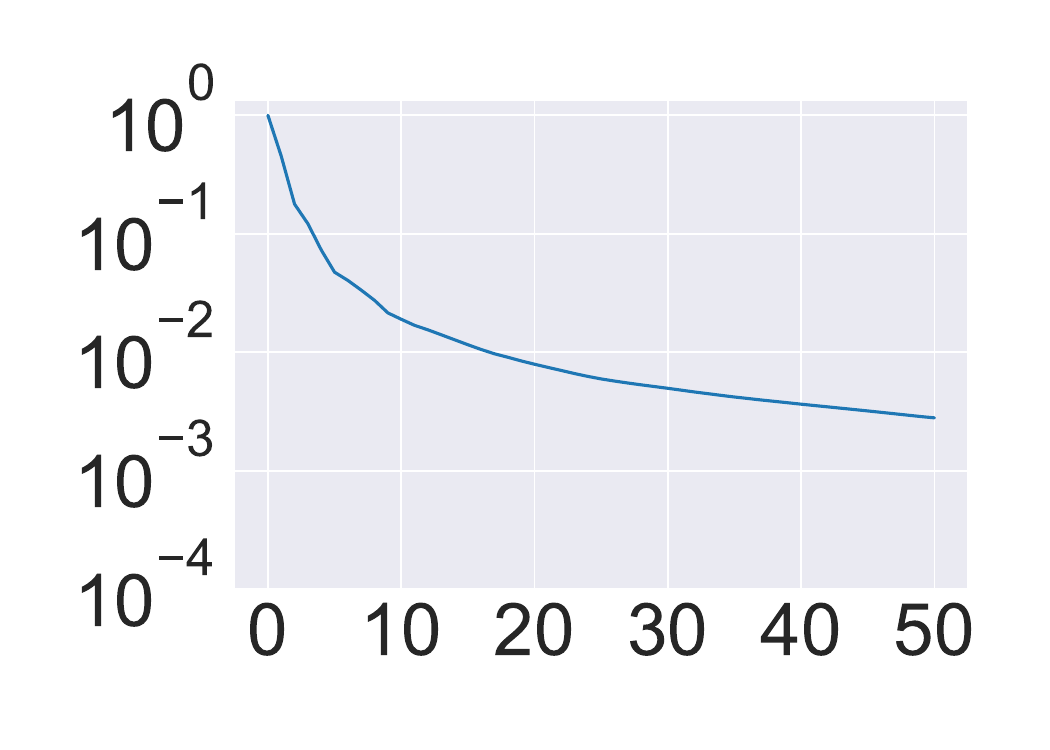}
\caption{$n=20,\ \nu=0.5$}
\end{subfigure}
\hfill
\begin{subfigure}[b]{0.20\textwidth}
\includegraphics[width=\textwidth, trim=55pt 20pt 0pt -10pt, clip]{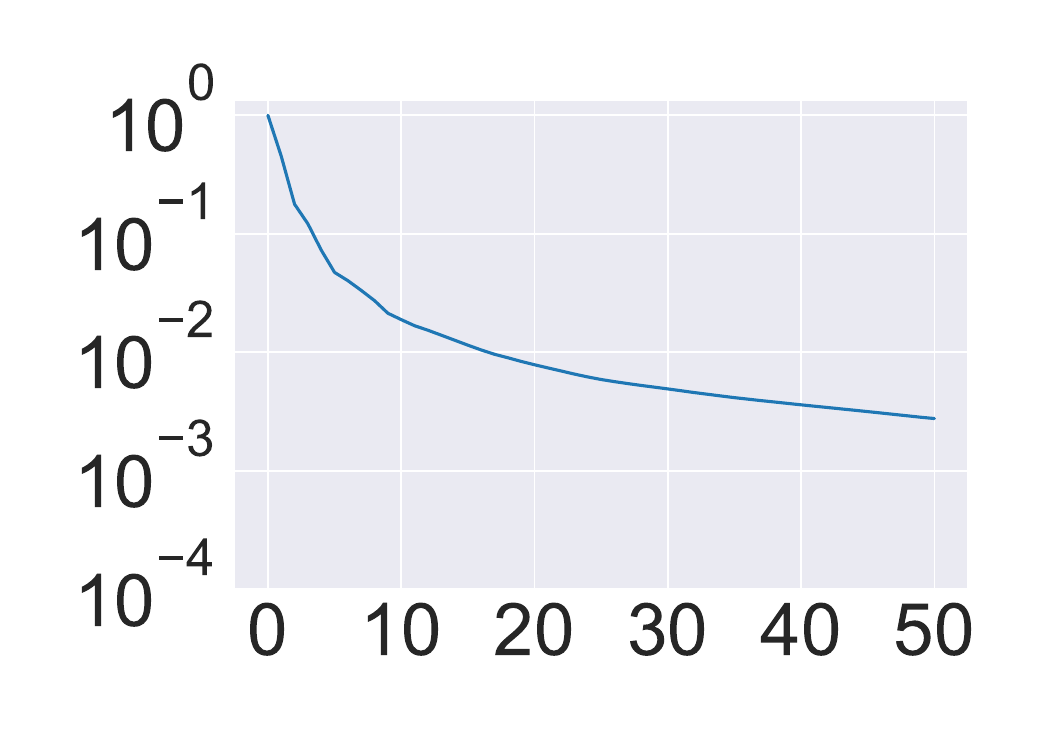}
\caption{$n=20,\ \nu=1.5$}
\end{subfigure}
\hfill
\begin{subfigure}[b]{0.20\textwidth}
\includegraphics[width=\textwidth, trim=55pt 20pt 0pt -10pt, clip]{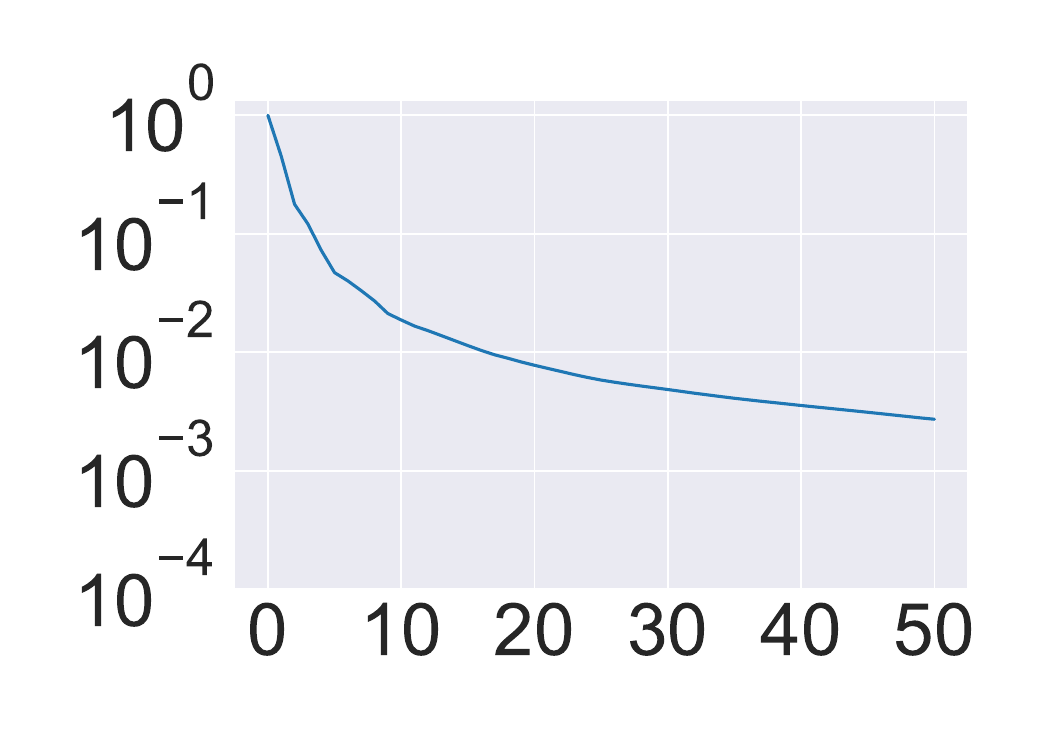}
\caption{$n=20,\ \nu=2.5$}
\end{subfigure}
\hfill
\begin{subfigure}[b]{0.20\textwidth}
\includegraphics[width=\textwidth, trim=55pt 20pt 0pt -10pt, clip]{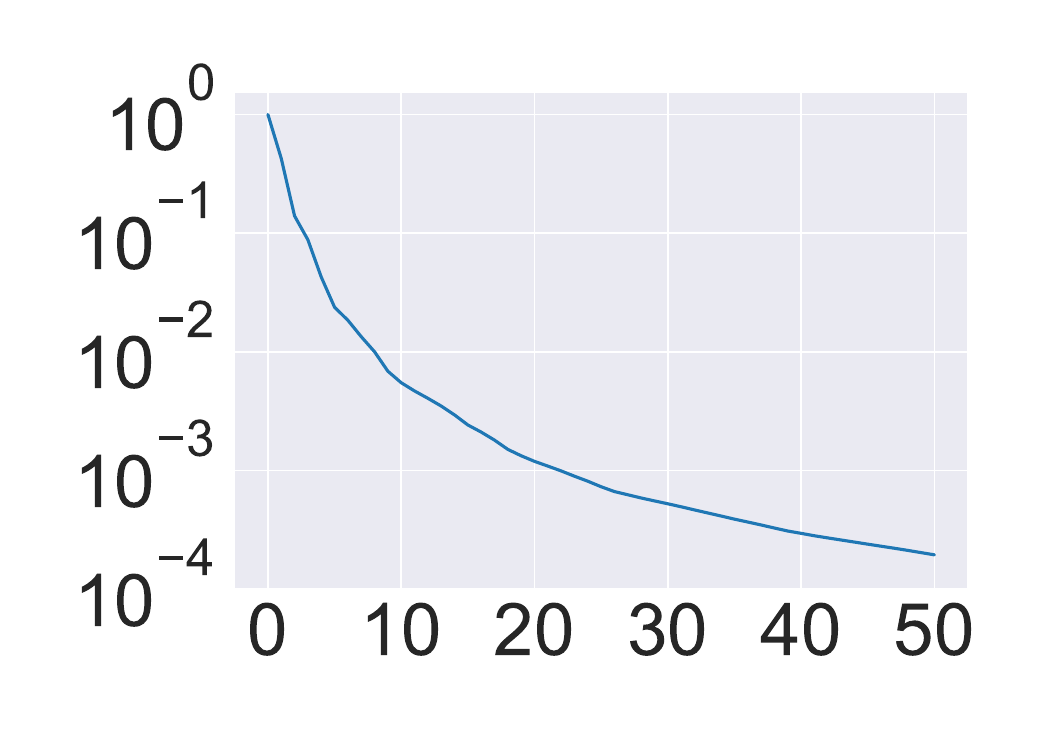}
\caption{$n=20,\ \nu=\infty$}
\end{subfigure}

\begin{subfigure}[b]{0.20\textwidth}
\includegraphics[width=\textwidth, trim=55pt 20pt 0pt -10pt, clip]{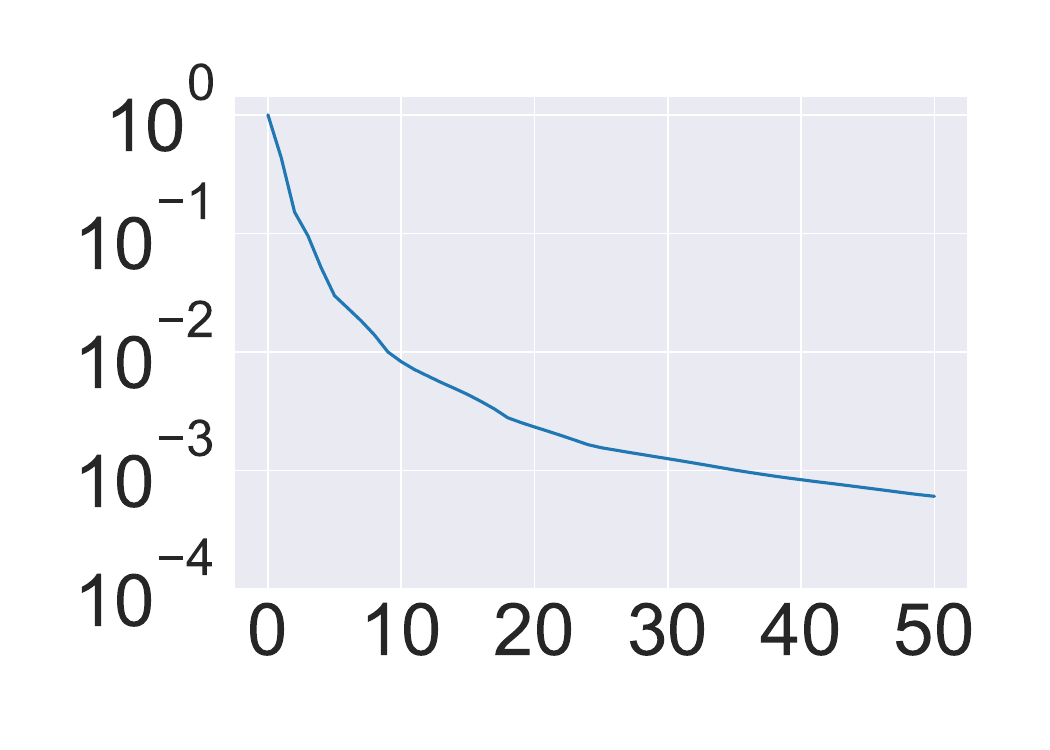}
\caption{$n=25,\ \nu=0.5$}
\end{subfigure}
\hfill
\begin{subfigure}[b]{0.20\textwidth}
\includegraphics[width=\textwidth, trim=55pt 20pt 0pt -10pt, clip]{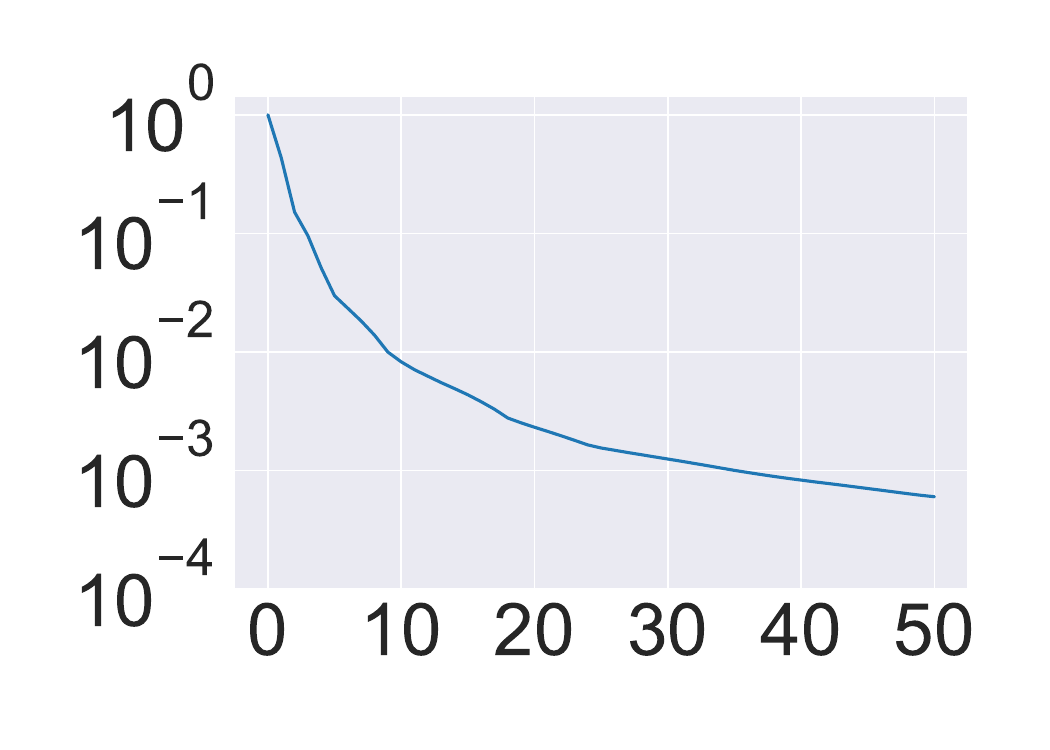}
\caption{$n=25,\ \nu=1.5$}
\end{subfigure}
\hfill
\begin{subfigure}[b]{0.20\textwidth}
\includegraphics[width=\textwidth, trim=55pt 20pt 0pt -10pt, clip]{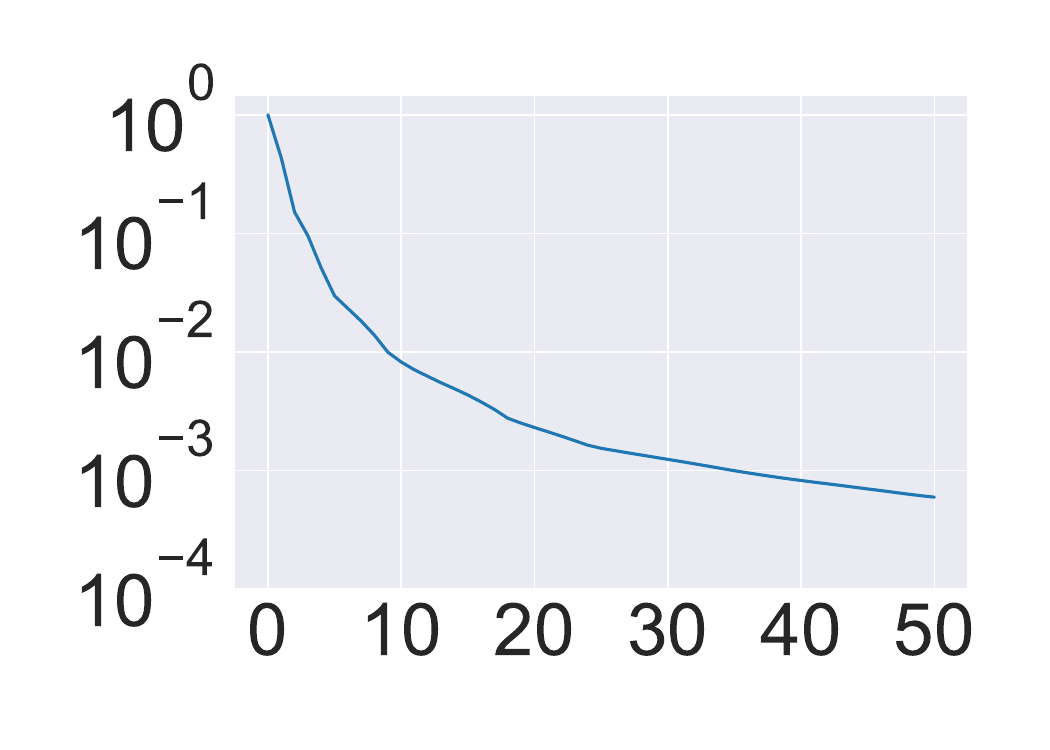}
\caption{$n=25,\ \nu=2.5$}
\end{subfigure}
\hfill
\begin{subfigure}[b]{0.20\textwidth}
\includegraphics[width=\textwidth, trim=55pt 20pt 0pt -10pt, clip]{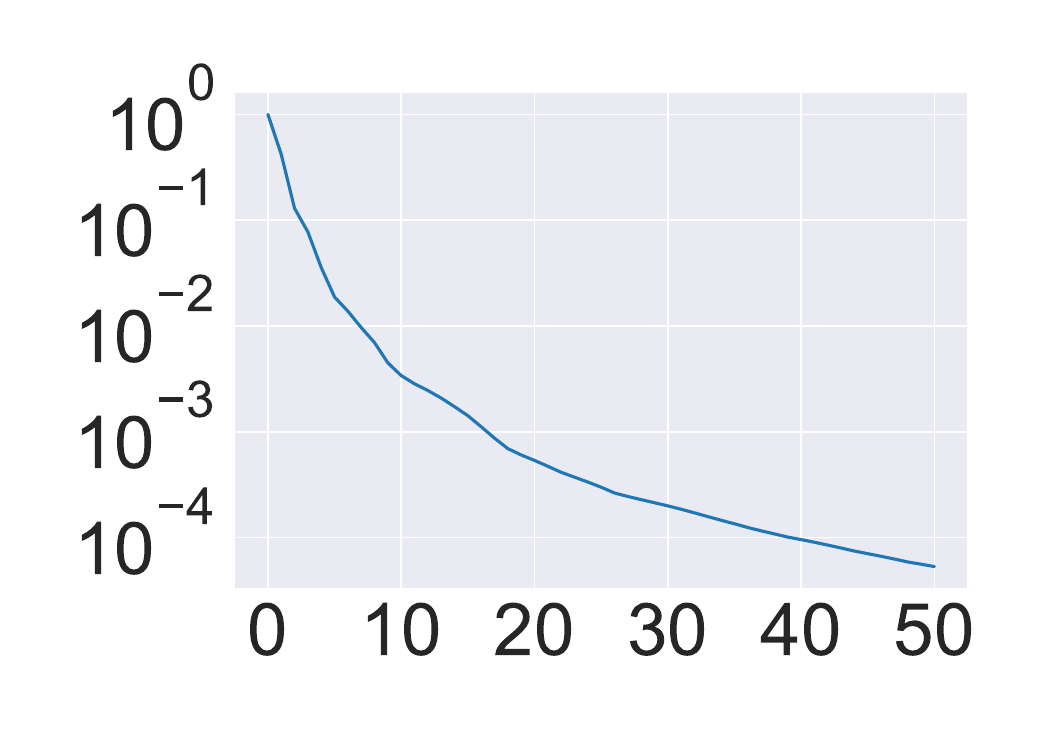}
\caption{$n=25,\ \nu=\infty$}
\end{subfigure}

\begin{subfigure}[b]{0.20\textwidth}
\includegraphics[width=\textwidth, trim=55pt 20pt 0pt -10pt, clip]{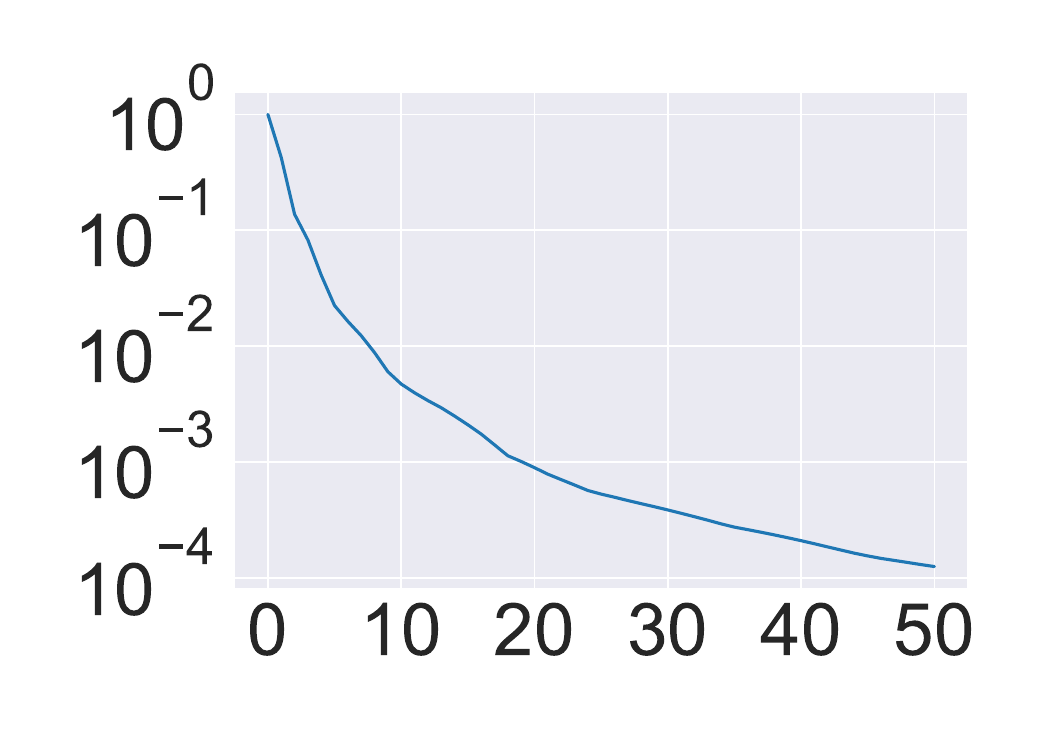}
\caption{$n=30,\ \nu=0.5$}
\end{subfigure}
\hfill
\begin{subfigure}[b]{0.20\textwidth}
\includegraphics[width=\textwidth, trim=55pt 20pt 0pt -10pt, clip]{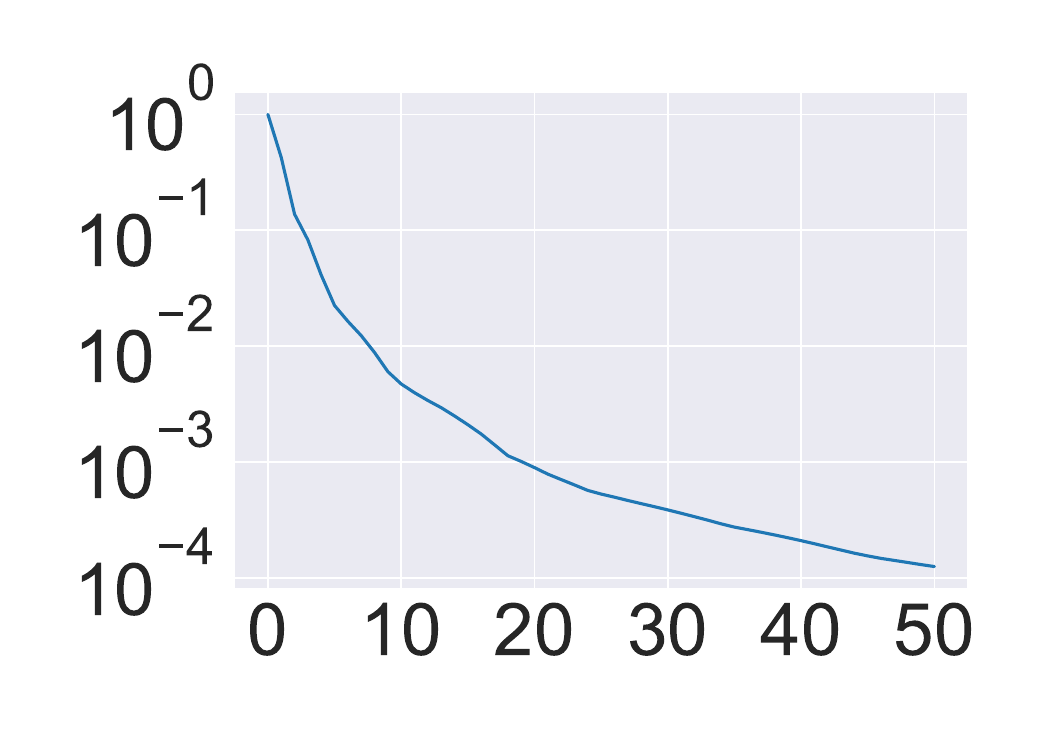}
\caption{$n=30,\ \nu=1.5$}
\end{subfigure}
\hfill
\begin{subfigure}[b]{0.20\textwidth}
\includegraphics[width=\textwidth, trim=55pt 20pt 0pt -10pt, clip]{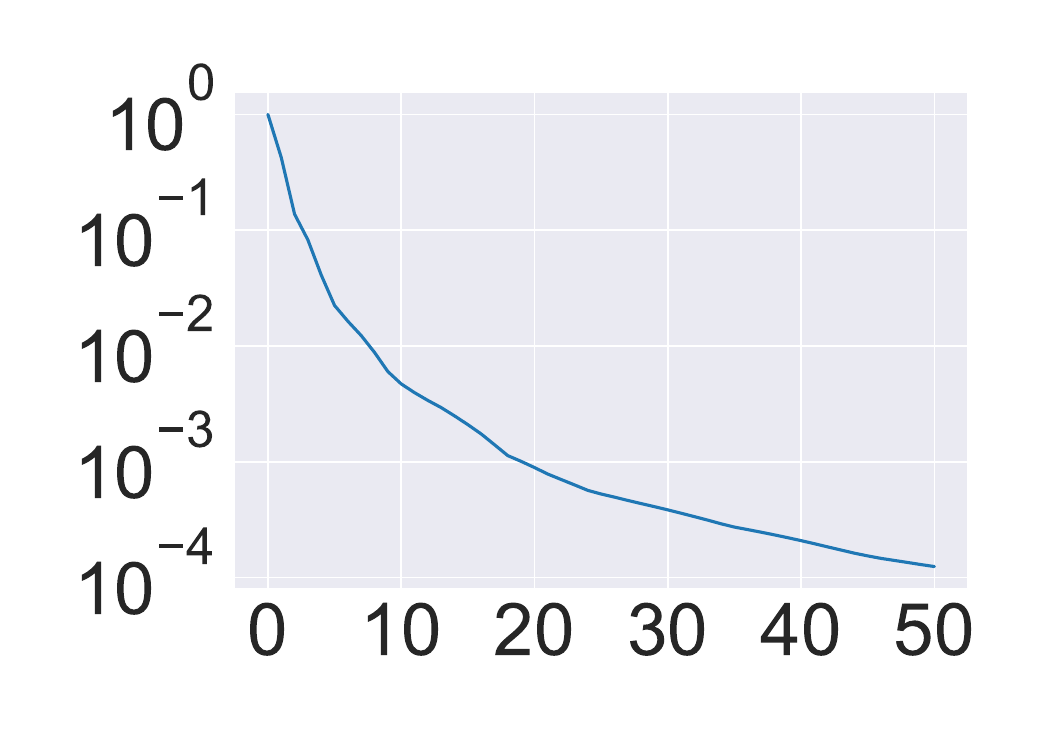}
\caption{$n=30,\ \nu=2.5$}
\end{subfigure}
\hfill
\begin{subfigure}[b]{0.20\textwidth}
\includegraphics[width=\textwidth, trim=55pt 20pt 0pt -10pt, clip]{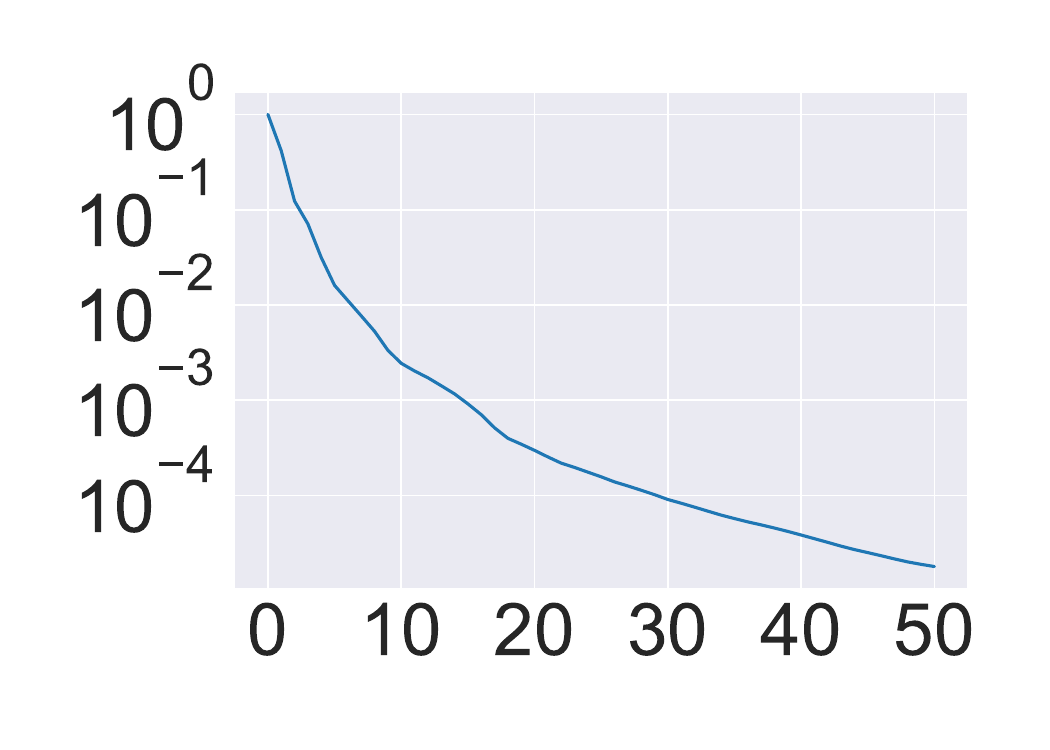}
\caption{$n=30,\ \nu=\infty$}
\end{subfigure}

\begin{subfigure}[b]{0.20\textwidth}
\includegraphics[width=\textwidth, trim=55pt 20pt 0pt -10pt, clip]{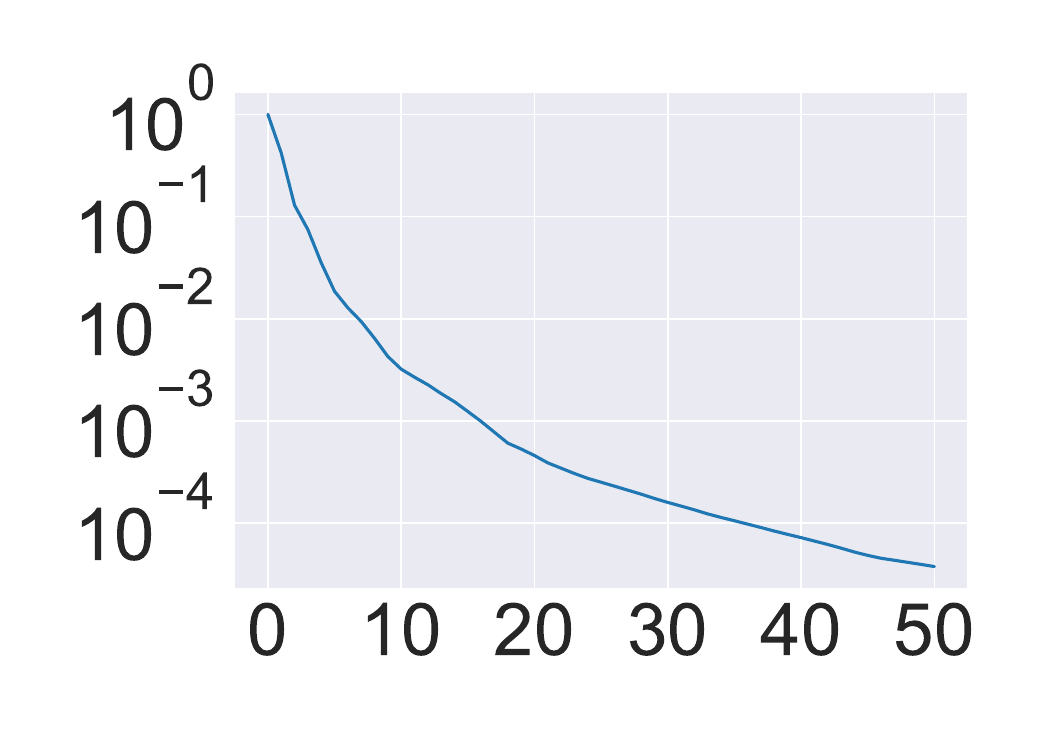}
\caption{$n=35,\ \nu=0.5$}
\end{subfigure}
\hfill
\begin{subfigure}[b]{0.20\textwidth}
\includegraphics[width=\textwidth, trim=55pt 20pt 0pt -10pt, clip]{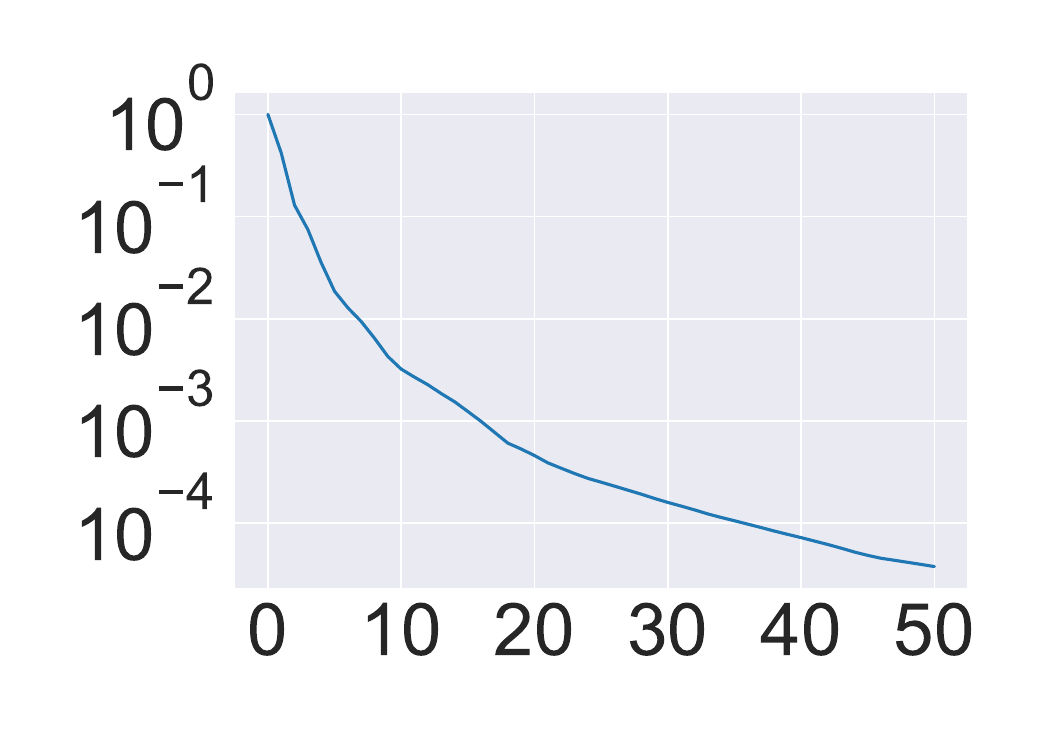}
\caption{$n=35,\ \nu=1.5$}
\end{subfigure}
\hfill
\begin{subfigure}[b]{0.20\textwidth}
\includegraphics[width=\textwidth, trim=55pt 20pt 0pt -10pt, clip]{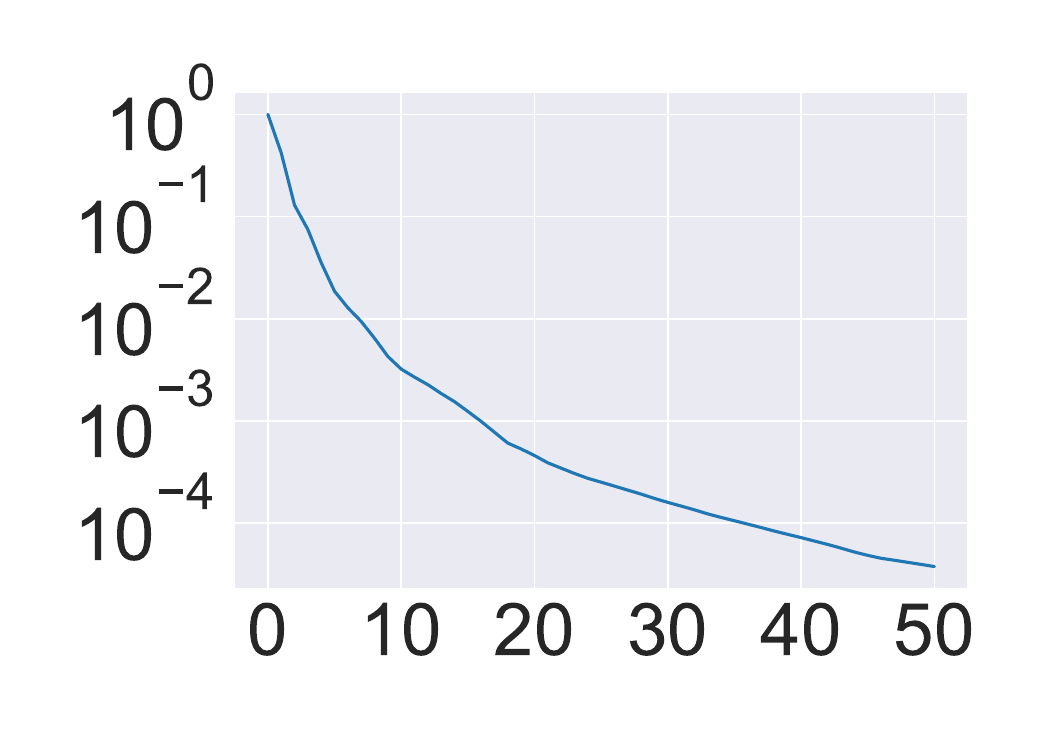}
\caption{$n=35,\ \nu=2.5$}
\end{subfigure}
\hfill
\begin{subfigure}[b]{0.20\textwidth}
\includegraphics[width=\textwidth, trim=55pt 20pt 0pt -10pt, clip]{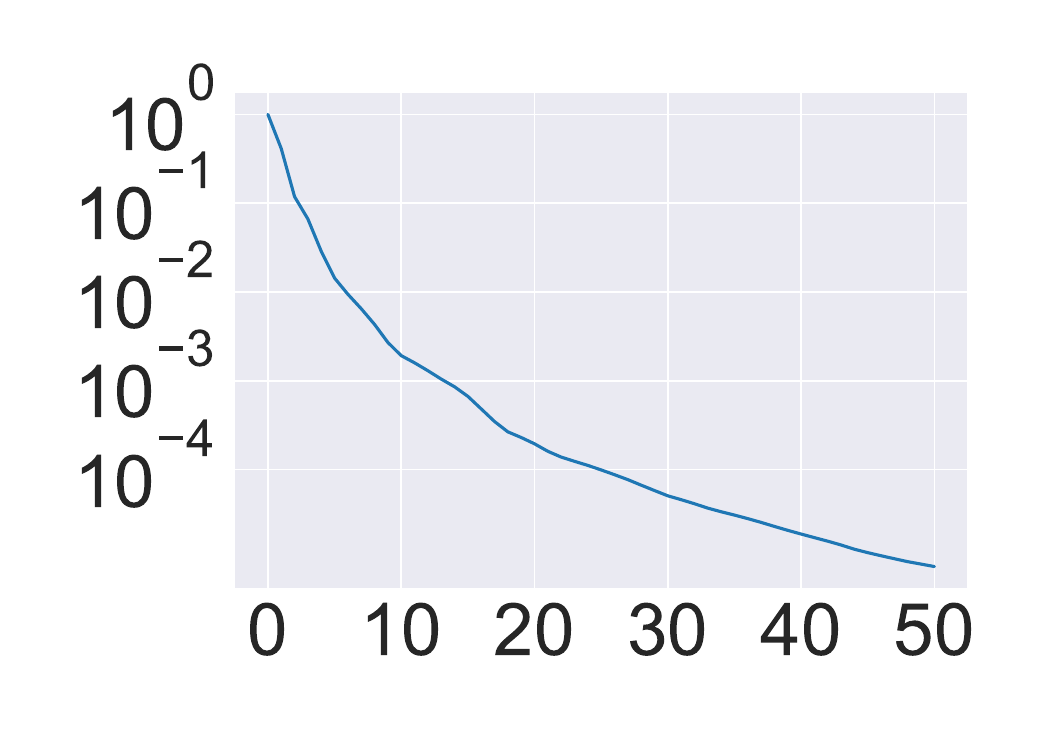}
\caption{$n=35,\ \nu=\infty$}
\end{subfigure}

\begin{subfigure}[b]{0.20\textwidth}
\includegraphics[width=\textwidth, trim=55pt 20pt 0pt -10pt, clip]{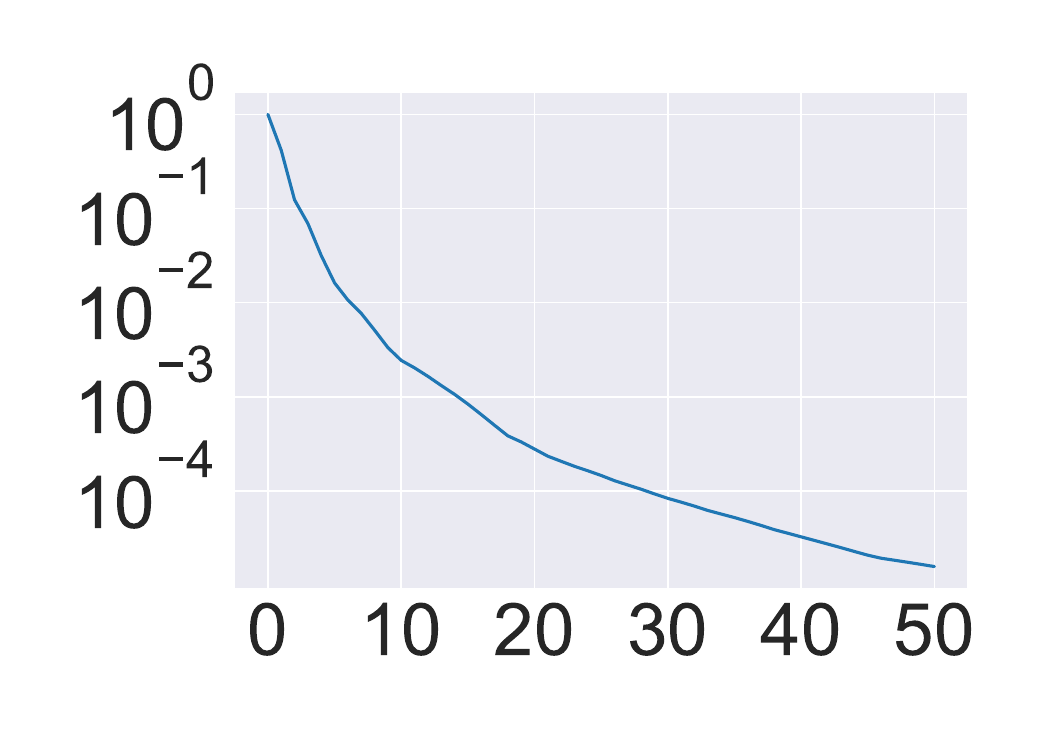}
\caption{$n=40,\ \nu=0.5$}
\end{subfigure}
\hfill
\begin{subfigure}[b]{0.20\textwidth}
\includegraphics[width=\textwidth, trim=55pt 20pt 0pt -10pt, clip]{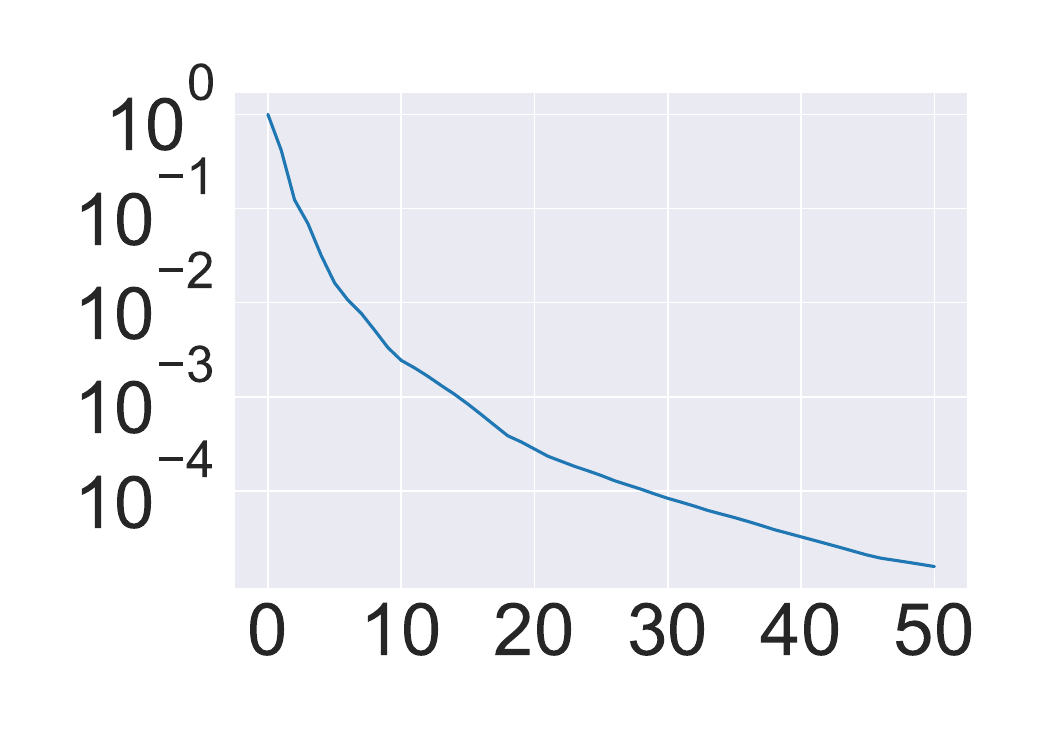}
\caption{$n=40,\ \nu=1.5$}
\end{subfigure}
\hfill
\begin{subfigure}[b]{0.20\textwidth}
\includegraphics[width=\textwidth, trim=55pt 20pt 0pt -10pt, clip]{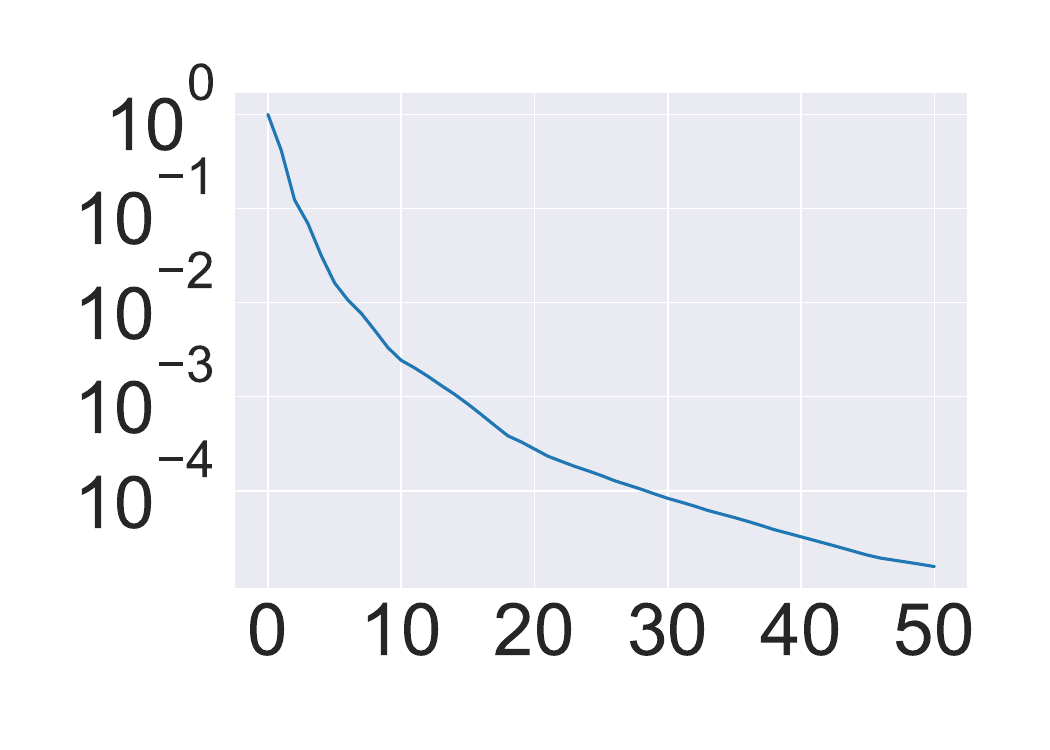}
\caption{$n=40,\ \nu=2.5$}
\end{subfigure}
\hfill
\begin{subfigure}[b]{0.20\textwidth}
\includegraphics[width=\textwidth, trim=55pt 20pt 0pt -10pt, clip]{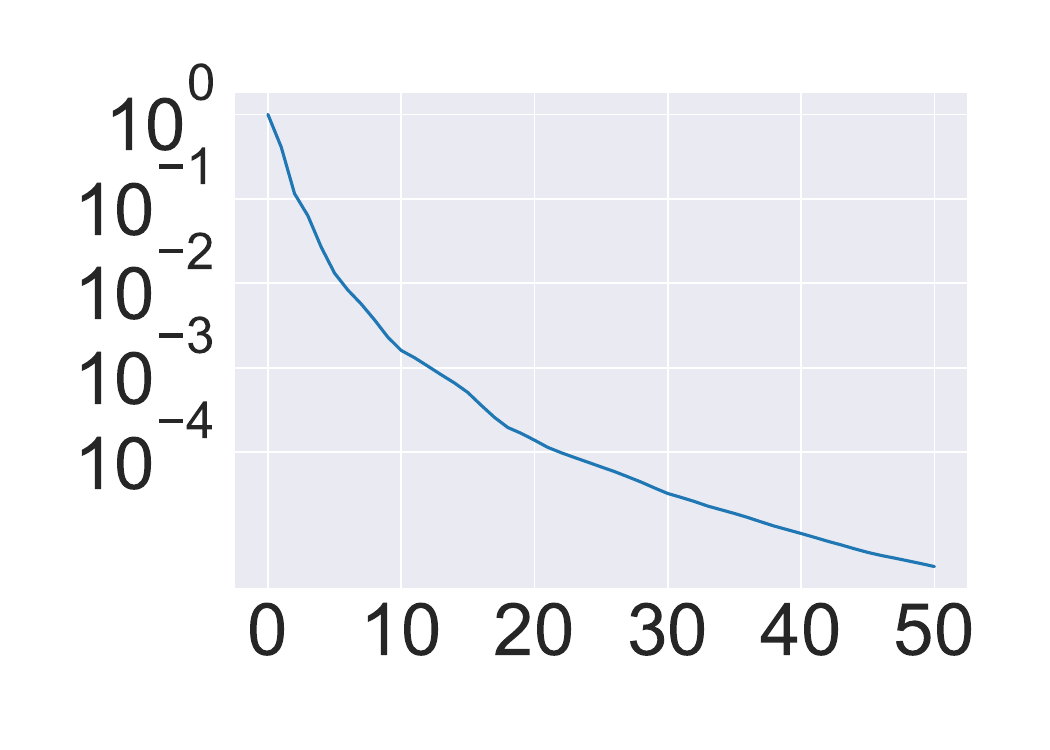}
\caption{$n=40,\ \nu=\infty$}
\end{subfigure}
\caption{
Additional results on approximation quality of truncated kernels as a function of $|\mathcal{R}|$.
The $x$-axis represents the number of partitions $|\mathcal{R}|$ used in the truncation; the $y$-axis displays the relative $L^2$ approximation error $\norm{k - k_\mathcal{R}}_{L^2(\c{X} \times \c{X})} / \norm{k}_{L^2(\c{X} \times \c{X})}$ (log scale).
For each $|\mathcal{R}|$, the set $\mathcal{R}$ is selected by the heuristic maximizing~$\rho_1$.
}
\label{fig:extended-approximation-error}
\end{figure}

\begin{figure}[t]

\begin{subfigure}[b]{0.20\textwidth}
\includegraphics[width=1.0\textwidth, trim=55pt 0pt 0pt 0pt]{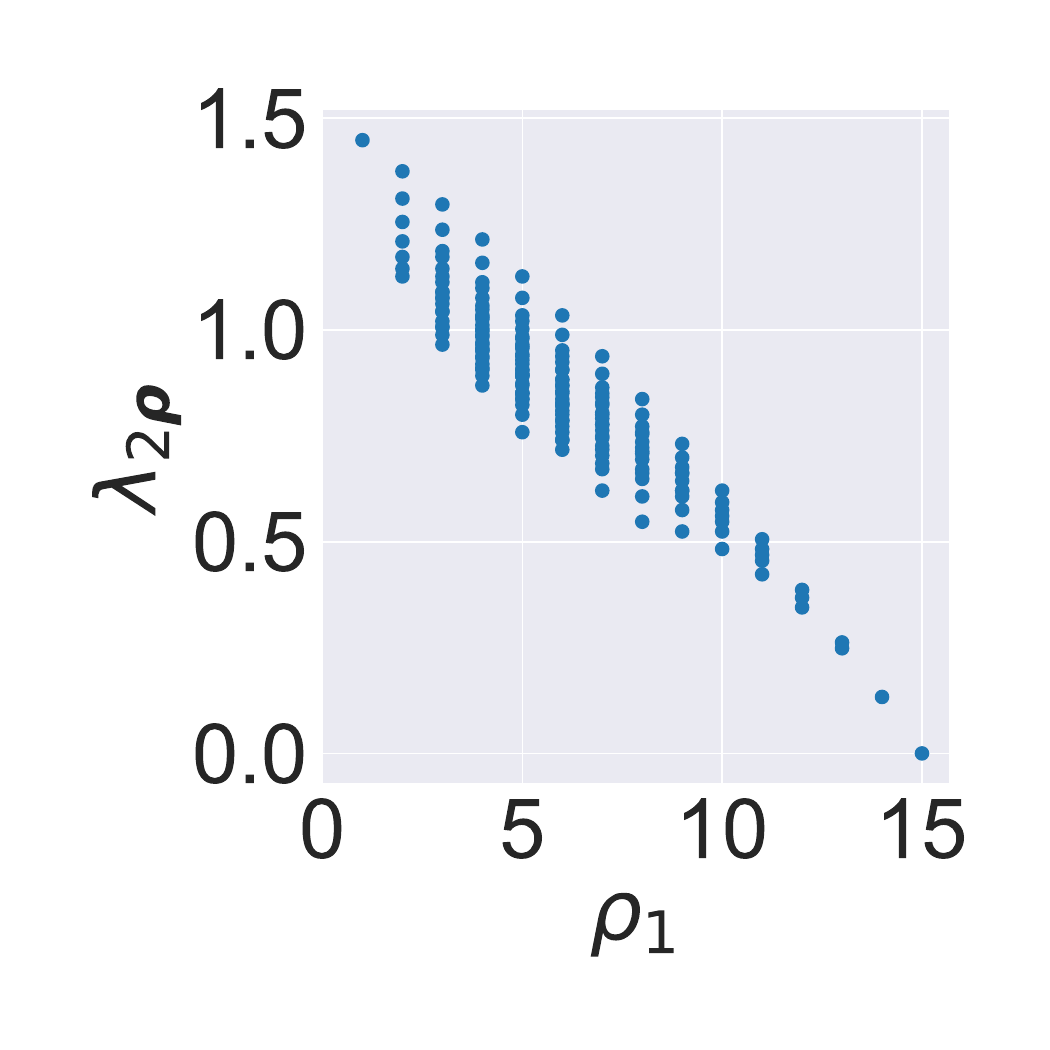}
\caption{$n=15$, max}
\end{subfigure}
\hfill
\begin{subfigure}[b]{0.20\textwidth}
\includegraphics[width=1.0\textwidth, trim=55pt 0pt 0pt 0pt]{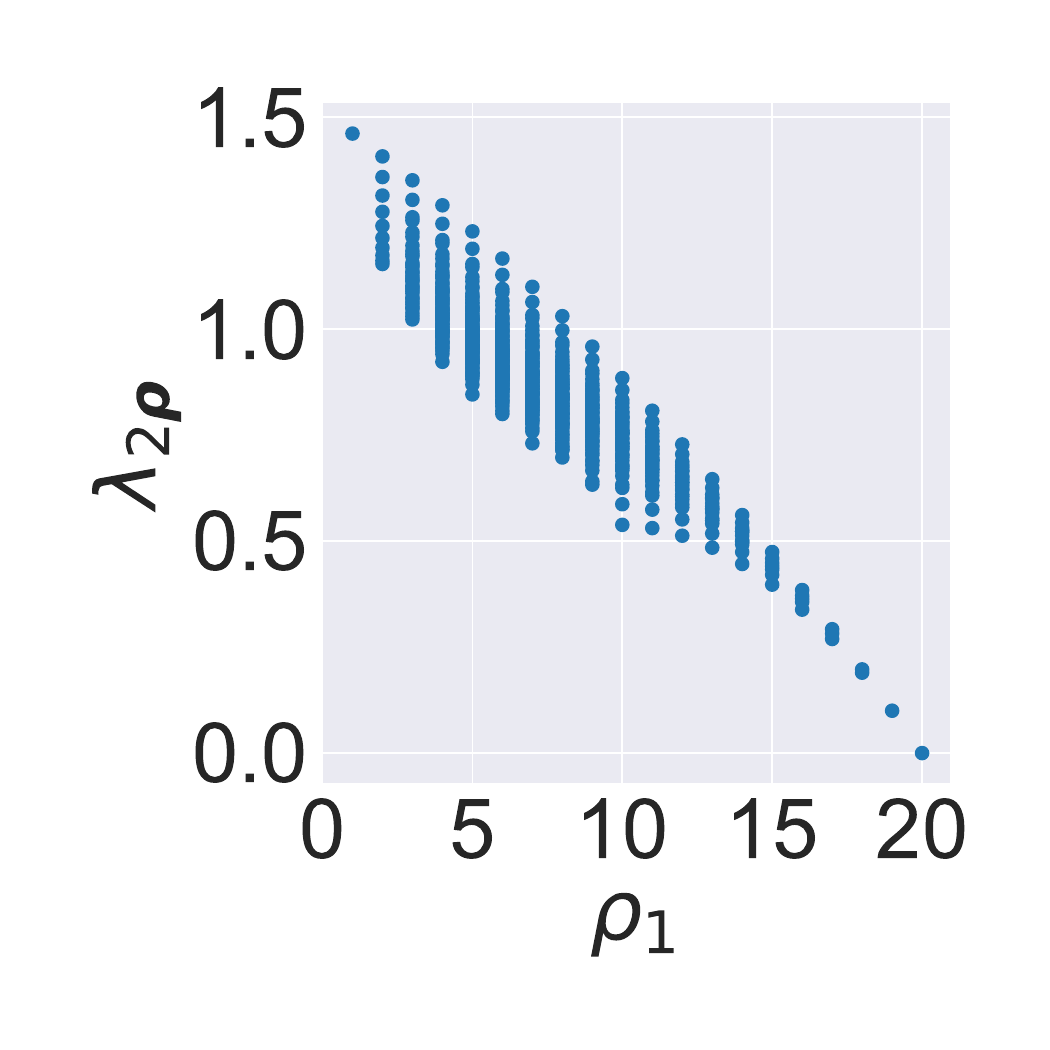}
\caption{$n=20$, max}
\end{subfigure}
\hfill
\begin{subfigure}[b]{0.20\textwidth}
\includegraphics[width=1.0\textwidth, trim=55pt 0pt 0pt 0pt]{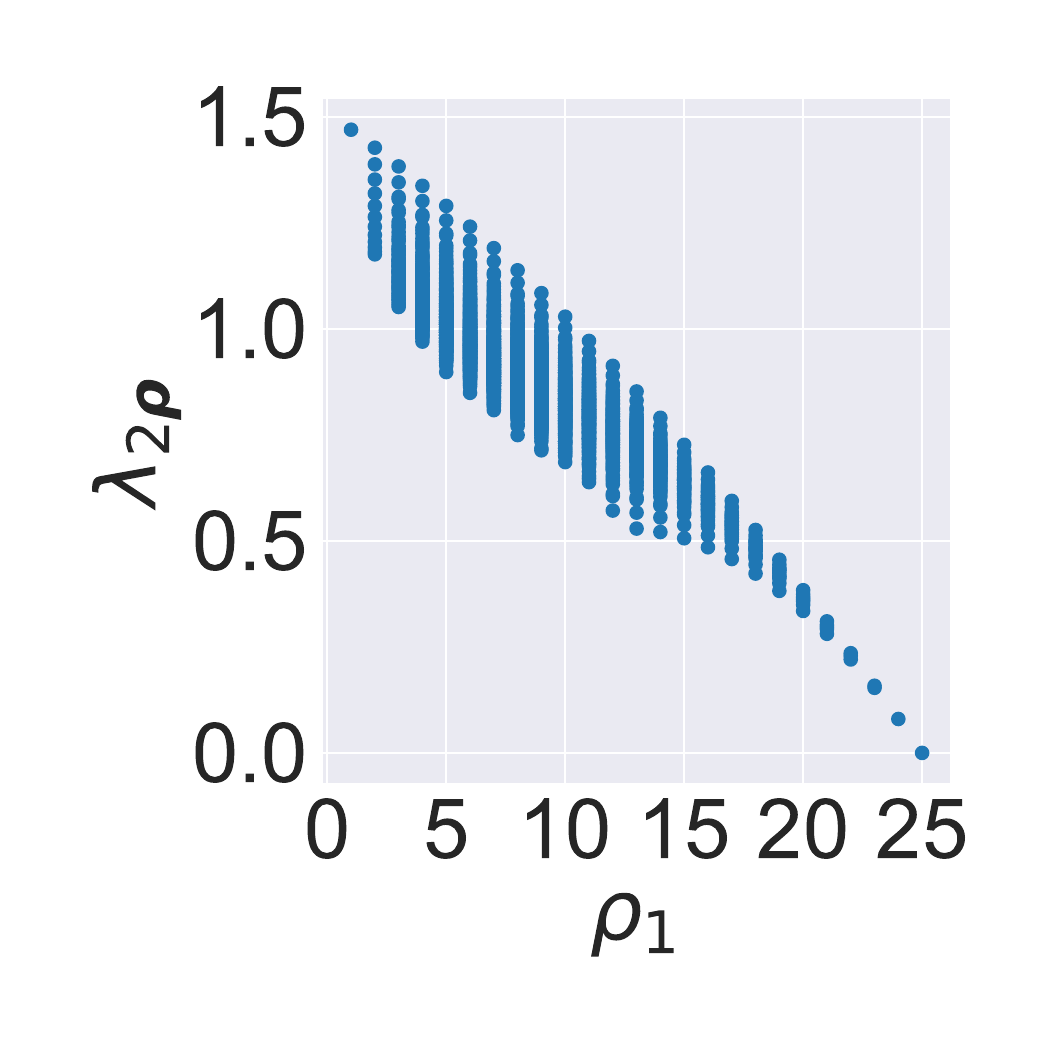}
\caption{$n=25$, max}
\end{subfigure}
\hfill
\begin{subfigure}[b]{0.20\textwidth}
\includegraphics[width=1.0\textwidth, trim=55pt 0pt 0pt 0pt]{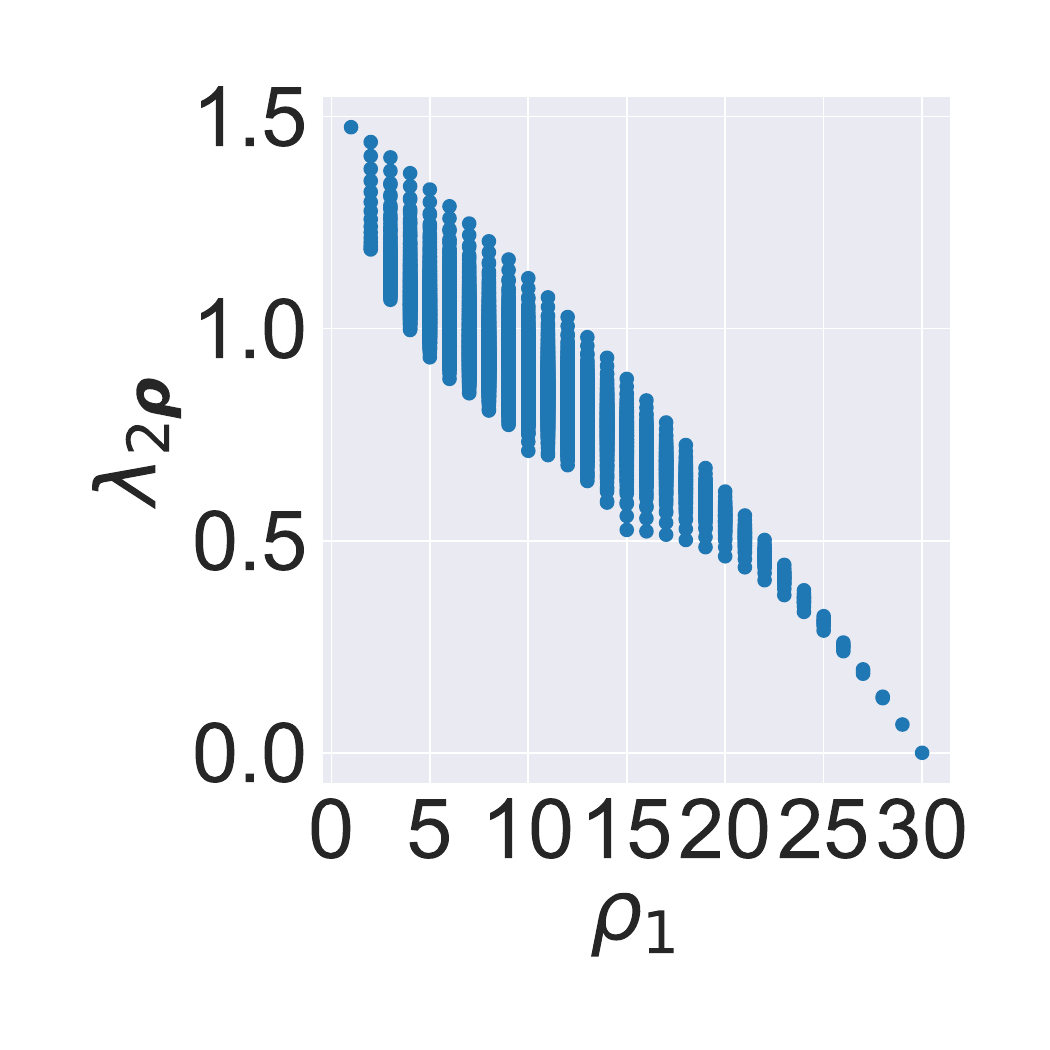}
\caption{$n=30$, max}
\end{subfigure}

\begin{subfigure}[b]{0.20\textwidth}
\includegraphics[width=1.0\textwidth, trim=55pt 0pt 0pt 0pt]{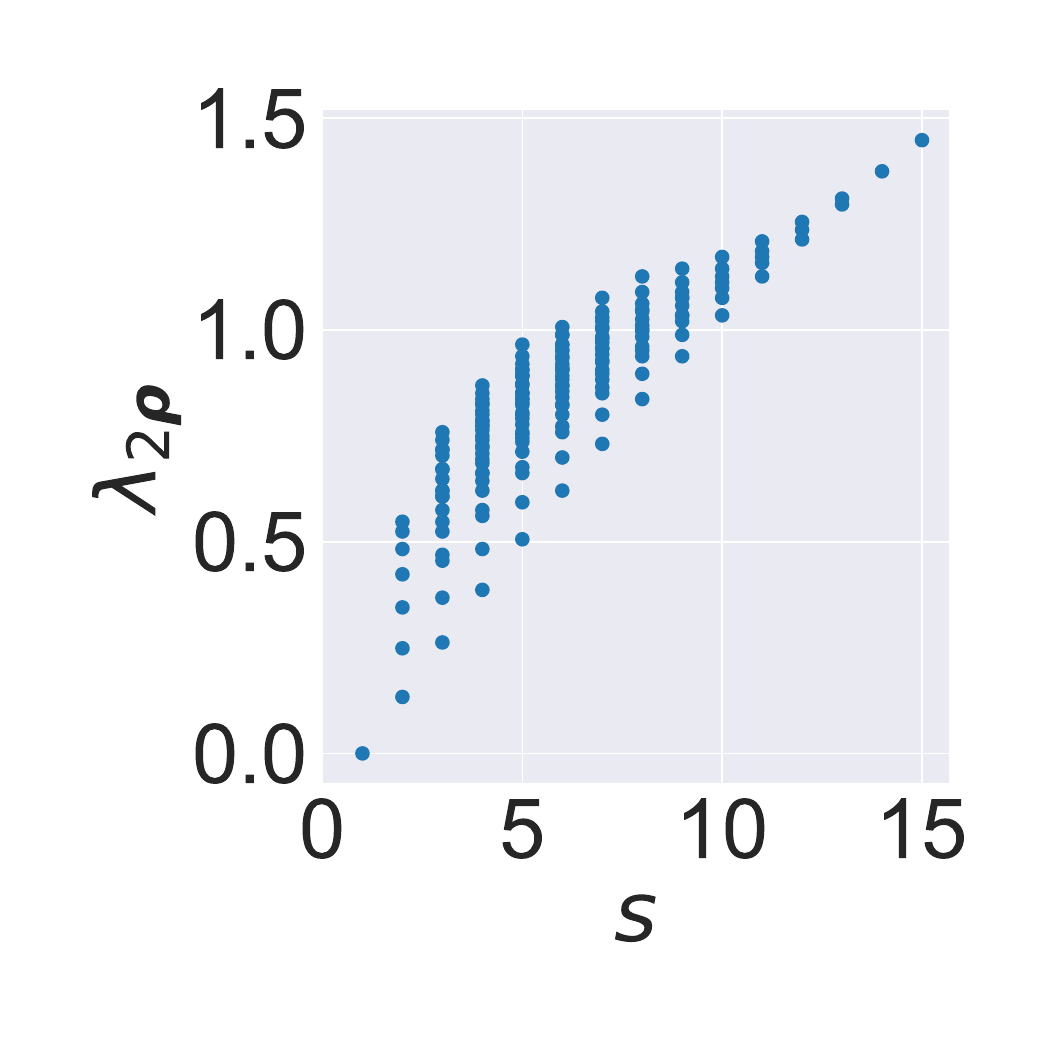}
\caption{$n=15$, length}
\end{subfigure}
\hfill
\begin{subfigure}[b]{0.20\textwidth}
\includegraphics[width=1.0\textwidth, trim=55pt 0pt 0pt 0pt]{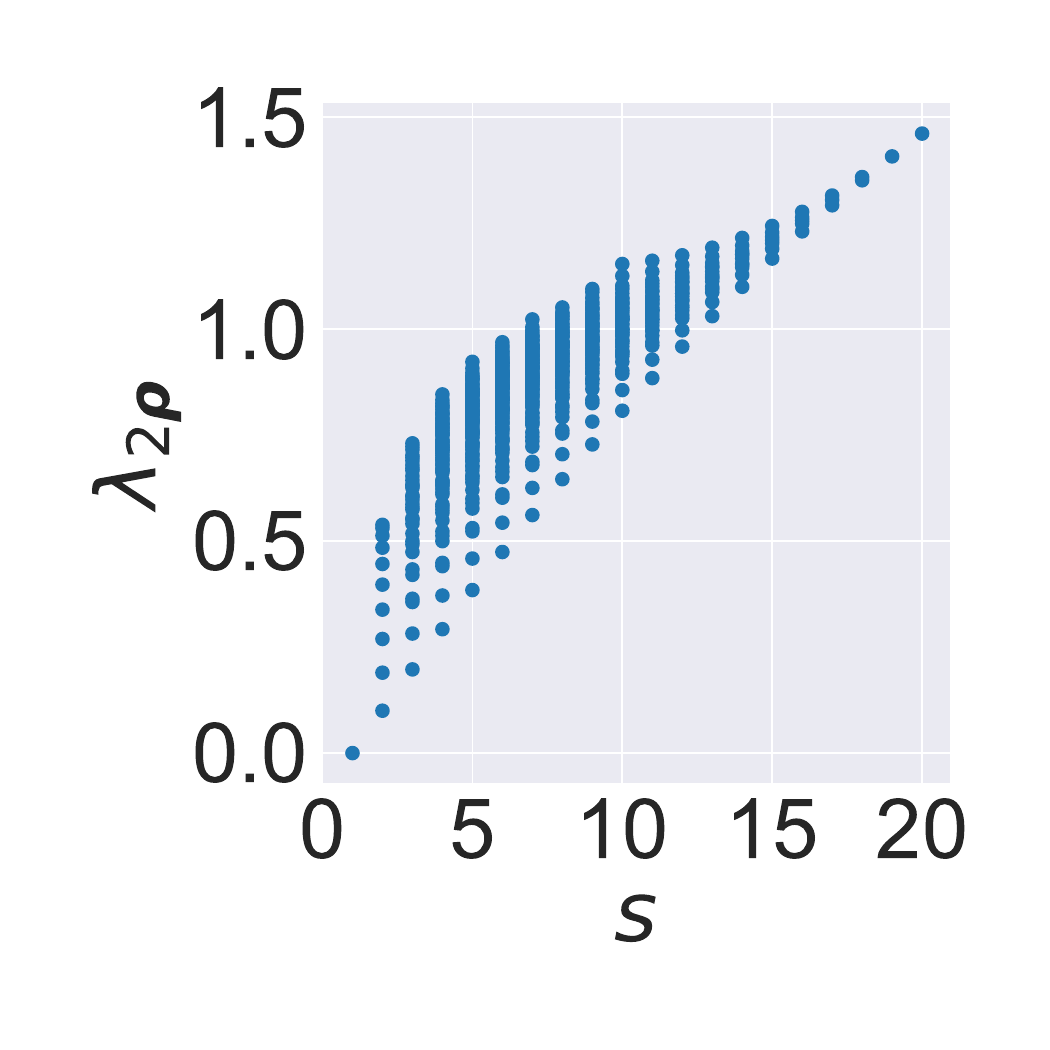}
\caption{$n=20$, length}
\end{subfigure}
\hfill
\begin{subfigure}[b]{0.20\textwidth}
\includegraphics[width=1.0\textwidth, trim=55pt 0pt 0pt 0pt]{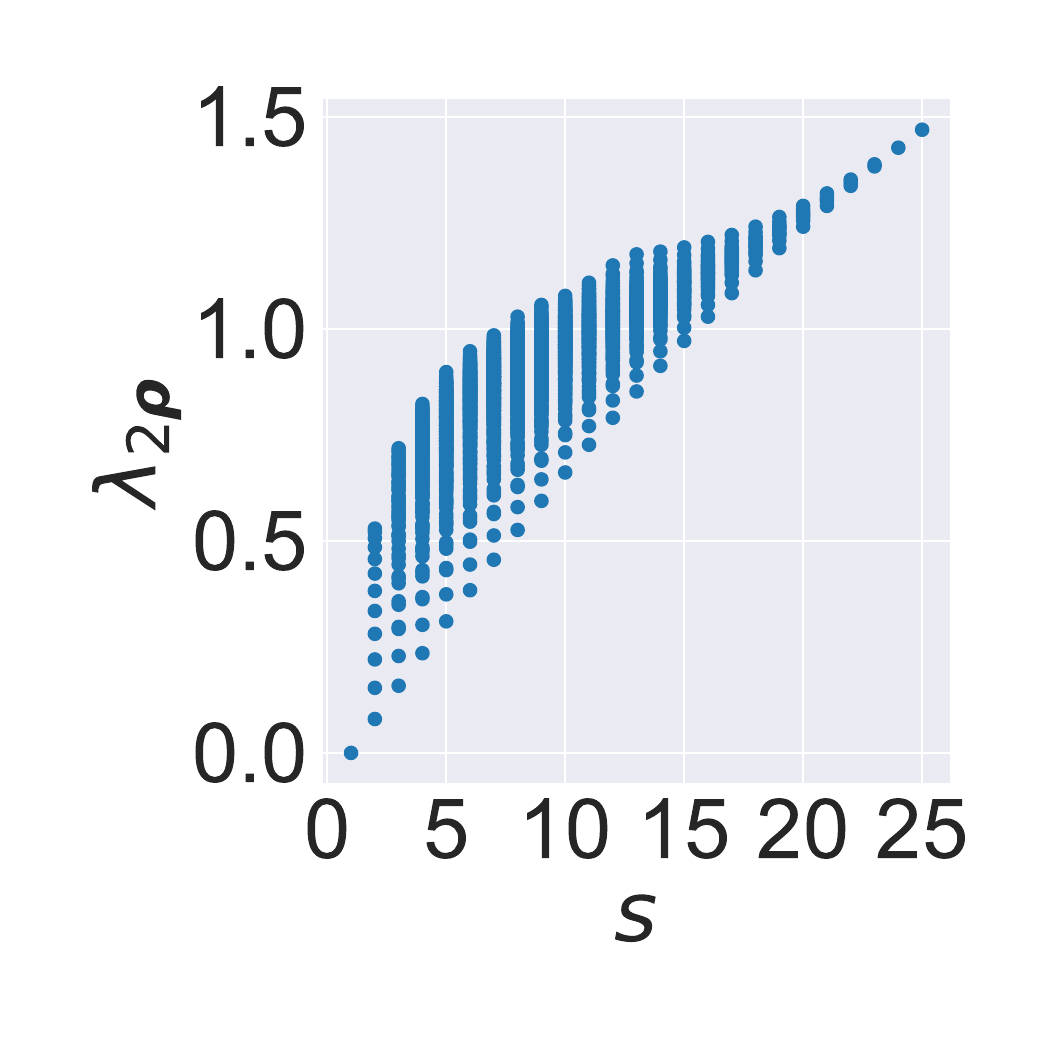}
\caption{$n=25$, length}
\end{subfigure}
\hfill
\begin{subfigure}[b]{0.20\textwidth}
\includegraphics[width=1.0\textwidth, trim=55pt 0pt 0pt 0pt]{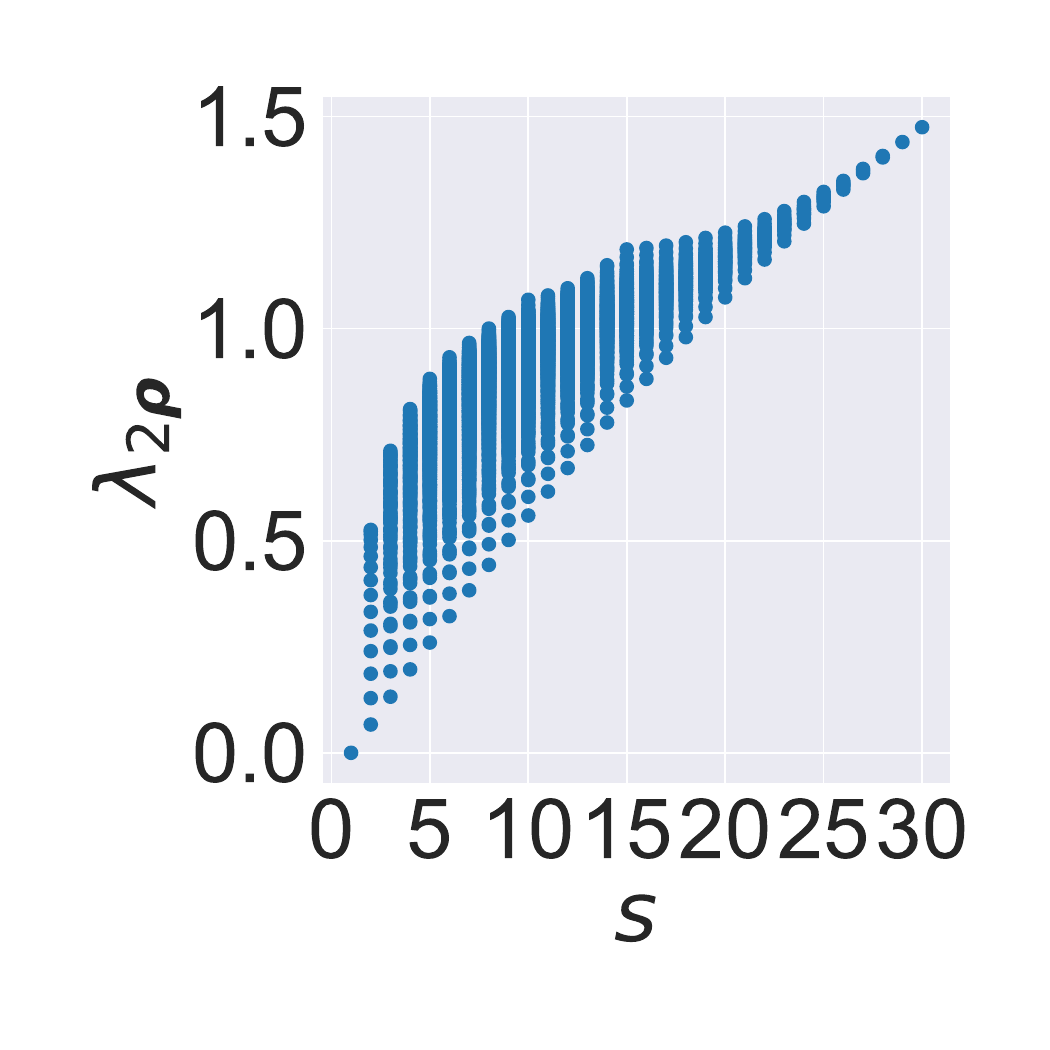}
\caption{$n=30$, length}
\end{subfigure}

\begin{subfigure}[b]{0.20\textwidth}
\includegraphics[width=1.0\textwidth, trim=55pt 0pt 0pt 0pt]{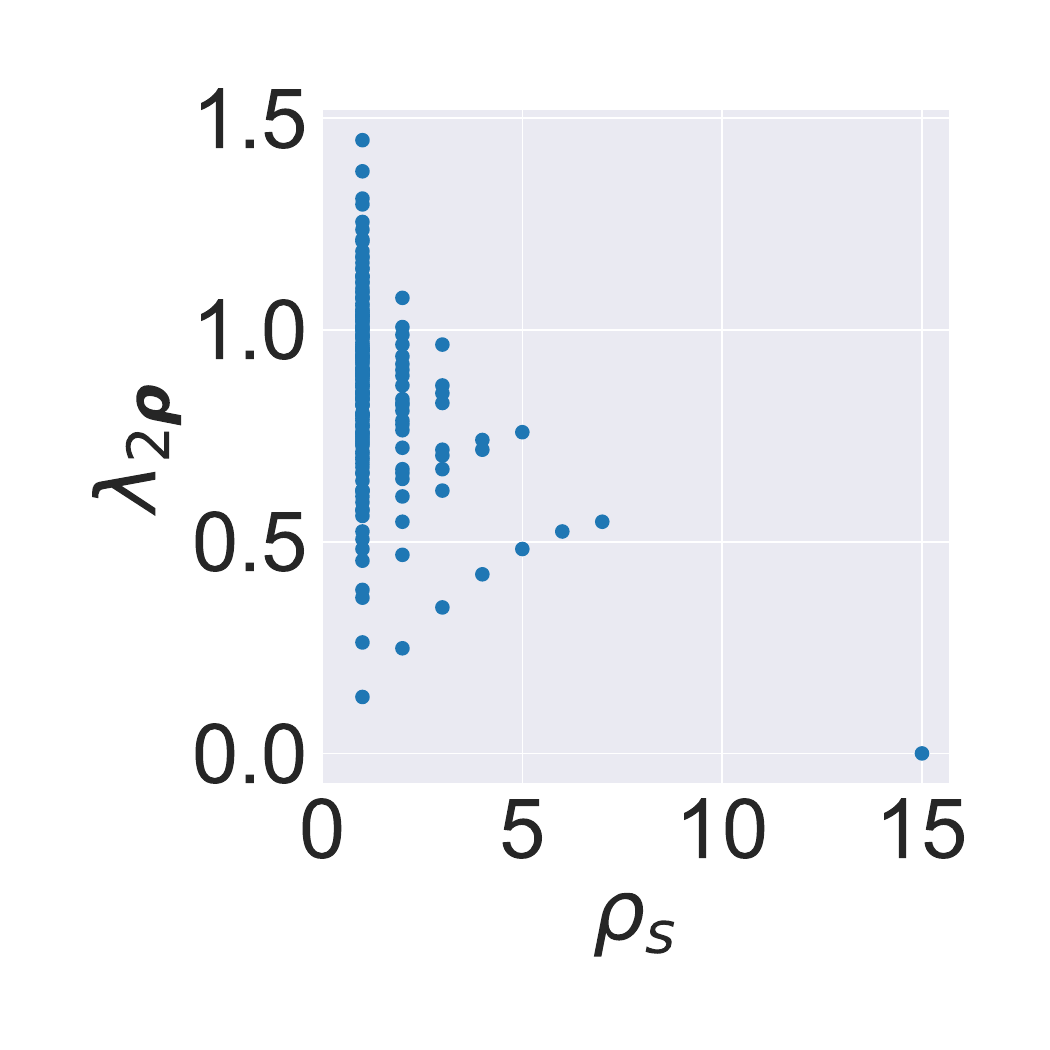}
\caption{$n=15$, min}
\end{subfigure}
\hfill
\begin{subfigure}[b]{0.20\textwidth}
\includegraphics[width=1.0\textwidth, trim=55pt 0pt 0pt 0pt]{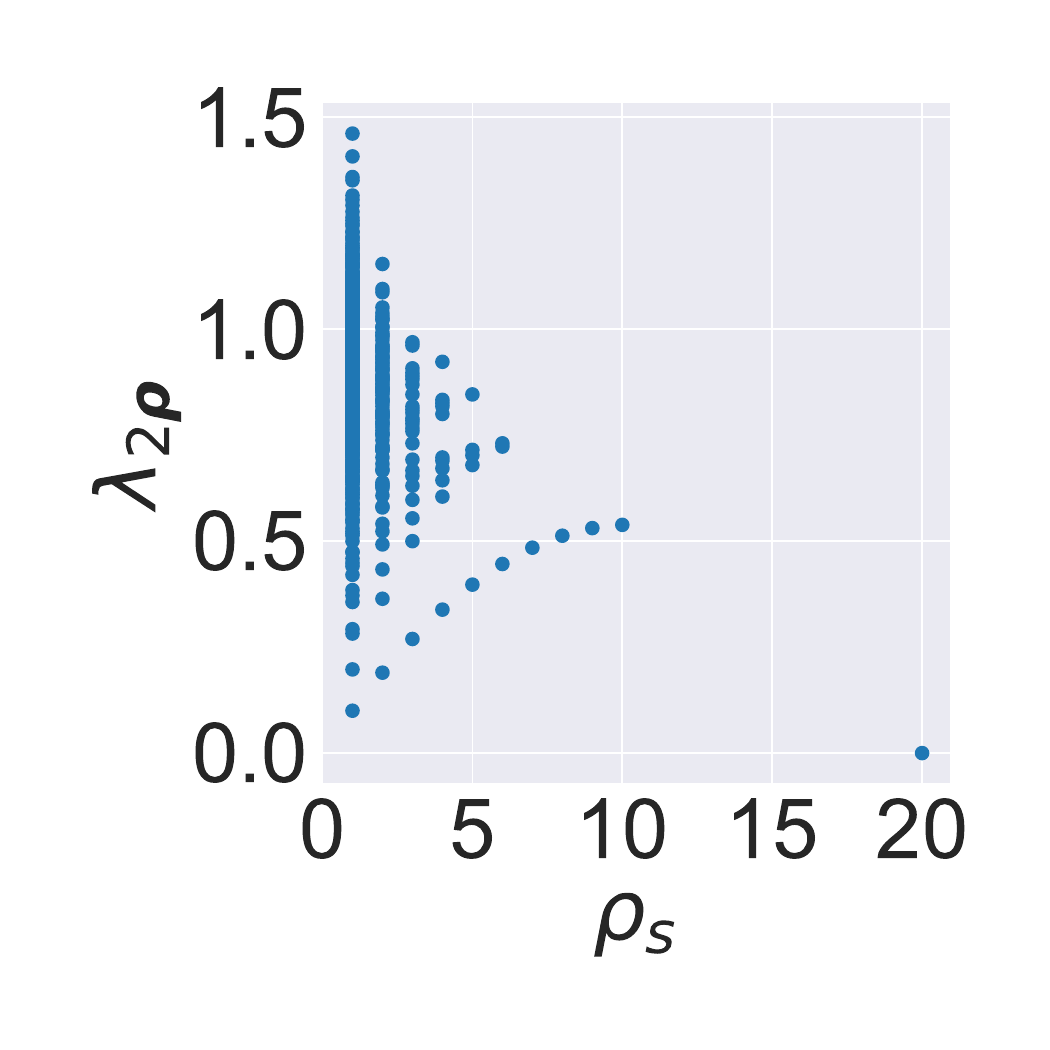}
\caption{$n=20$, min}
\end{subfigure}
\hfill
\begin{subfigure}[b]{0.20\textwidth}
\includegraphics[width=1.0\textwidth, trim=55pt 0pt 0pt 0pt]{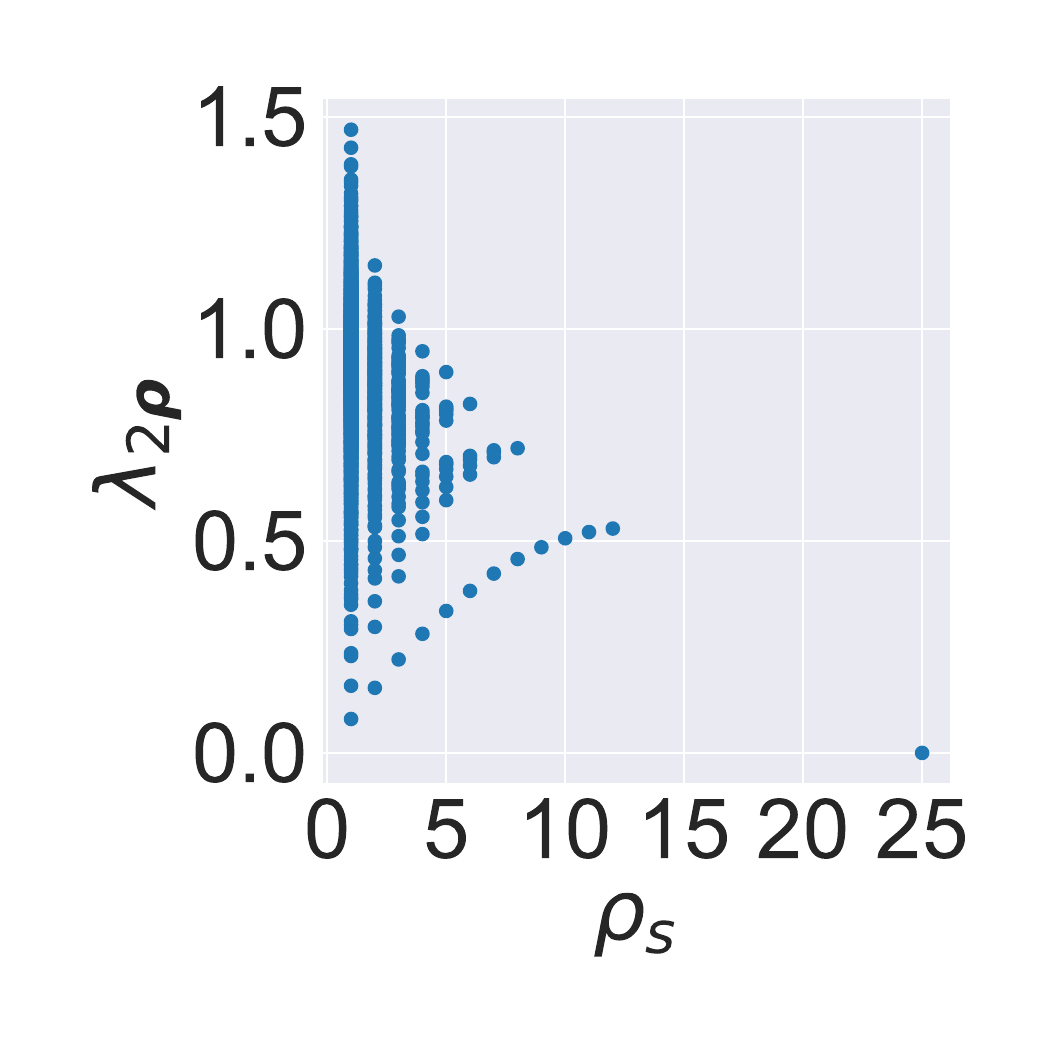}
\caption{$n=25$, min}
\end{subfigure}
\hfill
\begin{subfigure}[b]{0.20\textwidth}
\includegraphics[width=1.0\textwidth, trim=55pt 0pt 0pt 0pt]{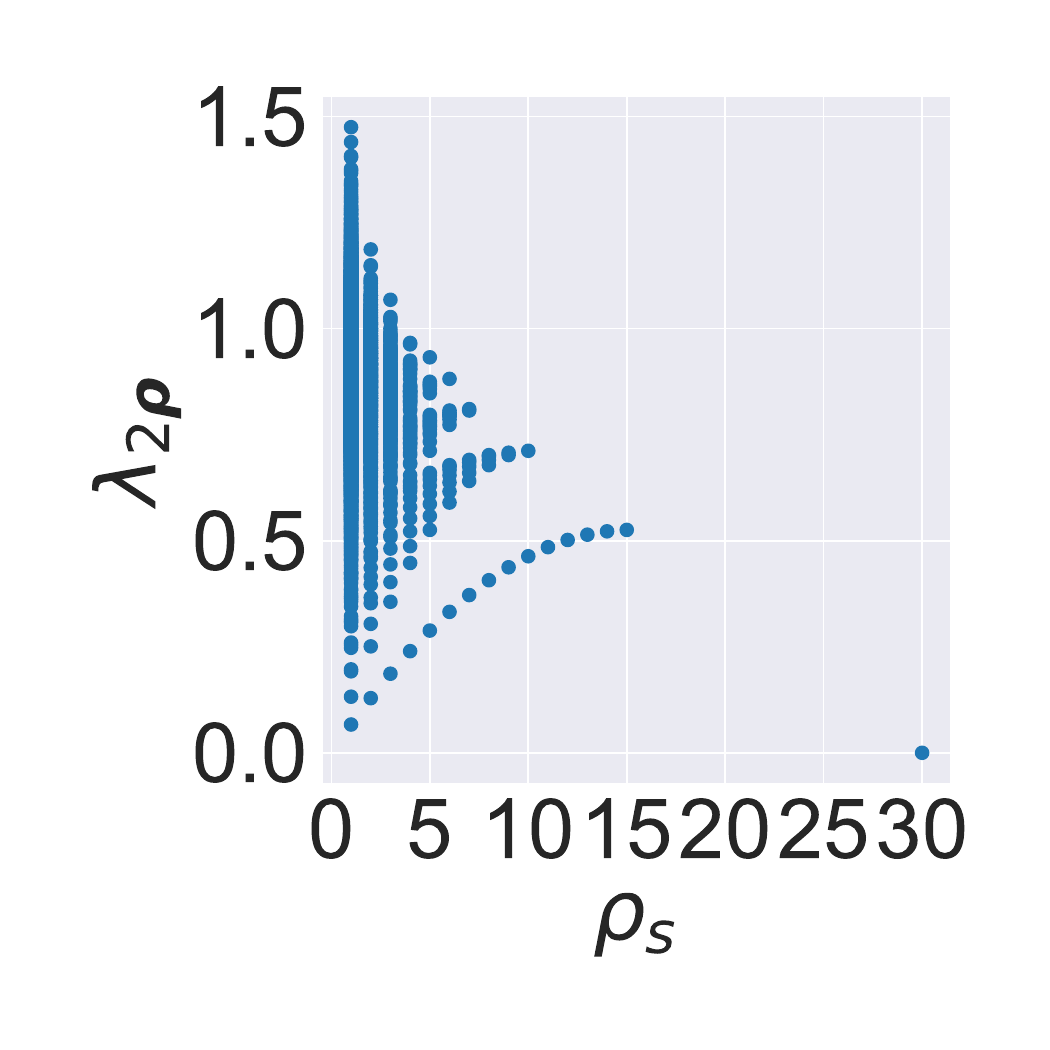}
\caption{$n=30$, min}
\end{subfigure}

\caption{Laplacian eigenvalues $\lambda_{2\v{\rho}}$ versus various truncation heuristics, namely, maximal entry $\rho_1$, length $s$ and minimal entry $\rho_s$, for all partitions $\v{\rho} = (\rho_1, \ldots, \rho_s) \vdash n$.}
\label{fig:alternative-truncation-heuristics}
\end{figure}

\begin{figure}[t]

\begin{subfigure}[b]{0.20\textwidth}
\includegraphics[width=1.0\textwidth, trim=55pt 0pt 0pt 0pt]{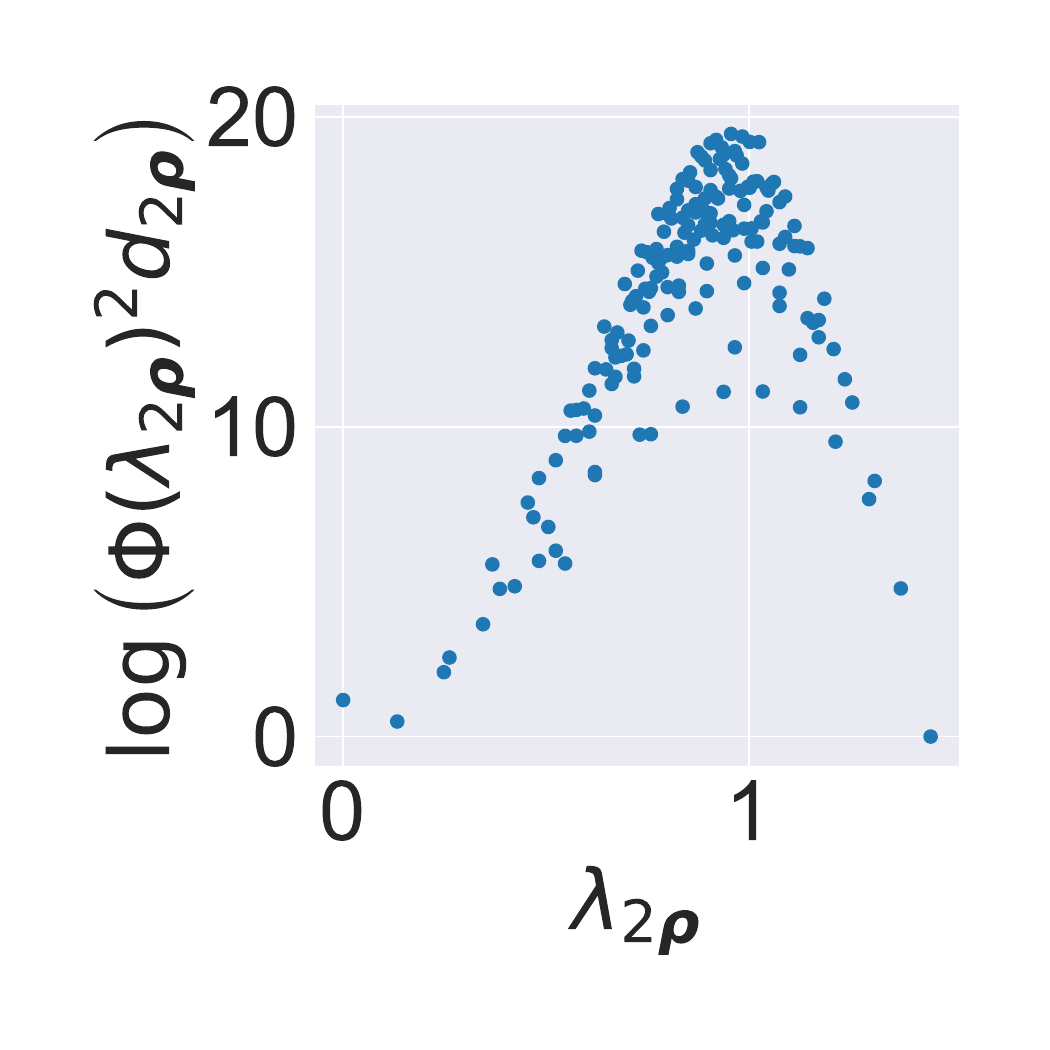}
\caption{$n=15$, without DC}
\end{subfigure}
\hfill
\begin{subfigure}[b]{0.20\textwidth}
\includegraphics[width=1.0\textwidth, trim=55pt 0pt 0pt 0pt]{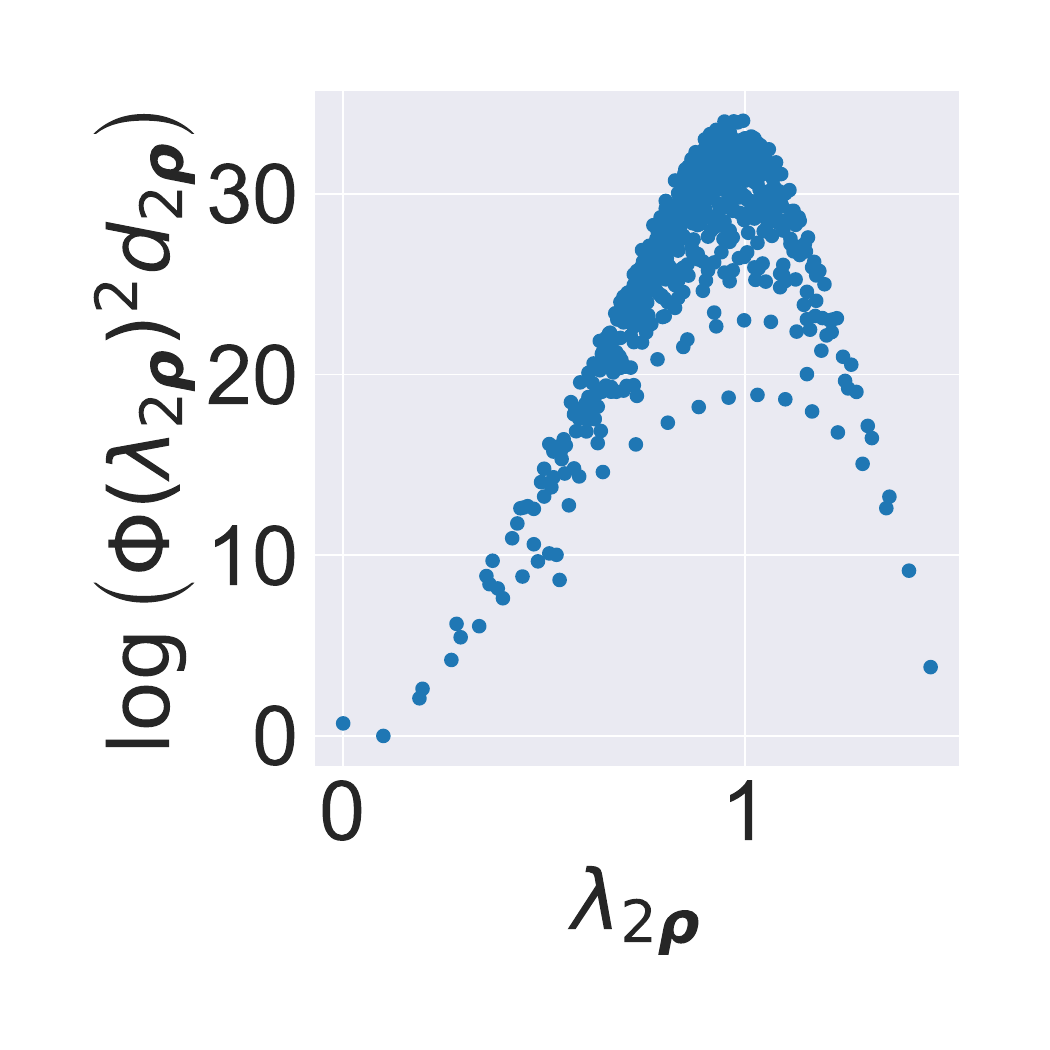}
\caption{$n=20$, without DC}
\end{subfigure}
\hfill
\begin{subfigure}[b]{0.20\textwidth}
\includegraphics[width=1.0\textwidth, trim=55pt 0pt 0pt 0pt]{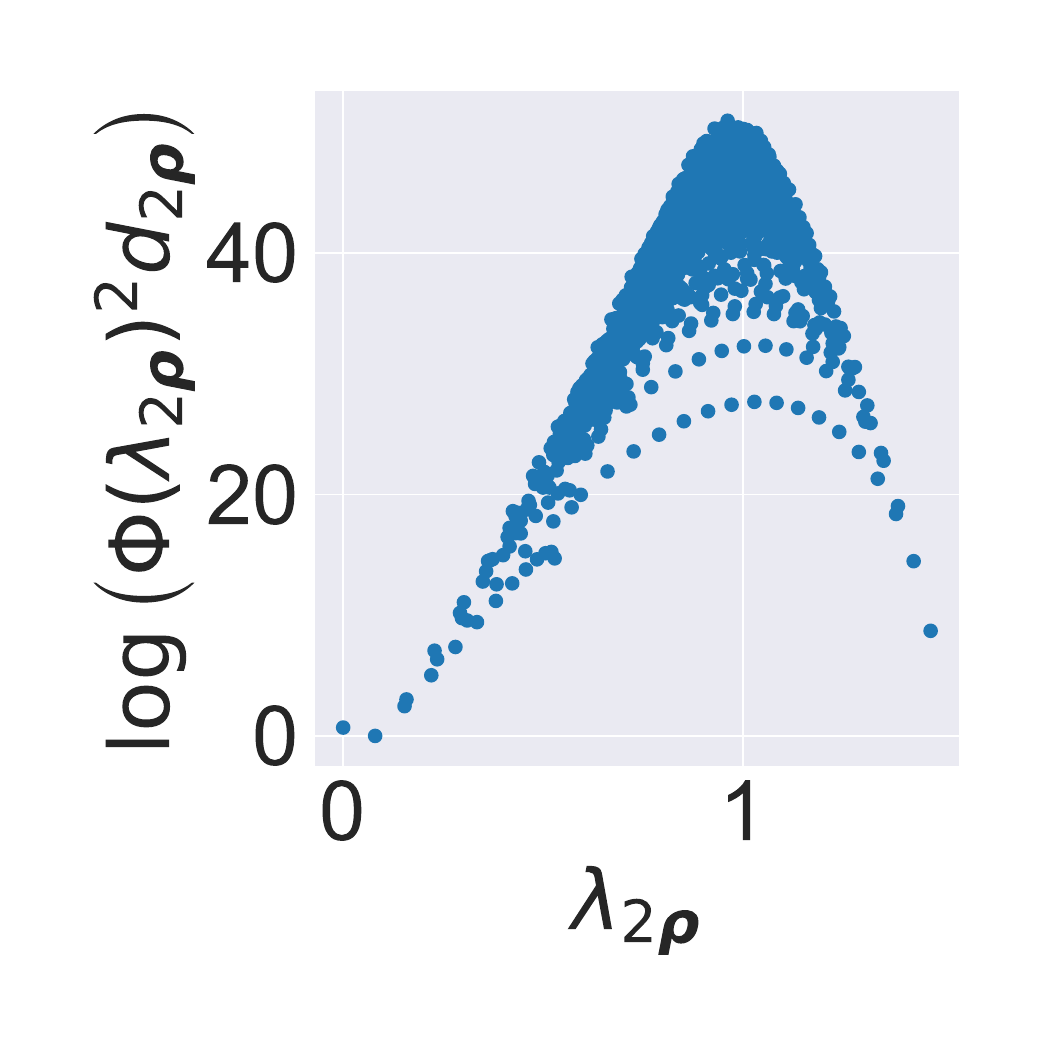}
\caption{$n=25$, without DC}
\end{subfigure}
\hfill
\begin{subfigure}[b]{0.20\textwidth}
\includegraphics[width=1.0\textwidth, trim=55pt 0pt 0pt 0pt]{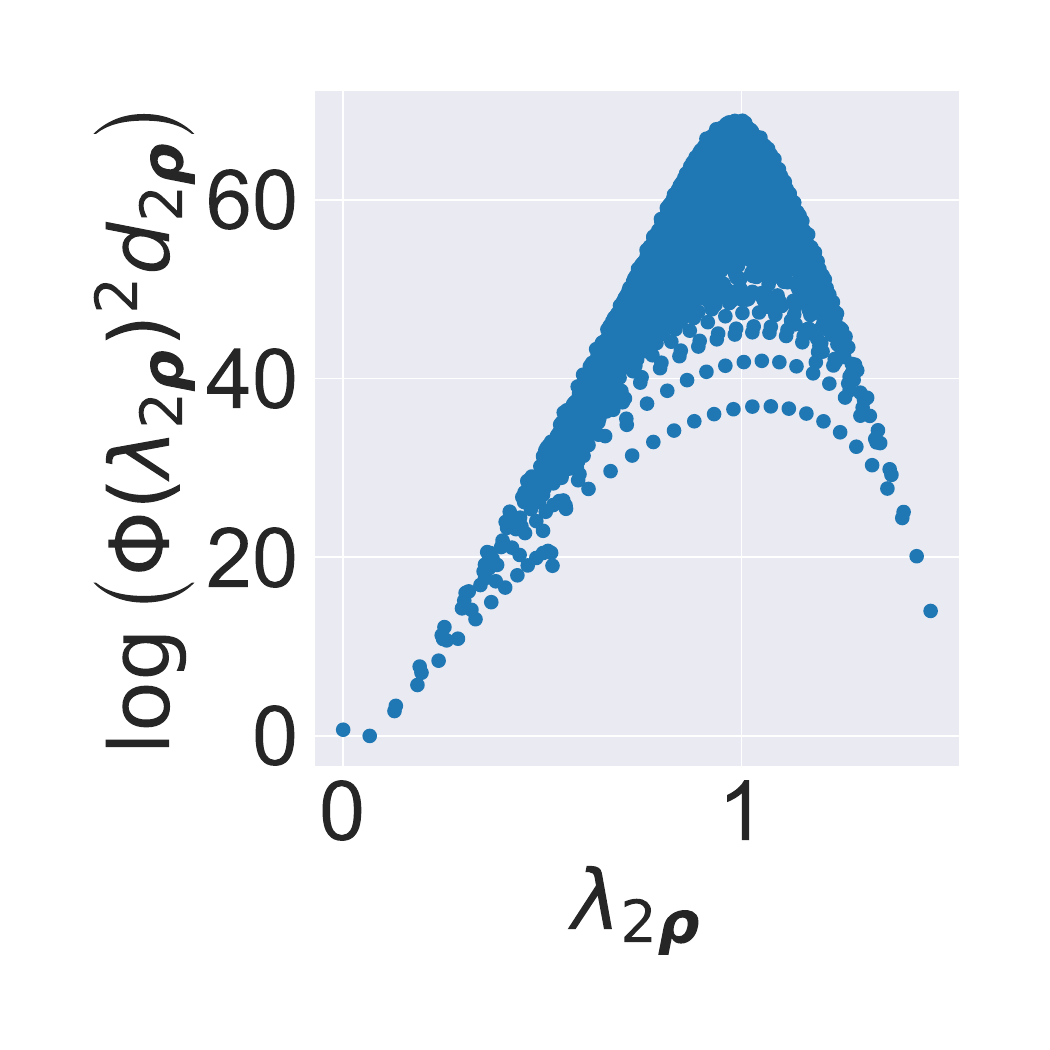}
\caption{$n=30$, without DC}
\end{subfigure}

\begin{subfigure}[b]{0.20\textwidth}
\includegraphics[width=1.0\textwidth, trim=55pt 0pt 0pt 0pt]{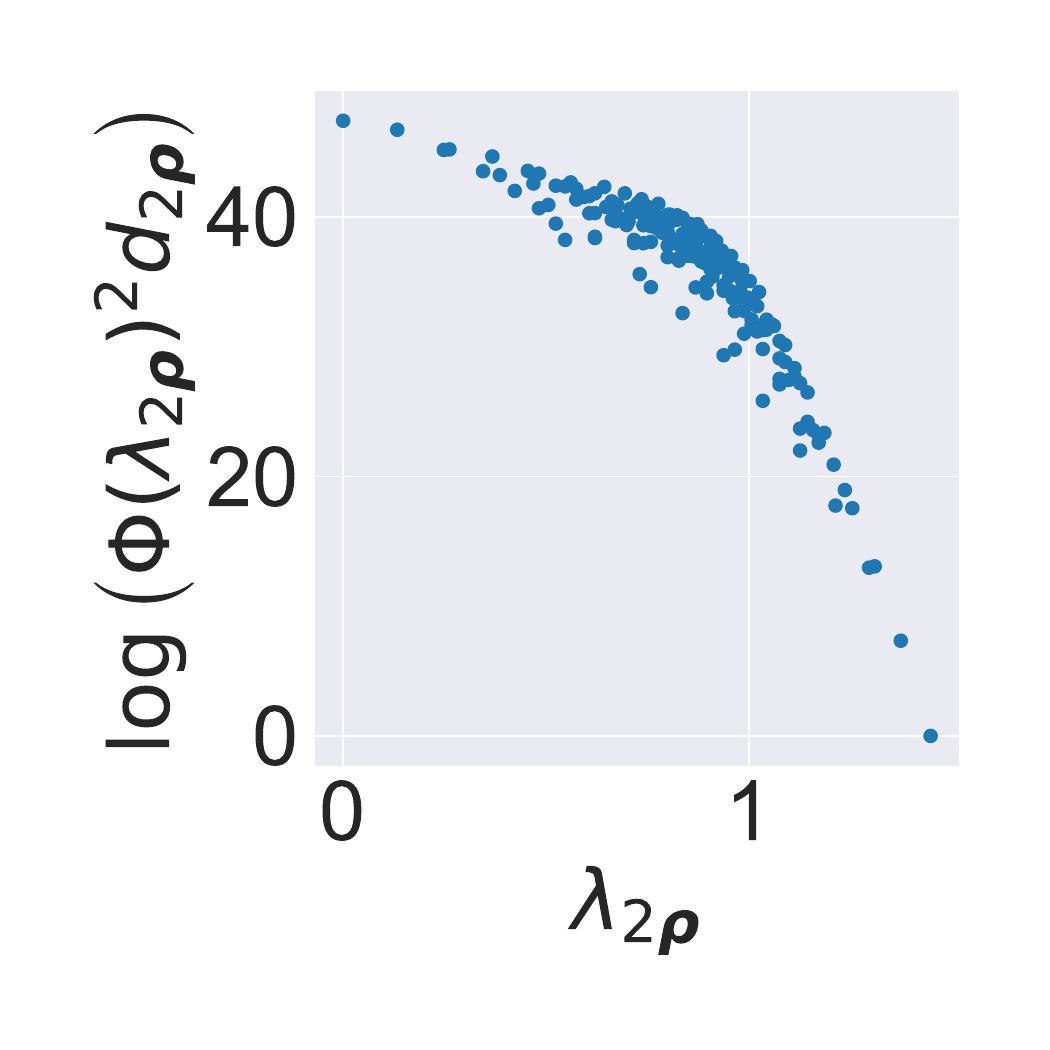}
\caption{$n=15$, with DC}
\end{subfigure}
\hfill
\begin{subfigure}[b]{0.20\textwidth}
\includegraphics[width=1.0\textwidth, trim=55pt 0pt 0pt 0pt]{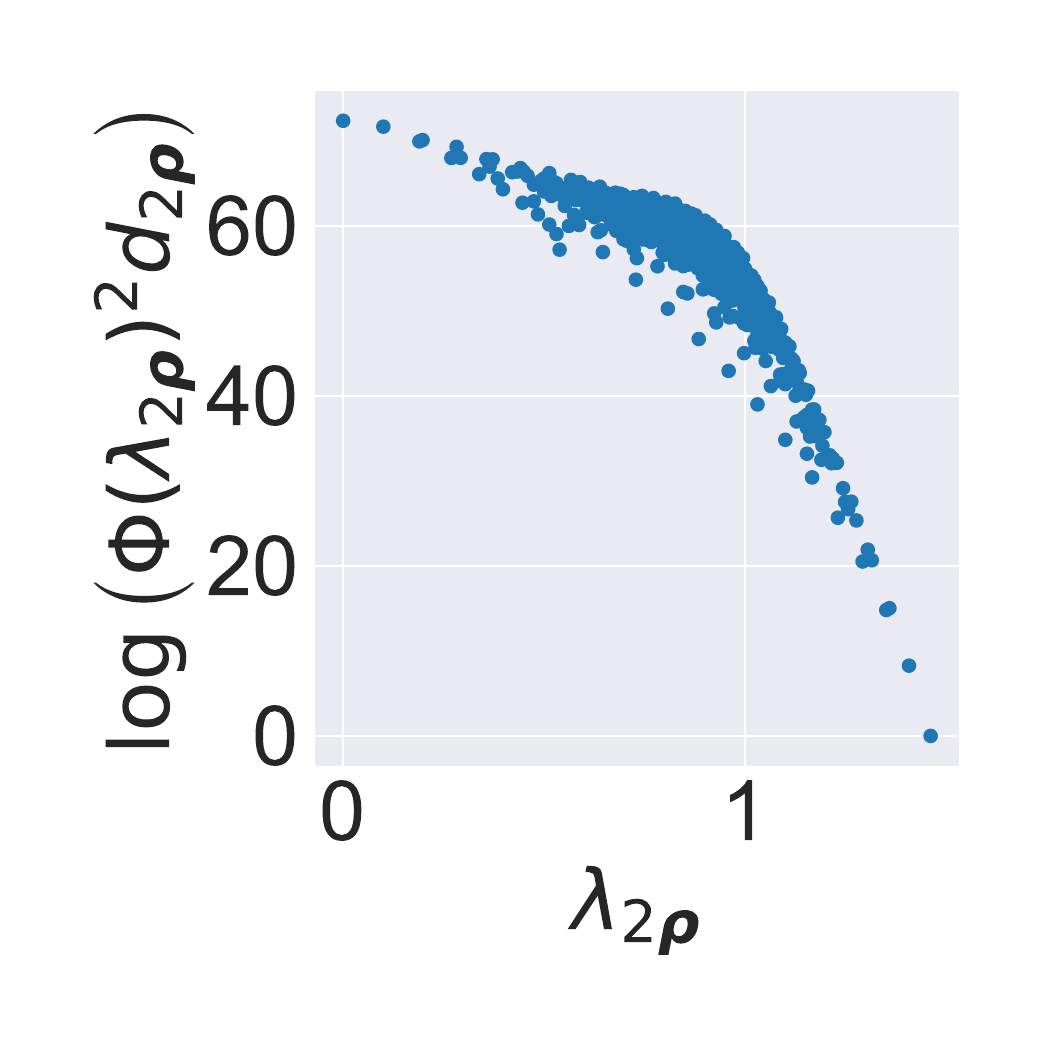}
\caption{$n=20$, with DC}
\end{subfigure}
\hfill
\begin{subfigure}[b]{0.20\textwidth}
\includegraphics[width=1.0\textwidth, trim=55pt 0pt 0pt 0pt]{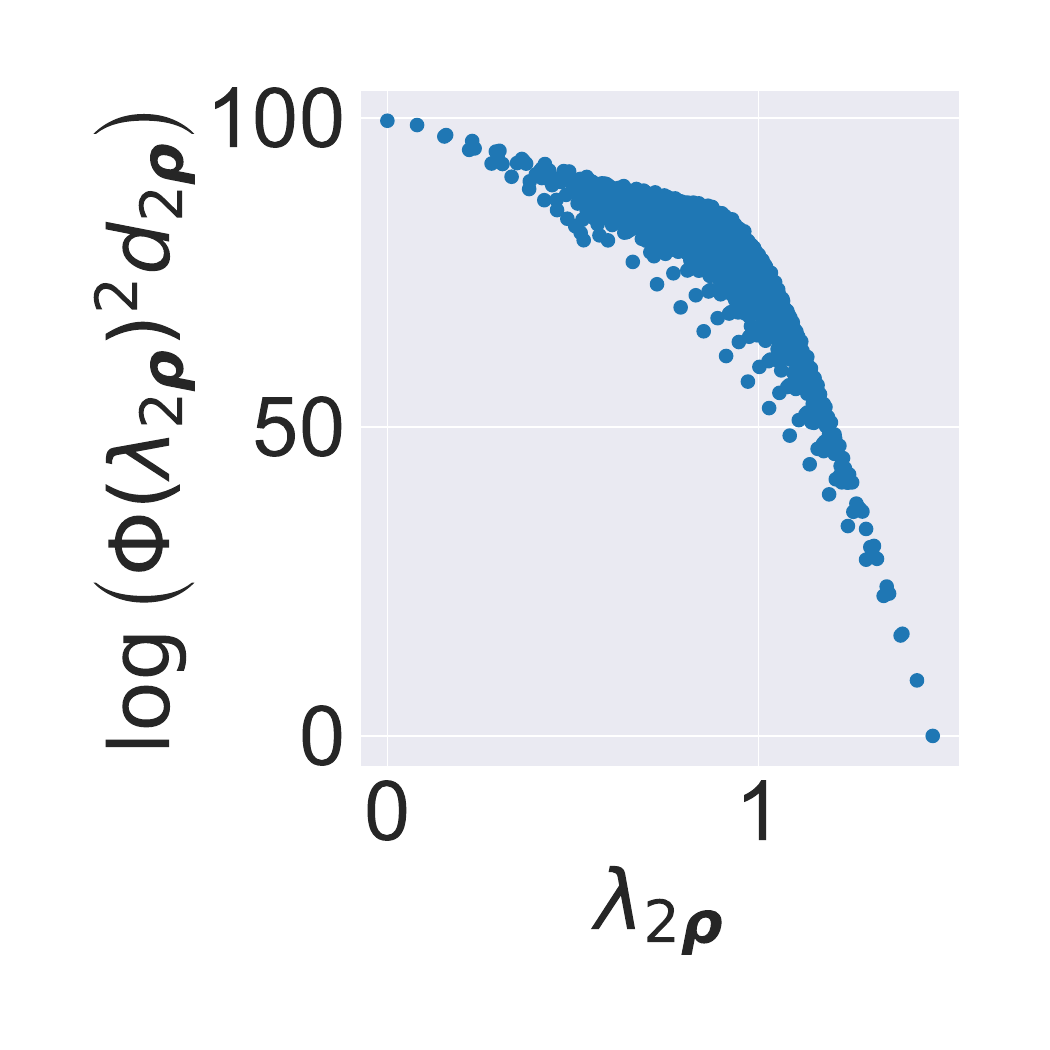}
\caption{$n=25$, with DC}
\end{subfigure}
\hfill
\begin{subfigure}[b]{0.20\textwidth}
\includegraphics[width=1.0\textwidth, trim=55pt 0pt 0pt 0pt]{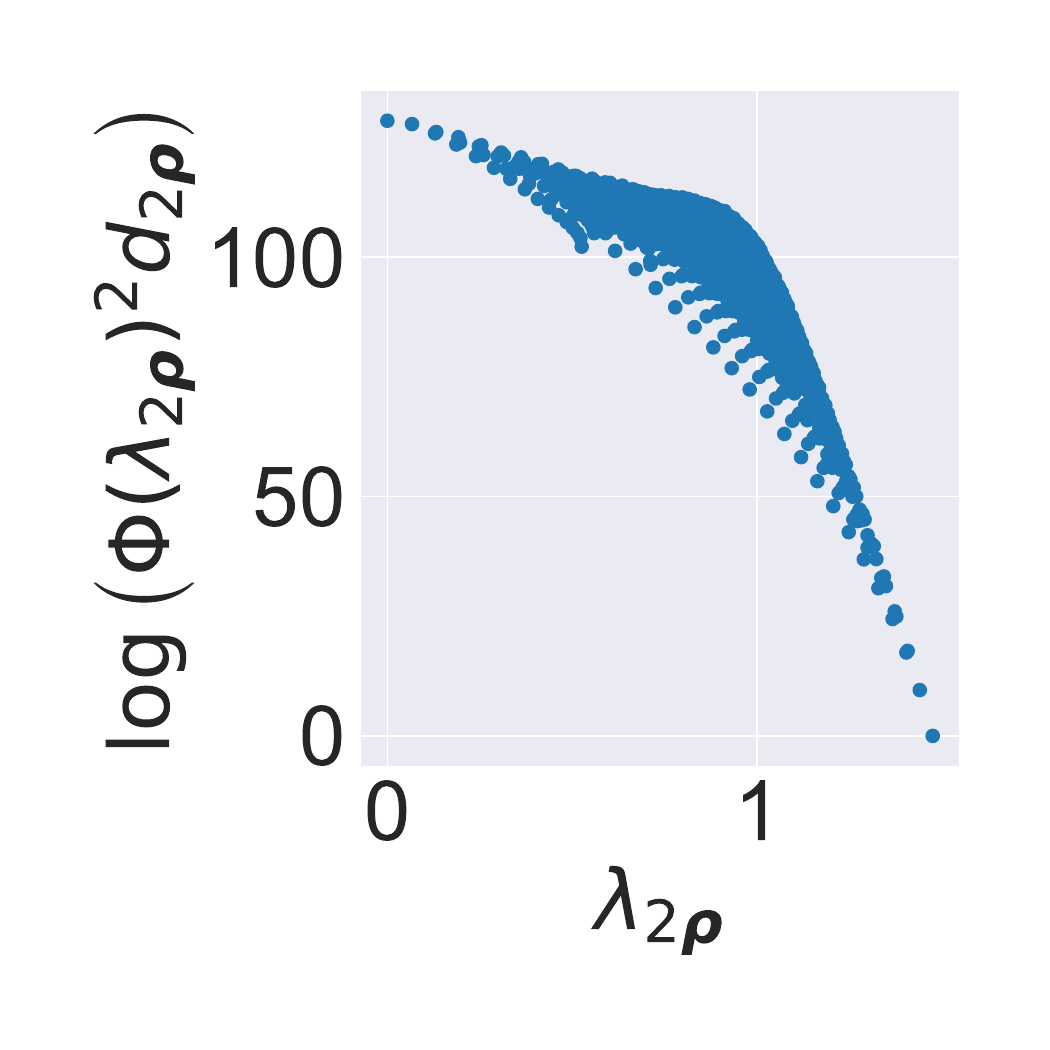}
\caption{$n=30$, with DC}
\end{subfigure}

\caption{Scatter plot of eigenvalue $\lambda_{2\v{\rho}}$ against log-scaled ``spectral density'' $\log \Phi(\lambda_{2\v{\rho}})^2 d_{2\v{\rho}}$ for base $\nu=2.5$, with and without degree correction (DC).}
\label{fig:spectral-density}
\end{figure}

\end{document}

%% file: algorithms/zsf-naive.tex
\begin{algorithm}[t]
\caption{Naive method}\label{algorithm:ZSFNaive}
\begin{algorithmic}
\ForAll{$\mu \vdash n$}
    \State $a_\mu^\lambda \gets 0$
    \State $|A_\mu| \gets 0$
\EndFor
\ForAll{$x \in X$}
    \State $\mu \gets d(x, x_0)$
    \State $|A_\mu| \gets |A_\mu| + 1$
    \ForAll{$\sigma \in C'_t$}
        \If{$\{ \sigma t\} \text{ covers } x$}
            \State $a_\mu^\lambda \gets a_\mu^\lambda + \sgn(\sigma)$
        \EndIf
    \EndFor
\EndFor
\ForAll{$\mu \vdash n$}
    \State $\phi_\lambda(\mu) \gets \frac{a_\mu^\lambda}{|A_\mu|}$
\EndFor
\end{algorithmic}
\end{algorithm}

%% file: algorithms/augmented-monomials.tex
\begin{algorithm}[t]
\caption{Decomposition of $\widetilde{m}_\kappa$ in the basis $p_\mu$}\label{algorithm:AugmentedMonomials}
\begin{algorithmic}
\State \textbf{Procedure:} $monom$
\State \textbf{Input:} $\kappa = (\kappa_1, \ldots, \kappa_r)$
\State \textbf{Global Variables:} $cache$
\If{$cache[\kappa] \neq \varnothing$}
    \State \textbf{return } $cache[\kappa]$
\EndIf
\If{$r = 1$}
    \State \textbf{return } $p_{\kappa_1}$
\EndIf
\If{$r = 2$}
    \State \textbf{return } $p_{\kappa_1} \cdot p_{\kappa_2} - p_{\kappa_1 + \kappa_2}$
\Else
    \State $s \gets p_{\kappa_r} \cdot monom(\kappa_1, \ldots, \kappa_{r-1})$
    \ForAll{$i = 1, \ldots, r-1$}
        \State $\rho \gets (\kappa_1, \ldots, \kappa_{r-1})$
        \State $\rho_i \gets \rho_i + \kappa_r$
        \State $s \gets s - monom(\rho)$
    \EndFor
    \State $cache[\kappa] \gets s$
    \State \textbf{return } s
\EndIf
\end{algorithmic}
\end{algorithm}

%% file: figures/small-nni-big-matching-dist.tex
\begin{figure}[t!]
\begin{center}
\begin{subfigure}[b]{0.45\textwidth}
\begin{forest}
[R [1] [$\bullet$ [2] [$\bullet$ [9] [$\bullet$ [8] [$\bullet$ [7] [$\bullet$ [6] [$\bullet$ [5] [$\bullet$ [3] [4]]]]]]]]]
\end{forest}
\end{subfigure}
\hfill
\begin{subfigure}[b]{0.45\textwidth}
\begin{forest}
[R [$\bullet$ [1] [2]] [$\bullet$ [9] [$\bullet$ [8] [$\bullet$ [7] [$\bullet$ [6] [$\bullet$ [5] [$\bullet$ [3] [4]]]]]]]]
\end{forest}
\end{subfigure}
\end{center}
\caption{Example of two phylogenetic trees differing by a single NNI-move such that corresponding matchings have large distance in quotient Cayley graph. Specifically, x=[(1, 16), (2, 15), (3, 4), (5, 10), (6, 11), (7, 12), (8, 13), (9, 14)] and y=[(1, 2), (3, 4), (5, 11), (6, 12), (7, 13), (8, 14), (9, 15), (10, 16)]. This example also generalizes to larger $n$ by inserting new leaves between $2$ and $9$ in the left tree.}
\label{fig:small-nni-big-matching-dist}
\end{figure}